\ifdefined\COMBINEDSUPPLEMENT\else\def\COMBINEDSUPPLEMENT{1}\fi
\ifdefined\ARXIVVERSION\else\def\ARXIVVERSION{1}\fi

\documentclass[letterpaper,journal]{IEEEtran}

\usepackage{amsmath,amsfonts,amssymb,mathtools,amsthm}
\usepackage{graphicx}
\usepackage[caption=false,font=normalsize,labelfont=sf,textfont=sf]{subfig}
\usepackage{textcomp}
\usepackage{stfloats}
\usepackage{placeins}
\usepackage{url}
\usepackage{hyperref}
\hypersetup{hypertexnames=false}
\usepackage{booktabs}
\usepackage{multirow}
\usepackage{xcolor}
\usepackage{enumitem}
\usepackage{dsfont}
\usepackage{pifont}
\usepackage{algorithm}
\usepackage{algorithmic}
\usepackage{cite}

\definecolor{blue}{HTML}{006EBB}
\definecolor{red}{HTML}{C1272D}
\newcommand{\suppRef}[2]{\ifdefined\COMBINEDSUPPLEMENT the supplementary material, #1~\ref{#2}\else the supplementary material\fi}
\newcommand{\suppTableRef}[1]{\suppRef{Table}{#1}}

\newcommand{\suppSectionRef}[1]{\suppRef{Section}{#1}}

\newcommand{\suppShowcaseRefs}{\ifdefined\COMBINEDSUPPLEMENT the supplementary material, Figures~\ref{fig:showcase_volatile} and~\ref{fig:showcase_stable}\else the supplementary material\fi}

\theoremstyle{plain}
\newtheorem{theorem}{Theorem}[section]

\theoremstyle{definition}

\theoremstyle{remark}

\begin{document}

\title{Beyond Static Uncertainty: Modeling Temporal Uncertainty Dynamics for Probabilistic Time Series Forecasting}

\ifdefined\ARXIVVERSION
\author{Yijun Wang, Qiyuan Zhuang, Larysa Marchanka, and Xiu-Shen Wei
\thanks{Y. Wang, Q. Zhuang, and X.-S. Wei are with the Department of Computer Science, Southeast University, Nanjing, China. Corresponding author: Xiu-Shen Wei (e-mail: weixs@seu.edu.cn).}
\thanks{L. Marchanka is with Francisk Skorina Gomel State University, Belarus.}}
\markboth{arXiv preprint}%
{Wang \MakeLowercase{\textit{et al.}}: Beyond Static Uncertainty: Modeling Temporal Uncertainty Dynamics for Probabilistic Time Series Forecasting}
\else
\author{Yijun Wang, Qiyuan Zhuang, Larysa Marchanka, and Xiu-Shen Wei,~\IEEEmembership{Senior Member,~IEEE}
\thanks{Y. Wang, Q. Zhuang, and X.-S. Wei are with the Department of Computer Science, Southeast University, Nanjing, China. Corresponding author: Xiu-Shen Wei (e-mail: weixs@seu.edu.cn).}
\thanks{L. Marchanka is with Francisk Skorina Gomel State University, Belarus.}}

\markboth{SUBMITTED TO IEEE TKDE}%
{Wang \MakeLowercase{\textit{et al.}}: Beyond Static Uncertainty: Modeling Temporal Uncertainty Dynamics for Probabilistic Time Series Forecasting}
\fi

\maketitle

\begin{abstract}
Real-world time series exhibit temporally structured uncertainty: volatility clusters in turbulent regimes, dissipates in stable periods, and shifts abruptly around structural breaks.
Yet many probabilistic forecasting methods estimate predictive uncertainty as an independent per-step quantity, leaving the evolution and persistence of volatility regimes under-modeled.
We formalize this missing dimension as temporal uncertainty dynamics and instantiate it in the Volatility Dynamics Variational Autoencoder (VolDy-VAE), a non-autoregressive generative forecaster with a location-scale decoder.
VolDy-VAE combines a location path for mean prediction with a recurrent scale path that transfers and evolves a volatility hidden state from the look-back window to the forecasting horizon, enabling temporally coherent predictive variances.
This design yields an adaptive attenuation mechanism: high-variance observations receive lower influence on the location estimate while their uncertainty is preserved through explicit scale predictions.
We further provide a simplified regime-switching analysis showing that, when variances are known or consistently estimated, the volatility-aware objective reduces to inverse-variance weighting, whereas MSE-based estimators remain unbiased but statistically inefficient.
Experiments on nine benchmarks show that VolDy-VAE improves forecasting accuracy and uncertainty calibration over competitive probabilistic and point-forecasting baselines while maintaining low inference latency; plug-in studies further indicate that the VolDy principle can benefit GAN, Koopman VAE, and Transformer backbones.
The source code is publicly available at \url{https://github.com/wangyijunlyy/VolDy-VAE}.
\end{abstract}


\section{Introduction}
\label{sec:intro}
\IEEEPARstart{I}{n} recent years, deep learning has substantially advanced time series analysis~\cite{lim2021time,kong2025deep}, including tasks such as imputation~\cite{Cao2018BRITS,Tashiro2021CSDI,Gao_2025,ICLR2025_033394ac,wang2024deep} and anomaly detection~\cite{Xu2022AnomalyTransformer,wang2023drift,liu2024elephant, miao2025parameter, wu2025catch}. Among these tasks, {Probabilistic Time Series Forecasting (PTSF)} is important for high-stakes decision-making~\cite{Gneiting2007ProbForecast,gneiting2007probabilistic}.
In energy grids, probabilistic load forecasting is essential for risk-aware dispatch~\cite{Hong2016ProbLoadForecast,HaoWang1,sun2022solar}; in quantitative finance and supply chains, accurate uncertainty modeling supports robust risk control under volatility~\cite{Sezer2020FinanceDL,Li2018DCRNN,huang2022dgraph}.
Unlike deterministic forecasters that yield only point estimates, PTSF aims to model the full predictive distribution, allowing for the precise evaluation of confidence intervals and tail risks~\cite{Zhang2024ProbTS}.

Much recent work has focused on increasingly expressive generative architectures, from diffusion models~\cite{Rasul2021TimeGrad,Tashiro2021CSDI,Ye2025NsDiff} that require many iterative denoising steps to structured state space models~\cite{mamba,pessoa2025mamba} with sophisticated gating mechanisms.
This focus leaves a practical question open: which modeling deficiencies most limit current probabilistic forecasters?
We examine two simple but important shortcomings, a statistically inefficient training objective and a memoryless uncertainty estimator, and show that targeted lightweight mechanisms can address much of the resulting performance gap.

\begin{figure}[t]
\centering
\includegraphics[width=0.9\linewidth]{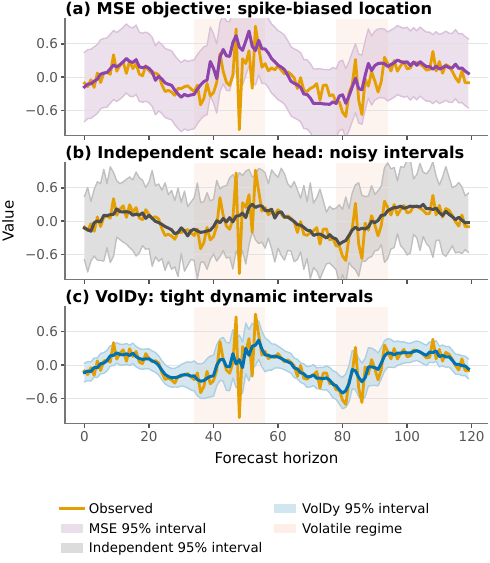}
\caption{{Conceptual illustration of temporal uncertainty dynamics.}
(a) MSE-based homoscedastic forecasters use a constant scale, so volatile spikes can distort the location path before and after high-variance regimes.
(b) Feed-forward heteroscedastic heads can output time-varying scales, but estimate them independently at each step, producing jagged and memoryless intervals.
(c) VolDy models temporally coherent scale dynamics, keeping the location prediction close to the predictable central trajectory while using the scale component to absorb volatile deviations.}
\label{fig:motivation}
\end{figure}

A challenge in PTSF lies in the temporal dynamics of uncertainty itself.
Real-world volatility is not an independent, per-step nuisance parameter. It evolves over time, depends on regimes, and often persists across adjacent intervals.
Wind power generation exhibits stable patterns during calm weather but undergoes violent fluctuations during storms that persist over contiguous intervals; financial markets alternate between extended periods of tranquility and sustained bursts of extreme volatility during crises~\cite{Gneiting2007ProbForecast}.
These phenomena, known as volatility clustering in econometrics, reveal that uncertainty at time $t$ is strongly correlated with uncertainty at neighboring time steps, forming coherent volatility regimes.

However, several prevalent generative forecasting models, including standard VAEs~\cite{Kingma2014VAE,desai2021timevae} and $K^2$VAE~\cite{Wu2025K2VAE}, do not explicitly model this temporal structure in the predictive scale.
These methods suffer from two compounding limitations.
At the objective level, the predominant Mean Squared Error (MSE) loss imposes a rigid {homoscedastic assumption} ($\sigma^2 = \text{const}$)~\cite{kendall2017uncertainties}, which cannot represent time-varying uncertainty and gives high-variance observations disproportionately large gradient influence during finite-sample training (Fig.~\ref{fig:motivation}(a)).
At the architectural level, one might expect that replacing MSE with a heteroscedastic likelihood and adding a learned variance head would resolve the first deficiency. However, many learned-variance designs~\cite{kendall2017uncertainties} still estimate $\sigma_t$ through local feed-forward projections, treating uncertainty as a largely per-step attribute.
This memoryless estimation ignores the fact that real-world volatility is temporally persistent and sequentially correlated: high-volatility periods tend to cluster and propagate, forming coherent regimes rather than isolated spikes.
Without an explicit mechanism to carry forward and evolve a volatility state across the forecast horizon, such models may fail to capture volatility clustering, regime persistence, or smooth transitions between uncertainty regimes.

To address these dual limitations, we formalize the concept of temporal uncertainty dynamics: probabilistic forecasting should predict the location and scale of future values while also modeling how uncertainty evolves and transitions across time.
We instantiate this principle in the {Volatility Dynamics Variational Autoencoder (VolDy-VAE)}, whose location-scale decoder explicitly parameterizes both the predictive mean and a temporally coherent variance.
The architecture uses a {recurrent volatility dynamics module} based on a Gated Recurrent Unit (GRU), which maintains and evolves a volatility hidden state $\mathbf{h}^{\mathrm{vol}}_t$ across time steps.
Unlike feed-forward variance heads that produce independent per-step estimates, this GRU-based mechanism enables the predicted scale $\sigma_t$ to depend on the history of volatility evolution, capturing volatility clustering, regime persistence, and smooth inter-regime transitions.
The recurrent volatility module works synergistically with a heteroscedastic Gaussian NLL objective, jointly addressing both the architectural and objective-level deficiencies of existing methods.
This formulation naturally induces an {adaptive attenuation effect} (visualized in Fig.~\ref{fig:motivation}(b)):
by dynamically estimating a time-dependent variance $\sigma_t$ that captures volatility transitions, the model yields confidence intervals with {time-varying widths} that expand during volatile regimes and contract in stable periods.
From an optimization perspective, the scale component absorbs irreducible stochastic uncertainty ($\propto 1/\sigma_t^2$ weighting), allowing the location component to remain stable and recover the underlying trend.

The main contributions of this work are as follows:
\begin{itemize}
    \item We identify and formalize a dual gap in existing probabilistic forecasting methods: (i)~the objective-level homoscedastic misspecification induced by MSE training, and (ii)~the limited treatment of temporal uncertainty memory, where learned-variance models often estimate $\sigma_t$ locally without explicitly modeling how volatility evolves across the forecast horizon. We show in a simplified regime-switching setting that treating uncertainty as an independent per-step quantity can lead to statistically inefficient estimation when gradient updates are strongly affected by high-volatility observations.
    \item We propose VolDy-VAE, a generative framework that instantiates this principle through a dual-head decoder combining a location head for trend prediction with a GRU-based recurrent scale head for volatility dynamics modeling. The recurrent architecture maintains a volatility hidden state that captures how uncertainty propagates and transitions across time, addressing both the objective-level misspecification of MSE and the architectural gap between generic recurrent distributional forecasting and dedicated volatility-state modeling. We further analyze why, when the variance is known or consistently estimated, this volatility-aware formulation reduces to the inverse-variance weighting that is efficient under heteroscedastic regimes.
    \item We provide plug-in-style evidence that the proposed temporal uncertainty modeling principle can be instantiated beyond the base VAE. By extending the VolDy mechanism to GAN-based (TimeGAN), Koopman-based ($K^2$VAE), and Transformer-based (PatchTST) backbones, we observe consistent improvements on representative datasets.
    \item Experiments on nine benchmarks demonstrate that VolDy-VAE improves predictive accuracy and uncertainty calibration over competitive baselines, while maintaining low inference latency suitable for real-time deployment.
\end{itemize}

\section{Related Works}
\label{sec:related_works}

\subsection{Heteroscedasticity and Distributional Forecasting}
Quantifying the stochastic uncertainty inherent in future states, specifically the time-varying volatility (heteroscedasticity), has long been a central pursuit in probabilistic forecasting.
Research in this domain can be broadly categorized into deep autoregressive models, distributional regression, and stochastic process-based approaches.
Recent probabilistic forecasting work has further explored multi-scale attention flows to capture uncertainty at different temporal resolutions~\cite{Feng2024MSAF}.

\subsubsection{Deep Autoregressive Models}
Seminal works like {DeepAR}~\cite{Salinas2020DeepAR} pioneered the use of Recurrent Neural Networks (RNNs) to model hidden state transitions, estimating likelihood parameters at each time step. While initially formulated with Gaussian likelihoods, subsequent extensions in libraries like {GluonTS}~\cite{alexandrov2020gluonts} have incorporated heavy-tailed distributions to enhance robustness. Similarly, {DeepState}~\cite{rangapuram2018deep} hybridizes state space models with RNNs.
In multivariate forecasting, recent TKDE studies also model inter-variable dependencies through dynamic graph neural ODEs~\cite{Jin2023DGODE} or lightweight correlated-series distillation~\cite{Lai2024LightCTS}.
These models do maintain recurrent sequence states, and therefore should not be confused with purely feed-forward heteroscedastic heads. Their recurrence, however, is used primarily for autoregressive value generation and likelihood-parameter emission rather than for a dedicated volatility state transferred from the look-back window to a directly generated future horizon. VolDy therefore differs from RNN-based variance emission: it combines look-back volatility-state transfer, non-autoregressive horizon generation, and scale-only recurrent memory within a forecasting VAE. Moreover, recursive decoding precludes parallelization and can suffer from error accumulation over long forecasting horizons.

\subsubsection{Distributional Regression and MDNs}
Beyond simple parametric shapes, advanced distributional regression techniques have been adapted for forecasting. {GAMLSS}~\cite{rigby2005generalized} provides a rigorous statistical framework for modeling location, scale, and shape parameters. To capture multimodal densities, Mixture Density Networks ({MDNs})~\cite{bishop1994mixture} output a weighted combination of distributions. While powerful, MDNs are notoriously difficult to train due to numerical instabilities and mode collapse, and they are typically applied in step-wise regression settings rather than global sequence generation.

\subsubsection{Gaussian Processes and Volatility Models}
Bayesian non-parametric methods offer principled uncertainty quantification.
{Heteroscedastic Gaussian Processes}~\cite{lazaro2011variational} explicitly model {input-dependent volatility}.
However, GP-based methods generally suffer from cubic computational complexity $\mathcal{O}(N^3)$, limiting their scalability to long multivariate time series.
In the econometrics literature, the GARCH family~\cite{bollerslev1986generalized} provides a principled framework for modeling temporal volatility persistence through autoregressive variance recursions.
However, classical GARCH models are constrained to linear variance dynamics with globally fixed coefficients, operate on scalar residuals rather than rich latent representations, and are primarily designed for univariate financial return series rather than general-purpose multivariate long-horizon forecasting.

\subsection{Generative Forecasting Models}
To overcome the expressiveness limitations of parametric models, deep generative models have become widely used in forecasting. Related non-generative forecasting advances also exploit time-frequency multi-resolution analysis~\cite{Yan2024MREA} and prompt-based learning paradigms~\cite{Xue2024PromptCast}. We review three mainstream generative categories: iterative diffusion models, emerging state space models, and variational latent variable models.

\subsubsection{Iterative Denoising Models}
Denoising Diffusion Probabilistic Models (DDPMs), such as {TimeGrad}~\cite{Rasul2021TimeGrad} and {CSDI}~\cite{Tashiro2021CSDI}, formulate forecasting as a conditional denoising process. Recent advancements like {NsDiff}~\cite{Ye2025NsDiff} and {DiffusionTS}~\cite{yuan2024diffusionts} further adapt this framework to handle non-stationary characteristics. While diffusion models offer superior distributional expressiveness, they face a severe {efficiency bottleneck}: the requirement for dozens or hundreds of iterative denoising steps renders them computationally expensive for real-time applications.

\subsubsection{State Space Models (SSMs)}
Recently, Structured State Space Models like {Mamba}~\cite{mamba} have emerged, offering linear scaling complexity $\mathcal{O}(L)$. Approaches such as {Mamba-ProbTSF}~\cite{pessoa2025mamba} attempt to adapt this architecture for probabilistic tasks. While highly efficient for long sequences, these models typically function as generic sequence encoders. Unlike dedicated generative frameworks, they often lack explicit mechanisms to structurally decouple signal from noise, potentially leading to over-smoothed uncertainty estimates when confronting rapid regime changes or extreme volatility.
Recent multivariate forecasting analysis has also shown that channel-independent designs involve a capacity--robustness trade-off~\cite{Han2024Capacity}, further highlighting the need to distinguish generic sequence modeling choices from mechanisms that explicitly encode uncertainty dynamics.

\subsubsection{Variational Autoencoders (VAEs) and Temporal Uncertainty}
VAEs offer a compelling balance between capability and efficiency via {one-step generation}.
Standard forecasting approaches, such as {TimeVAE}~\cite{desai2021timevae} and the state-of-the-art {$K^2$VAE}~\cite{Wu2025K2VAE}, excel in modeling long-term nonlinearity but predominantly rely on MSE-based reconstruction, implicitly assuming a {fixed-variance Gaussian likelihood}.
While models like {HeTVAE}~\cite{HeTVAE} have integrated heteroscedastic modeling into the VAE framework, they are designed for different tasks, namely {imputation and interpolation of irregularly sampled data}, requiring complex attention mechanisms and ODE solvers that are neither optimized for nor applicable to efficient long-horizon forecasting.
This leaves a gap for a framework that explicitly models the temporal dynamics of uncertainty (volatility evolution, regime persistence) while retaining efficient long-horizon generation.

\subsection{Positioning: Temporal Uncertainty Dynamics}

To clarify the novelty of our approach relative to representative heteroscedastic and probabilistic forecasting families, we provide a structured comparison in Table~\ref{tab:positioning}.
Existing methods cover different parts of the design space: autoregressive distributional forecasters maintain generic recurrent states, diffusion models provide expressive sample generation, and forecasting VAEs offer efficient direct generation. VolDy targets the less explored intersection of: (1) an explicit volatility state for modeling how uncertainty evolves across time, (2) forecasting-specialized long-horizon prediction, and (3) efficient direct generation.

\begin{table}[!t]
    \centering
    \caption{Positioning of VolDy relative to representative heteroscedastic and probabilistic forecasting families. ``Explicit Vol. State'' indicates whether the approach uses a dedicated state for the evolution and persistence of volatility regimes, as opposed to generic sequence states or independent per-step variance estimation. $\triangle$ denotes partial support through a generic recurrent or generative state rather than an explicit volatility state. ``Plug-in Evidence'' indicates whether the temporal uncertainty mechanism is evaluated beyond a single base architecture.}
    \label{tab:positioning}
    \scriptsize
    \setlength{\tabcolsep}{3pt}
    \begin{tabular}{l|cccc}
        \toprule
        {Approach} & {Forecasting} & {Explicit Vol.} & {Direct} & {Plug-in} \\
         & {Task} & {State} & {Generation} & {Evidence} \\
        \midrule
        Kendall \& Gal~\cite{kendall2017uncertainties} & Regression & \ding{55} & -- & \ding{55} \\
        HeTVAE~\cite{HeTVAE} & Imputation & \ding{55} & \ding{51} & \ding{55} \\
        DeepAR/DeepState~\cite{Salinas2020DeepAR,rangapuram2018deep} & Forecasting & $\triangle$ & \ding{55} & \ding{55} \\
        Diffusion PTSF~\cite{Rasul2021TimeGrad,Tashiro2021CSDI,Ye2025NsDiff} & Forecasting & $\triangle$ & \ding{55} & \ding{55} \\
        $K^2$VAE~\cite{Wu2025K2VAE} & Forecasting & \ding{55} & \ding{51} & \ding{55} \\
        \midrule
        {VolDy principle/plug-in} & {Forecasting} & {\ding{51}} & {\ding{51}} & {\ding{51}} \\
        \bottomrule
    \end{tabular}
\end{table}

Our proposed {VolDy-VAE} addresses this gap by introducing a volatility dynamics module within a VAE forecasting framework, where the recurrent scale head captures the temporal coherence of uncertainty regimes.
This positions our work as a forecasting framework for temporal uncertainty dynamics rather than a direct application of heteroscedastic regression.

\section{Methodology}
\label{sec:method}
\begin{figure*}[t]
    \centering
    \includegraphics[width=\linewidth]{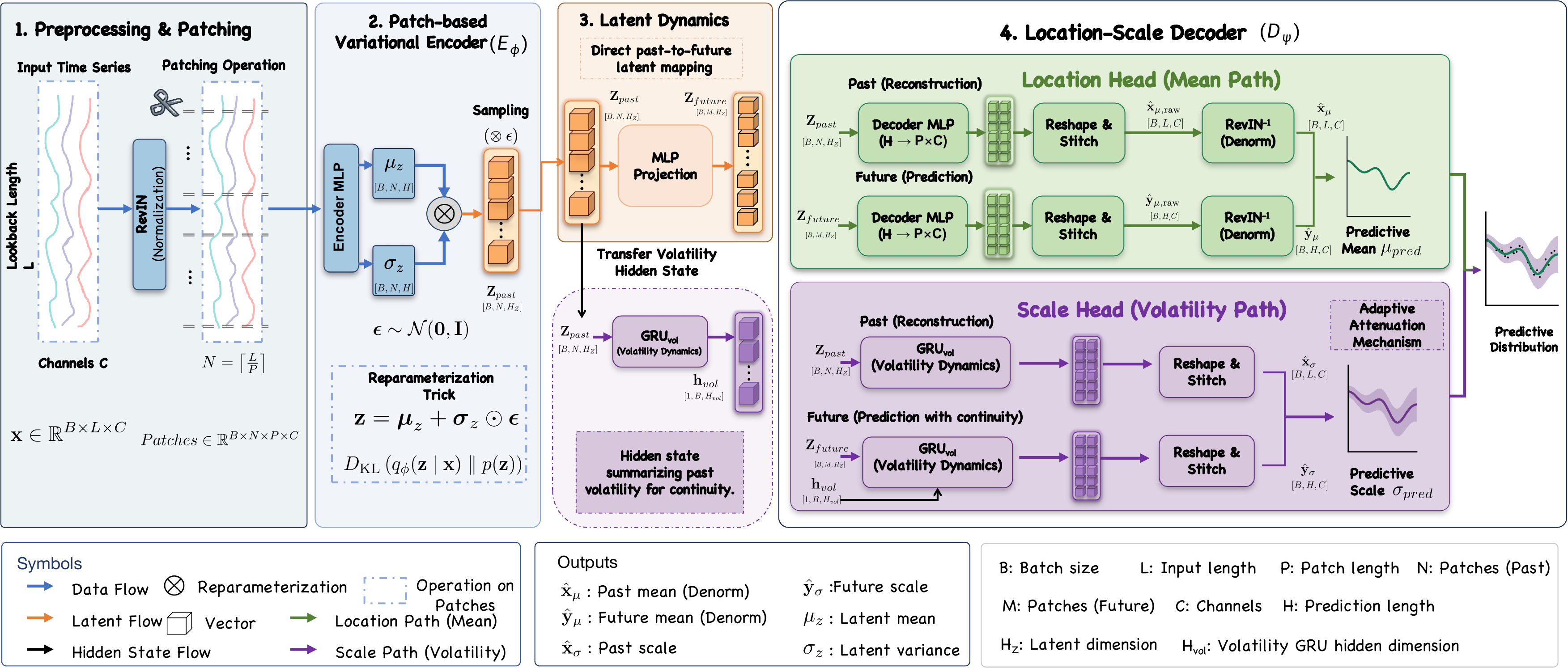}
    \caption{ {Overview of the VolDy-VAE Framework.} The model normalizes and segments the input series into patches, encodes them into a probabilistic latent space, directly maps past latent patches to future latent patches, and decodes both reconstruction and prediction through a location-scale decoder. The location path predicts the denormalized mean, while the recurrent scale path transfers the past volatility hidden state to the future horizon so that the predictive scale evolves coherently across the look-back/prediction boundary.}
    \label{fig:architecture}
\end{figure*}

\subsection{Problem Formulation and Motivation}
\label{subsec:problem_formulation}

\subsubsection{Probabilistic Forecasting Formulation}
We consider a multivariate time series forecasting scenario. Let $\boldsymbol{\mathcal{X}} = \{\mathbf{x}_t\}_{t=1}^L$ denote the historical look-back window of length $L$, where $\mathbf{x}_t \in \mathbb{R}^C$ represents the observed values of $C$ variables at time step $t$.
The objective is to model the conditional probability density of the future horizon $\boldsymbol{\mathcal{Y}} = \{\mathbf{x}_t\}_{t=L+1}^{L+H}$ of length $H$:
\begin{equation}
    p_\theta(\boldsymbol{\mathcal{Y}} | \boldsymbol{\mathcal{X}}).
\end{equation}
Unlike deterministic approaches that yield a single point estimate $\hat{\boldsymbol{\mathcal{Y}}} \approx \mathbb{E}[\boldsymbol{\mathcal{Y}}|\boldsymbol{\mathcal{X}}]$ (often minimizing $L_2$ distance), our goal is to capture the full predictive distribution. This formulation quantifies the aleatoric uncertainty intrinsic to the data generation process, specifically the time-varying volatility (heteroscedasticity).

\subsubsection{The Homoscedastic Limitation in Vanilla VAEs}
We formalize the dual deficiency introduced in Section~\ref{sec:intro}.
In standard time series VAEs~\cite{Kingma2014VAE}, the decoder assumes a fixed-variance Gaussian likelihood $p_\theta(\boldsymbol{\mathcal{X}}| \mathbf{Z}) = \mathcal{N}(\boldsymbol{\mathcal{X}}; \boldsymbol{\mu}_\theta(\mathbf{Z}), \sigma_{const}^2 \mathbf{I})$.
Maximizing the ELBO then degenerates into the MSE objective:
\begin{align}
    \log p_\theta(\boldsymbol{\mathcal{X}}| \mathbf{Z})
    &\propto - \sum_{t} \|\mathbf{x}_t - \boldsymbol{\mu}_t(\mathbf{Z})\|_2^2,
\end{align}
which imposes a homoscedastic misspecification ($\sigma^2 = \text{const}$) whenever predictive uncertainty is time-varying.
Even when this is relaxed by adding a learned variance head, existing methods estimate $\sigma_t$ via feed-forward projections at each step independently, failing to model the temporal coherence of volatility regimes.
Our framework addresses both levels: a heteroscedastic NLL objective removes the homoscedastic constraint, and a recurrent scale head captures how volatility evolves across time.

\subsection{The VolDy-VAE Framework}
We present the Volatility Dynamics Variational Autoencoder (VolDy-VAE), a generative framework designed to model temporal uncertainty dynamics, including the temporal evolution, persistence, and regime transitions of volatility. As depicted in Fig.~\ref{fig:architecture}, the architecture is composed of three modules: (1) a Patch-based Variational Encoder responsible for extracting local semantic contexts; (2) an Efficient Latent Dynamics module that directly maps past latent patches to future latent patches; and (3) a Location-Scale Decoder with Volatility Dynamics that parameterizes the final predictive distribution while capturing the temporal structure of uncertainty.

\subsubsection{Patching and Variational Encoding}
To alleviate the computational bottlenecks of point-wise processing and to better capture local semantic structures, we employ a patching strategy~\cite{Nie2023PatchTST}.
We first apply {Reversible Instance Normalization (RevIN)} to the input $\boldsymbol{\mathcal{X}}$ to mitigate non-stationarity~\cite{Kim2022RevIN}.
The normalized series is then segmented into $N$ non-overlapping patches $\boldsymbol{\mathcal{P}} \in \mathbb{R}^{N \times P \times C}$, where $P$ denotes the patch length.

The encoder $E_\phi$ maps these patches into a stochastic latent space. Following the variational autoencoder framework~\cite{Kingma2014VAE}, we model the posterior distribution of the past latent states $\mathbf{Z}_{past}$ as a diagonal Gaussian to account for epistemic uncertainty arising from noisy observations:
\begin{equation}
    q_\phi(\mathbf{Z}_{past} \mid \boldsymbol{\mathcal{P}}) = \mathcal{N}(\boldsymbol{\mu}_z, \boldsymbol{\sigma}_z^2).
\end{equation}
Here, $(\boldsymbol{\mu}_z, \boldsymbol{\sigma}_z) = E_\phi(\boldsymbol{\mathcal{P}})$. Utilizing the reparameterization trick~\cite{Kingma2014VAE}, we sample $\mathbf{Z}_{past} = \boldsymbol{\mu}_z + \boldsymbol{\epsilon} \odot \boldsymbol{\sigma}_z$, where $\boldsymbol{\epsilon} \sim \mathcal{N}(0, \mathbf{I})$. This stochastic injection regularizes the latent representation and improves robustness to input perturbations.

\subsubsection{Efficient Latent Dynamics for Direct Forecast Generation}
Modeling the transition from past latent representations $\mathbf{Z}_{past}$ to future representations $\mathbf{Z}_{future}$ is an important step. While autoregressive observation decoders generate future values recursively, they are prone to error accumulation and suffer from high latency in long-horizon forecasting~\cite{Bengio2015ExposureBias}. To address this limitation, we use a lightweight direct latent dynamics mechanism.

We adopt the modeling assumption that the high-level semantic evolution of a time series can be approximated by a global transformation in the latent space. Specifically, we flatten the past latent representations and project them directly to the future horizon:
\begin{equation}
    \mathbf{Z}_{future} = \mathrm{Projection}(\mathrm{Flatten}(\mathbf{Z}_{past})) \in \mathbb{R}^{M \times D},
\end{equation}
where $M$ denotes the number of future patches. This direct latent projection avoids feeding generated observations back into the model, enabling the entire future latent horizon to be produced in a single forward pass. The design provides the efficiency benefits commonly associated with non-autoregressive forecasting, while the subsequent recurrent scale head is reserved specifically for modeling volatility memory rather than recursive value generation. This separation reduces the computational bottleneck of recursive generation and improves inference efficiency compared to autoregressive probabilistic forecasters such as DeepAR~\cite{Salinas2020DeepAR} and TimeGrad~\cite{Rasul2021TimeGrad}.

\subsubsection{Location-Scale Decoding with Volatility Dynamics}
\label{sec:loc_scale_vol}

This module is the main mechanism for modeling {temporal uncertainty dynamics}.
Unlike standard heteroscedastic regression that estimates $\sigma_t$ independently at each time step, our decoder is designed to capture the temporal evolution and persistence of volatility regimes.
We introduce a {shared} Dual-Head Decoder $D_\psi$ that operates with {identical weights} on both the historical latent representations $\mathbf{Z}_{past}$ and the future latent representations $\mathbf{Z}_{future}$, ensuring a consistent mapping from latent space to the observation space.
The decoder splits into two parallel pathways from the latent representations $\mathbf{Z}$:

\textbf{Location Head.}
An MLP projects $\mathbf{Z}$ to predict the central tendency (trend and seasonality). The inverse RevIN operation restores the original scale:
\begin{equation}
    \boldsymbol{\mu} = \text{RevIN}^{-1}\left(\text{MLP}_{\mu}(\mathbf{Z})\right).
\end{equation}

\textbf{Scale Head with Volatility Dynamics.}
A separate recurrent pathway takes the same $\mathbf{Z}$ as input but models the temporal coherence of volatility through a recurrent hidden state, rather than predicting $\sigma_t$ independently from each $\mathbf{Z}_t$. Specifically:
\begin{align}
    \mathbf{h}^{\mathrm{vol}}_t &= \mathrm{GRU}_{\mathrm{vol}}(\mathbf{h}^{\mathrm{vol}}_{t-1},\; \mathbf{Z}_t), \\
    \boldsymbol{\sigma}_t &= \text{Softplus}\left(\text{Linear}(\mathbf{h}^{\mathrm{vol}}_t)\right) + \xi,
\end{align}
where $\mathrm{GRU}_{\mathrm{vol}}$ is a single-layer GRU that maintains a volatility hidden state $\mathbf{h}^{\mathrm{vol}}_t$, $\xi$ is a small stability constant (e.g., $10^{-6}$), and the subscript $t$ indexes patches rather than individual time steps: the GRU operates over the $M$ patch-level latent vectors $\{\mathbf{Z}_1, \ldots, \mathbf{Z}_M\}$, producing one scale vector per patch that is then reshaped to yield per-step $\sigma$ values within each patch, matching the granularity of the location head.
For prediction, the same $\mathrm{GRU}_{\mathrm{vol}}$ first summarizes the historical latent sequence $\mathbf{Z}_{past}$ into a final volatility state and uses this state to initialize the future scale path over $\mathbf{Z}_{future}$; this hidden-state transfer preserves volatility continuity across the boundary between reconstruction and forecasting.
This recurrent formulation ensures that the predicted scale at patch $t$ depends not only on the current latent state but also on the history of volatility evolution, enabling the model to capture volatility clustering, regime persistence, and smooth transitions between uncertainty regimes.
This design is motivated by the volatility clustering phenomenon widely observed in real-world time series: periods of high volatility tend to be followed by high volatility, and low-volatility periods tend to persist~\cite{Gneiting2007ProbForecast}.

\textbf{Connection to Classical Volatility Models.}
The recurrent volatility dynamics module can be understood as a learned, nonlinear generalization of the Generalized Autoregressive Conditional Heteroscedasticity (GARCH) family~\cite{bollerslev1986generalized}.
A standard GARCH(1,1) model defines the conditional variance as $\sigma_t^2 = \omega + \alpha \epsilon_{t-1}^2 + \beta \sigma_{t-1}^2$, where the current variance depends linearly on the previous residual and variance.
Our GRU-based formulation extends this principle in three key aspects:
(i)~the gating mechanism of the GRU (reset and update gates) implicitly learns adaptive mixing coefficients $\alpha, \beta$ that are input-dependent rather than globally fixed, enabling the model to modulate volatility persistence differently across regimes;
(ii)~the nonlinear state transitions allow capturing complex volatility dynamics beyond the linear recursion of GARCH, such as nonlinear regime transitions and context-dependent volatility responses;
and (iii)~the volatility state is conditioned on the rich latent representation $\mathbf{Z}_t$ rather than scalar residuals, enabling the model to leverage multi-scale contextual information for uncertainty estimation.
This connection provides theoretical grounding for our architectural choice: the GRU naturally encodes the autoregressive persistence structure of volatility while offering the flexibility to capture nonlinear regime transitions that classical parametric models cannot represent.

\textbf{Role of Temporal Volatility Dynamics.}
A natural question is why estimating $\sigma_t$ independently at each time step, as in standard heteroscedastic regression, is insufficient.
Let $\{\hat{\sigma}_t^{\mathrm{ind}}\}$ denote variance estimates from an independent per-step estimator and $\{\hat{\sigma}_t^{\mathrm{dyn}}\}$ denote estimates from a recurrent volatility dynamics module.
A recurrent estimator does not unconditionally dominate an independent estimator.
However, under the common regime-switching pattern where the true volatility is locally persistent and the recurrent state learns to pool information from adjacent steps without oversmoothing regime boundaries, temporal sharing can reduce the variance of scale estimates within a stable regime.
An independent estimator must infer $\sigma_t$ from a single latent vector $\mathbf{z}_t$, treating each step as an isolated regression problem.
In contrast, the recurrent estimator uses the volatility hidden state $\mathbf{h}_t^{\mathrm{vol}}$ to aggregate evidence from neighboring latent states.
When adjacent steps share the same volatility regime, this pooling behaves like a learned temporal smoother and reduces step-wise scale noise.
At regime boundaries, the GRU gates can limit smoothing lag by adapting how much past volatility state is retained.
This conditional intuition is consistent with the synthetic experiments in Section~\ref{sec:analytical_validation}, where the predicted $\hat{\sigma}_t$ tracks the ground-truth volatility with high temporal coherence ($\rho > 0.98$) rather than producing noisy per-step estimates.

\subsubsection{Objective Function: Volatility-Aware Forecasting}
The training objective is to maximize the Evidence Lower Bound (ELBO). Unlike many forecasting-oriented VAE variants that rely on MSE-based reconstruction objectives, we reformulate the objective to explicitly model temporal uncertainty dynamics in both reconstruction and prediction phases.

\textbf{The Composite Location-Scale Loss.}
We adopt the Gaussian Negative Log-Likelihood (NLL) for both the look-back window reconstruction ($\mathcal{L}_{rec}$) and the future horizon prediction ($\mathcal{L}_{pred}$). The total loss function $\mathcal{L}$ is defined as:
\begin{align}
    \mathcal{L} &= \underbrace{\mathcal{L}_{rec}(\boldsymbol{\mathcal{X}})}_{\text{Reconstruction}} + \underbrace{\mathcal{L}_{pred}(\boldsymbol{\mathcal{Y}})}_{\text{Prediction}} \nonumber \\
    &\quad + \beta \cdot D_{KL}\left(q_\phi(\mathbf{Z}_{past} \mid \boldsymbol{\mathcal{P}}) \parallel \mathcal{N}(0, \mathbf{I})\right).
\end{align}

Both $\mathcal{L}_{rec}$ and $\mathcal{L}_{pred}$ are computed using the heteroscedastic Gaussian NLL. For a generic sequence $\mathbf{u}$ of length $T$ (where $\mathbf{u}$ is $\boldsymbol{\mathbf{x}}$ or $\boldsymbol{\mathbf{y}}$), the loss is:
\begin{equation}
\mathcal{L}_{\text{NLL}}(\mathbf{u}) = \frac{1}{T C} \sum_{t,c} \left( \log \boldsymbol{\sigma}_{t,c} + \frac{(\mathbf{u}_{t,c} - \boldsymbol{\mu}_{t,c})^2}{2\boldsymbol{\sigma}_{t,c}^2} \right).
\label{eq:ob}
\end{equation}
By jointly minimizing $\mathcal{L}_{rec}$ and $\mathcal{L}_{pred}$, the model is forced to learn a latent representation that is both descriptive of historical volatility patterns and predictive of future uncertainty dynamics.

\textbf{Theoretical Analysis (Adaptive Attenuation and Temporal Robustness).}
The gradient of the NLL term reveals an adaptive attenuation mechanism:
\begin{equation}
    \nabla_{\boldsymbol{\mu}} \mathcal{L}_{NLL} \propto \frac{1}{\boldsymbol{\sigma}^2}(\mathbf{u} - \boldsymbol{\mu}).
\end{equation}
When the model encounters an outlier or a volatile regime, the volatility dynamics module learns to increase $\boldsymbol{\sigma}$, effectively down-weighting the gradient contribution of that point ($\frac{1}{\boldsymbol{\sigma}^2} \to 0$) and preventing the location head from overfitting to {transient stochastic fluctuations}.
Because $\sigma_t$ is produced by the recurrent volatility module rather than estimated independently, this attenuation is temporally coherent: the model transitions between high-attention (low $\sigma$) and low-attention (high $\sigma$) regimes in a manner consistent with real-world volatility.

Beyond this gradient-level effect, a simplified regime-switching analysis in \suppSectionRef{app:consistency_proof} shows that, when time-varying variances are known or consistently estimated, the Gaussian NLL objective reduces to inverse-variance weighting, whereas MSE remains statistically inefficient under heteroscedasticity.
We use this result as a statistical rationale for the volatility-aware objective rather than as a proof of global optimality for the full neural forecasting model.
It explains why accurate scale estimates can make location learning less sensitive to high-volatility observations, which is consistent with the empirical behavior observed below.

\section{Experiments on Synthetic Data}
\label{sec:analytical_validation}

In real-world empirical studies, the true aleatoric uncertainty is inherently latent and unobservable, making it difficult to rigorously verify whether a model captures volatility dynamics or merely overfits residual artifacts~\cite{kendall2017uncertainties}.
To address this, we validate the temporal uncertainty modeling capability of VolDy-VAE within a {controlled analytical environment} where both the underlying trend and the {volatility profile} are mathematically defined and known~\cite{Gal2016Dropout}.
This experiment tests whether the model can disentangle the deterministic signal (Location Head) from temporally evolving volatility (Scale Head with volatility dynamics).

\begin{figure}[t]
    \centering
    \includegraphics[width=\linewidth]{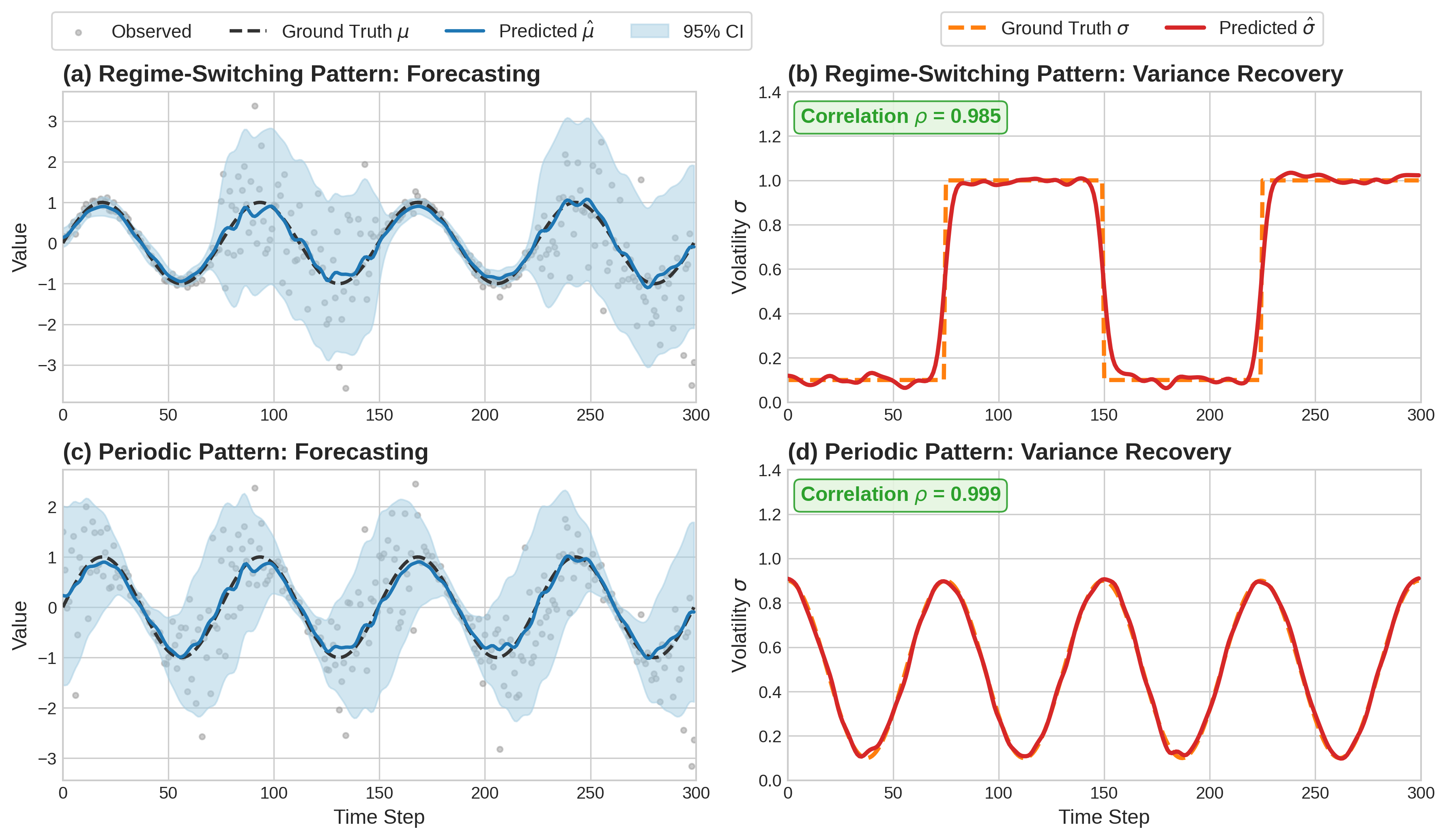}
    \caption{Qualitative results on synthetic datasets.
Left: forecasting with adaptive 95\% confidence intervals under varying volatility.
Right: uncertainty recovery, where the predicted volatility $\hat{\sigma}_t$ closely tracks the ground-truth temporal dynamics.}
    \label{fig:toy_results}
\end{figure}

\subsection{Data Generating Process}
We analyze univariate sequences governed by the process $y_t = \mu_t + \epsilon_t$, where the underlying deterministic trend is defined as $\mu_t = \sin(t)$ and the {stochastic term} follows $\epsilon_t \sim \mathcal{N}(0, \sigma_t^2)$.
To stress-test the model under {extreme heteroscedasticity}, we introduce two distinct volatility profiles as illustrated in Fig.~\ref{fig:toy_results}.
The first is a {Regime-Switching} scenario (Top Row), which simulates structural breaks by having $\sigma_t$ alternate abruptly between stable ($0.1$) and high-volatility ($1.0$) states, thereby challenging the model's adaptability to sudden distributional shifts.
The second is a {Periodic} pattern (Bottom Row), defined by a continuous volatility trajectory $\sigma_t = 0.5 + 0.4\cos(t)$; this profile is deliberately phase-shifted relative to the mean to rigorously test the model's ability to decouple correlated but distinct dynamics.

\subsection{Results and Analysis}
Fig.~\ref{fig:toy_results} visualizes the forecasting and variance recovery performance.
In the Regime-Switching scenario (a-b), the predicted scale $\hat{\sigma}_t$ closely tracks the ground-truth step function with high correlation ($\rho=0.985$). This estimation allows the confidence intervals to expand during volatility spikes and contract during stable periods without distorting the trend prediction.
Similarly, in the Periodic scenario (c-d), the model achieves near-perfect disentanglement ($\rho=0.999$), effectively distinguishing the signal's phase from the cyclic volatility profile.
These results show that the volatility dynamics module captures the temporal evolution of uncertainty: the predicted $\hat{\sigma}_t$ transitions smoothly between regimes rather than producing noisy per-step estimates.

\begin{table}[!t]
    \centering
    \caption{Statistics of the multivariate time series datasets.}
    \label{tab:dataset_stats}
    \scriptsize
    \setlength{\tabcolsep}{3pt}
    \resizebox{\columnwidth}{!}{
    \begin{tabular}{l|c|c|c|c}
        \toprule
        \textbf{Dataset} & \textbf{Features ($C$)} & \textbf{Timesteps} & \textbf{Frequency} & \textbf{Domain} \\
        \midrule
        ETTh1 & 7 & 17{,}420 & 1 Hour & Energy \\
        ETTh2 & 7 & 17{,}420 & 1 Hour & Energy \\
        ETTm1 & 7 & 69{,}680 & 15 Min & Energy \\
        ETTm2 & 7 & 69{,}680 & 15 Min & Energy \\
        Electricity & 321 & 26{,}304 & 1 Hour & Energy \\
        Traffic & 862 & 17{,}544 & 1 Hour & Transport \\
        Weather & 21 & 52{,}696 & 10 Min & Weather \\
        Exchange & 8 & 7{,}588 & 1 Day & Economics \\
        ILI & 7 & 966 & 1 Week & Health \\
        \bottomrule
    \end{tabular}
    }
\end{table}

\section{Experiments on Real-World Data}
\label{sec:experiments}

We evaluate the effectiveness, robustness, and generalizability of VolDy-VAE through experiments on benchmark forecasting tasks.
Our evaluation focuses on large-scale real-world probabilistic time series forecasting tasks, followed by ablation studies, qualitative analysis, generalization experiments, and efficiency comparisons.

We conduct experiments under the ProbTS/$K^2$VAE protocol~\cite{Zhang2024ProbTS,Wu2025K2VAE}, a benchmark covering both point and probabilistic forecasters across diverse long-horizon settings.
Under this unified protocol, we compare against 12 competitive baselines, including NsDiff~\cite{Ye2025NsDiff} as an additional diffusion-based baseline.

\subsection{Evaluation under the \texorpdfstring{$K^2$VAE}{K2VAE} Protocol}
\label{Setting_I}

\subsubsection{Datasets} We conduct experiments on nine widely-recognized real-world benchmarks from ProbTS~\cite{Zhang2024ProbTS}, spanning diverse domains: Energy (\textit{ETTh1}, \textit{ETTh2}, \textit{ETTm1}, \textit{ETTm2}, \textit{Electricity}), Transportation (\textit{Traffic}), Weather (\textit{Weather}), Economics (\textit{Exchange}), and Health (\textit{ILI}).
Following the standard protocol for long-term forecasting, we set the look-back window length to $L=96$ for all datasets, except for \textit{ILI} where $L=36$.
The forecasting horizons are set to $H \in \{96, 192, 336, 720\}$ for the general datasets and $H \in \{24, 36, 48, 60\}$ for \textit{ILI}.
Table~\ref{tab:dataset_stats} summarizes the dataset statistics.
\FloatBarrier

\subsubsection{Baselines}
We compare VolDy-VAE with  {12 competitive baselines}, covering both
deterministic and probabilistic forecasting models.
Specifically, the selected baselines include deep point forecasters such as FITS~\cite{Xu2023FITS},
PatchTST~\cite{Nie2023PatchTST}, iTransformer~\cite{Liu2024iTransformer}, and Koopa~\cite{Liu2023Koopa},
which are originally designed for deterministic prediction.
For fair probabilistic evaluation, these models are augmented with a Gaussian
output head to predict distribution parameters $(\mu, \sigma)$.
All augmented point forecasters are trained with the same Gaussian NLL objective used for probabilistic evaluation, and the reported numbers are obtained from six independent retraining-and-evaluation runs under the same data splits and forecasting horizons.
In addition, we consider a broad range of probabilistic and generative models,
including TSDiff~\cite{Kollovieh2023TSDiff}, TimeGrad~\cite{Rasul2021TimeGrad},
CSDI~\cite{Tashiro2021CSDI}, NsDiff~\cite{Ye2025NsDiff}, GRU NVP, GRU MAF, Trans MAF~\cite{Trans-MAF}, as well as $K^2$VAE~\cite{Wu2025K2VAE}.
\subsubsection{Implementation Details}
\label{sec:implementation}
VolDy-VAE adopts the same patching strategy (patch length $P=24$) and MLP-based encoder as $K^2$VAE~\cite{Wu2025K2VAE} for fair comparison.
The decoder features a dual-head structure: a Location Head with linear projection and a Scale Head with a single-layer GRU (hidden dimension $= D$) followed by \texttt{Softplus} activation.
In \suppTableRef{tab:hyper_sensitivity_specific}, the layer number $N$ denotes the depth of the MLP encoder, not the number of GRU layers in the scale head.
We implement the model in PyTorch and train on a single NVIDIA RTX 4090 GPU using the Adam optimizer (learning rate $\in [10^{-4}, 10^{-3}]$, batch size 32, max 50 epochs with early stopping patience of 10).

\subsubsection{Metrics} We employ two complementary metrics: {Continuous Ranked Probability Score (CRPS)} to measure the distributional fit and {Normalized Mean Absolute Error (NMAE)} to assess point prediction accuracy. CRPS is a strictly proper scoring rule and is applicable to both closed-form Gaussian outputs and sample-based probabilistic baselines, making it suitable for cross-family comparison where likelihood values are not uniformly available. Lower values indicate better performance for both metrics. Detailed descriptions of these metrics can
be found in \suppSectionRef{appendix:metrics}.

\begin{table*}[t]
\caption{Comparison on long-term probabilistic forecasting (forecasting horizon $H=720$) scenarios across nine real-world datasets. Lower CRPS or NMAE values indicate better predictions. The reported values are means with standard deviations over 6 independent retraining-and-evaluation runs. \textcolor{red}{RED}: THE BEST, \textcolor{blue}{\underline{BLUE}}: THE 2ND BEST. The full results of all four horizons {96, 192, 336, 720} are listed in the supplementary material.}
\centering
\scriptsize
\setlength{\tabcolsep}{4pt}
\renewcommand{\arraystretch}{1.1}
\begin{tabular}{c|c|ccccccccc}
\toprule
Model & Metric & ETTm1  & ETTm2  & ETTh1  & ETTh2  & Electricity  & Traffic  & Weather  & Exchange  & ILI  \\
\midrule
\multirow{2}{*}{FITS} & CRPS
& $0.305_{\pm.024}$ & $0.449_{\pm.034}$ & $0.348_{\pm.025}$ & $0.314_{\pm.022}$ & $0.115_{\pm.024}$ & $0.374_{\pm.004}$ & $0.267_{\pm.003}$ & $0.074_{\pm.011}$ & $0.211_{\pm.011}$ \\
~ & NMAE
& $0.406_{\pm.072}$ & $0.540_{\pm.052}$ & $0.468_{\pm.012}$ & $0.401_{\pm.022}$ & $0.149_{\pm.012}$ & $0.453_{\pm.022}$ & $0.317_{\pm.021}$ & $0.097_{\pm.011}$ & $0.245_{\pm.017}$ \\
\hline
\multirow{2}{*}{PatchTST} & CRPS
& $0.304_{\pm.029}$ & $0.229_{\pm.036}$ & $0.323_{\pm.020}$ & $0.304_{\pm.018}$ & $0.127_{\pm.015}$ & $0.214_{\pm.001}$ & $0.142_{\pm.005}$ & $0.097_{\pm.007}$ & $0.233_{\pm.019}$ \\
~ & NMAE
& $0.382_{\pm.066}$ & $0.288_{\pm.034}$ & $0.428_{\pm.024}$ & $0.371_{\pm.021}$ & $0.164_{\pm.024}$ & $0.253_{\pm.012}$ & $0.152_{\pm.029}$ & $0.126_{\pm.001}$ & $0.287_{\pm.023}$ \\
\hline
\multirow{2}{*}{iTransformer} & CRPS
& $0.455_{\pm.021}$ & $0.311_{\pm.024}$ & $0.350_{\pm.019}$ & $0.542_{\pm.015}$ & $0.109_{\pm.044}$ & $0.284_{\pm.004}$ & $0.133_{\pm.004}$ & $0.087_{\pm.023}$ & $0.222_{\pm.020}$ \\
~ & NMAE
& $0.490_{\pm.038}$ & $0.385_{\pm.042}$ & $0.449_{\pm.022}$ & $0.667_{\pm.012}$ & $0.140_{\pm.009}$ & $0.361_{\pm.030}$ & $0.147_{\pm.019}$ & $0.113_{\pm.015}$ & $0.278_{\pm.017}$ \\
\hline
\multirow{2}{*}{Koopa} & CRPS
& $0.295_{\pm.027}$ & $0.233_{\pm.025}$ & $0.318_{\pm.009}$ & $0.293_{\pm.026}$ & $0.113_{\pm.018}$ & $0.358_{\pm.022}$ & $0.140_{\pm.007}$ & $0.091_{\pm.012}$ & $0.228_{\pm.022}$ \\
~ & NMAE
& $0.377_{\pm.037}$ & $0.290_{\pm.033}$ & $0.412_{\pm.008}$ & $0.286_{\pm.042}$ & $0.149_{\pm.025}$ & $0.432_{\pm.032}$ & $0.162_{\pm.009}$ & $0.116_{\pm.022}$ & $0.288_{\pm.031}$ \\
\hline
\multirow{2}{*}{TSDiff} & CRPS
& $0.478_{\pm.027}$ & $0.344_{\pm.046}$ & $0.516_{\pm.027}$ & $0.406_{\pm.056}$ & $0.478_{\pm.005}$ & $0.391_{\pm.002}$ & $0.152_{\pm.003}$ & $0.082_{\pm.010}$ & $0.263_{\pm.022}$ \\
~ & NMAE
& $0.622_{\pm.045}$ & $0.416_{\pm.065}$ & $0.657_{\pm.017}$ & $0.482_{\pm.022}$ & $0.622_{\pm.142}$ & $0.478_{\pm.006}$ & $0.141_{\pm.026}$ & $0.142_{\pm.009}$ & $0.272_{\pm.020}$ \\
\hline
\multirow{2}{*}{GRU NVP} & CRPS
& $0.546_{\pm.036}$ & $0.561_{\pm.273}$ & $0.502_{\pm.039}$ & $0.539_{\pm.090}$ & $0.114_{\pm.013}$ & $0.211_{\pm.004}$ & $0.110_{\pm.004}$ & $0.079_{\pm.009}$ & $0.307_{\pm.005}$ \\
~ & NMAE
& $0.707_{\pm.050}$ & $0.749_{\pm.385}$ & $0.643_{\pm.046}$ & $0.688_{\pm.161}$ & $0.144_{\pm.017}$ & $0.264_{\pm.006}$ & $0.135_{\pm.008}$ & $0.103_{\pm.009}$ & $0.333_{\pm.005}$ \\
\hline
\multirow{2}{*}{GRU MAF} & CRPS
& $0.536_{\pm.033}$ & $0.272_{\pm.029}$ & $0.393_{\pm.043}$ & $0.990_{\pm.023}$ & $0.106_{\pm.007}$ & - & $0.122_{\pm.006}$ & $0.160_{\pm.019}$ & $0.172_{\pm.034}$ \\
~ & NMAE
& $0.711_{\pm.081}$ & $0.355_{\pm.048}$ & $0.496_{\pm.019}$ & $1.092_{\pm.019}$ & $0.136_{\pm.098}$ & - & $0.149_{\pm.034}$ & $0.182_{\pm.010}$ & $0.216_{\pm.014}$ \\
\hline
\multirow{2}{*}{Trans MAF} & CRPS
& $0.688_{\pm.043}$ & $0.355_{\pm.043}$ & $0.363_{\pm.053}$ & $0.327_{\pm.033}$ & - & - & $0.113_{\pm.004}$ & $0.148_{\pm.017}$ & $0.155_{\pm.018}$ \\
~ & NMAE
& $0.822_{\pm.034}$ & $0.475_{\pm.029}$ & $0.455_{\pm.025}$ & $0.412_{\pm.020}$ & - & - & $0.148_{\pm.040}$ & $0.191_{\pm.006}$ & $0.183_{\pm.019}$ \\
\hline
\multirow{2}{*}{TimeGrad} & CRPS
& $0.621_{\pm.037}$ & $0.470_{\pm.054}$ & $0.523_{\pm.027}$ & $0.445_{\pm.016}$ & $0.108_{\pm.003}$ & $0.220_{\pm.002}$ & $0.113_{\pm.011}$ & $0.099_{\pm.015}$ & $0.295_{\pm.083}$ \\
~ & NMAE
& $0.793_{\pm.034}$ & $0.561_{\pm.044}$ & $0.672_{\pm.015}$ & $0.550_{\pm.018}$ & $0.134_{\pm.004}$ & $0.263_{\pm.001}$ & $0.136_{\pm.020}$ & $0.113_{\pm.016}$ & $0.325_{\pm.068}$ \\
\hline
\multirow{2}{*}{CSDI} & CRPS
& $0.448_{\pm.038}$ & $0.239_{\pm.035}$ & $0.528_{\pm.012}$ & $0.302_{\pm.040}$ & - & - & $0.087_{\pm.003}$ & $0.143_{\pm.020}$ & $0.283_{\pm.012}$ \\
~ & NMAE
& $0.578_{\pm.051}$ & $0.306_{\pm.040}$ & $0.657_{\pm.014}$ & $0.382_{\pm.030}$ & - & - & $0.102_{\pm.005}$ & $0.173_{\pm.020}$ & $0.299_{\pm.013}$ \\
\hline
\multirow{2}{*}{NsDiff} & CRPS
& $0.323_{\pm.030}$ & $0.238_{\pm.026}$ & $0.360_{\pm.018}$ & $0.289_{\pm.020}$ & $0.096_{\pm.008}$ & $0.207_{\pm.003}$ & $0.090_{\pm.004}$ & $0.076_{\pm.008}$ & $0.161_{\pm.010}$ \\
~ & NMAE
& $0.410_{\pm.036}$ & $0.295_{\pm.034}$ & $0.455_{\pm.020}$ & $0.322_{\pm.025}$ & $0.126_{\pm.014}$ & $0.256_{\pm.010}$ & $0.106_{\pm.008}$ & $0.094_{\pm.012}$ & $0.190_{\pm.012}$ \\
\hline
\multirow{2}{*}{$K^2$VAE} & CRPS
& $\textcolor{blue}{\underline{0.294}_{\pm.026}}$
& $\textcolor{blue}{\underline{0.221}_{\pm.023}}$
& $\textcolor{blue}{\underline{0.314}_{\pm.011}}$
& $\textcolor{blue}{\underline{0.280}_{\pm.014}}$
& $\textcolor{blue}{\underline{0.087}_{\pm.005}}$
& $\textcolor{blue}{\underline{0.200}_{\pm.001}}$
& $\textcolor{blue}{\underline{0.084}_{\pm.003}}$
& $\textcolor{blue}{\underline{0.069}_{\pm.005}}$
& $\textcolor{blue}{\underline{0.142}_{\pm.008}}$ \\
~ & NMAE
& $\textcolor{blue}{\underline{0.373}_{\pm.032}}$
& $\textcolor{blue}{\underline{0.275}_{\pm.035}}$
& $\textcolor{blue}{\underline{0.396}_{\pm.012}}$
& $\textcolor{blue}{\underline{0.278}_{\pm.020}}$
& $\textcolor{blue}{\underline{0.117}_{\pm.019}}$
& $\textcolor{blue}{\underline{0.248}_{\pm.010}}$
& $\textcolor{blue}{\underline{0.099}_{\pm.009}}$
& $\textcolor{blue}{\underline{0.084}_{\pm.017}}$
& $\textcolor{blue}{\underline{0.167}_{\pm.007}}$ \\
\hline
\multirow{2}{*}{VolDy-VAE$^{*}$} & CRPS
& $\textcolor{red}{0.279_{\pm.020}}$
& $\textcolor{red}{0.166_{\pm.018}}$
& $\textcolor{red}{0.310_{\pm.012}}$
& $\textcolor{red}{0.174_{\pm.010}}$
& $\textcolor{red}{0.080_{\pm.006}}$
& $\textcolor{red}{0.192_{\pm.003}}$
& $\textcolor{red}{0.073_{\pm.003}}$
& $\textcolor{red}{0.065_{\pm.004}}$
& $\textcolor{red}{0.130_{\pm.005}}$ \\
~ & NMAE
& $\textcolor{red}{0.367_{\pm.023}}$
& $\textcolor{red}{0.207_{\pm.026}}$
& $\textcolor{red}{0.389_{\pm.013}}$
& $\textcolor{red}{0.223_{\pm.015}}$
& $\textcolor{red}{0.111_{\pm.012}}$
& $\textcolor{red}{0.236_{\pm.009}}$
& $\textcolor{red}{0.086_{\pm.006}}$
& $\textcolor{red}{0.077_{\pm.012}}$
& $\textcolor{red}{0.155_{\pm.006}}$ \\
\bottomrule
\multicolumn{11}{l}{
\footnotesize
${*}$ indicates that VolDy-VAE significantly outperforms K$^2$VAE
under paired Wilcoxon signed-rank tests ($p<0.05$).
}
\end{tabular}
\label{tab: long-term}
\end{table*}

\subsubsection{Results and Analysis}
Table~\ref{tab: long-term} summarizes the performance for the challenging long-term horizon ($H=720$).
A notable finding is that VolDy-VAE consistently improves over specialized point forecasters, such as PatchTST and iTransformer, even on the deterministic metric ({NMAE}). This observation is consistent with the adaptive attenuation mechanism discussed above: when scale estimates reflect local volatility, the NLL objective can reduce the influence of high-variance observations on the location predictor.
Furthermore, compared to the previous SOTA VAE-based method, $K^2$VAE, our model achieves consistent improvements across all datasets, demonstrating the advantage of explicitly modeling volatility evolution over independent per-step variance estimation.
\subsection{Ablation Studies}
\label{sec:ablation}

To thoroughly investigate the contribution of each component in our training objective, we conduct a component-wise ablation study. The total loss is defined as $\mathcal{L} = \mathcal{L}_{rec} + \mathcal{L}_{pred} + \beta D_{KL}$. We measure performance using {CRPS} (probabilistic calibration) and {NMAE} (deterministic accuracy). We compare the full VolDy-VAE against three distinct variants: {w/o Prediction ($\mathcal{L}_{pred} = 0$)}, which is trained exclusively on reconstructing historical patches to test whether latent dynamics can be learned without future supervision; {w/o Reconstruction ($\mathcal{L}_{rec} = 0$)}, which removes the historical reconstruction loss to verify if the auxiliary reconstruction task aids in regularizing the latent space; and {Vanilla VAE (MSE)}, which replaces the heteroscedastic Location-Scale NLL with a standard homoscedastic MSE loss.

\begin{table}[t]
    \centering
    \caption{Component-wise ablation on \textit{ETTh1} and \textit{Electricity}.
We compare VolDy-VAE with variants without prediction loss (w/o P), without reconstruction loss (w/o R), and a vanilla VAE (MSE).
Metrics: CRPS and NMAE. \textcolor{red}{RED}: THE BEST.}
    \label{tab:ablation_component}
    \scriptsize
    \setlength{\tabcolsep}{3.5pt}
    \begin{tabular}{c|c|cccc|cccc}
        \toprule
        \multirow{2}{*}{{Dataset}} & \multirow{2}{*}{{Horizon}} & \multicolumn{4}{c|}{{CRPS} $\downarrow$} & \multicolumn{4}{c}{{NMAE} $\downarrow$} \\
        \cmidrule(lr){3-6} \cmidrule(lr){7-10}
         & & {Van.} & {w/o P} & {w/o R} & {Ours} & {Van.} & {w/o P} & {w/o R} & {Ours} \\
        \midrule
        \multirow{4}{*}{\rotatebox{90}{{ETTh1}}}
        & 96  & 0.287 & 0.374 & 0.277 & \textcolor{red}{0.252} & 0.364 & 0.464 & 0.354 & \textcolor{red}{0.328} \\
        & 192 & 0.323 & 0.378 & 0.282 & \textcolor{red}{0.277} & 0.388 & 0.470 & 0.370 & \textcolor{red}{0.365} \\
        & 336 & 0.352 & 0.386 & 0.305 & \textcolor{red}{0.300} & 0.409 & 0.486 & 0.395 & \textcolor{red}{0.384} \\
        & 720 & 0.355 & 0.393 & 0.318 & \textcolor{red}{0.310} & 0.416 & 0.504 & 0.398 & \textcolor{red}{0.389} \\
        \midrule
        \midrule
        \multirow{4}{*}{\rotatebox{90}{{Electricity}}}
        & 96  & 0.363 & 0.259 & 0.082 & \textcolor{red}{0.069} & 0.362 & 0.363 & 0.107 & \textcolor{red}{0.089} \\
        & 192 & 0.363 & 0.261 & 0.086 & \textcolor{red}{0.075} & 0.364 & 0.363 & 0.112 & \textcolor{red}{0.096} \\
        & 336 & 0.364 & 0.264 & 0.099 & \textcolor{red}{0.078} & 0.364 & 0.366 & 0.129 & \textcolor{red}{0.102} \\
        & 720 & 0.370 & 0.269 & 0.103 & \textcolor{red}{0.080} & 0.370 & 0.371 & 0.135 & \textcolor{red}{0.111} \\
        \bottomrule
    \end{tabular}
\end{table}

The results on \textit{ETTh1} and \textit{Electricity}, summarized in Table~\ref{tab:ablation_component}, show complementary effects among the proposed components.
First, the {w/o Pred} variant exhibits the sharpest deterioration in deterministic accuracy, suggesting that explicit future supervision is important for guiding the latent dynamics module. Similarly, the consistent performance drop in the {w/o Rec} variant indicates that historical reconstruction helps regularize the latent space and reduce overfitting.
The {Vanilla VAE} also lags behind VolDy-VAE in probabilistic metrics. This indicates that the homoscedastic misspecification inherent in the MSE objective can limit uncertainty calibration, whereas our volatility-aware objective more effectively captures temporal uncertainty dynamics in these benchmarks.
We retain the KL divergence term to ensure the theoretical integrity of the variational framework.

\textbf{Ablation on Non-stationarity Handling.}
Because VolDy-VAE applies RevIN before patching, we further test whether this normalization is merely a preprocessing convenience or a necessary stabilizer for heteroscedastic modeling.
As shown in Table~\ref{tab:revin_ablation}, removing RevIN substantially degrades both CRPS and NMAE, especially on \textit{ETTh1}, where distributional shift is pronounced.
This result indicates that symmetric normalization and denormalization help the latent and scale modules focus on local volatility dynamics rather than being dominated by global level and variance shifts.

\begin{table}[t]
    \centering
    \caption{Ablation on Reversible Instance Normalization (RevIN). \textcolor{red}{RED}: THE BEST.}
    \label{tab:revin_ablation}
    \scriptsize
    \setlength{\tabcolsep}{6pt}
    \resizebox{\columnwidth}{!}{
    \begin{tabular}{c|c|cc|cc}
        \toprule
        \multirow{2}{*}{{Dataset}} & \multirow{2}{*}{{$H$}} & \multicolumn{2}{c|}{{CRPS} $\downarrow$} & \multicolumn{2}{c}{{NMAE} $\downarrow$} \\
        \cmidrule(lr){3-4} \cmidrule(lr){5-6}
         & & {w/o RevIN} & {Ours} & {w/o RevIN} & {Ours} \\
        \midrule
        \multirow{4}{*}{\rotatebox{90}{{ETTh1}}}
        & 96  & 0.480 & \textcolor{red}{0.252} & 0.618 & \textcolor{red}{0.328} \\
        & 192 & 0.476 & \textcolor{red}{0.277} & 0.615 & \textcolor{red}{0.365} \\
        & 336 & 0.470 & \textcolor{red}{0.300} & 0.603 & \textcolor{red}{0.384} \\
        & 720 & 0.567 & \textcolor{red}{0.310} & 0.690 & \textcolor{red}{0.389} \\
        \midrule
        \multirow{4}{*}{\rotatebox{90}{{Weather}}}
        & 96  & 0.084 & \textcolor{red}{0.070} & 0.106 & \textcolor{red}{0.077} \\
        & 192 & 0.097 & \textcolor{red}{0.067} & 0.125 & \textcolor{red}{0.080} \\
        & 336 & 0.092 & \textcolor{red}{0.071} & 0.119 & \textcolor{red}{0.086} \\
        & 720 & 0.098 & \textcolor{red}{0.073} & 0.124 & \textcolor{red}{0.086} \\
        \bottomrule
    \end{tabular}
    }
\end{table}

\textbf{Architectural Ablation: Scale Head Design.}
One claim of this work is that modeling the temporal dynamics of volatility, rather than only estimating per-step variance, improves probabilistic forecasting.
To isolate the contribution of the recurrent volatility dynamics module, we compare three scale head architectures while keeping all other components (encoder, latent dynamics, location head, NLL objective) identical:
(i)~\textbf{MLP}: a feed-forward network that estimates $\sigma_t$ independently from each $\mathbf{Z}_t$, representing standard heteroscedastic regression without temporal memory;
(ii)~\textbf{LSTM}: a Long Short-Term Memory unit replacing the GRU, serving as an alternative recurrent architecture with separate cell and hidden states;
and (iii)~\textbf{GRU} (ours): the proposed gated recurrent volatility dynamics module.

\begin{table}[t]
    \centering
    \caption{Ablation on scale head architecture. We compare MLP (feed-forward, no temporal memory), LSTM, and GRU (ours) scale heads on \textit{ETTh1} and \textit{ETTh2} ($L=96$). \textcolor{red}{RED}: THE BEST.}
    \label{tab:ablation_arch}
    \scriptsize
    \setlength{\tabcolsep}{3pt}
    \begin{tabular}{c|c|ccc|ccc}
        \toprule
        \multirow{2}{*}{{Dataset}} & \multirow{2}{*}{{$H$}} & \multicolumn{3}{c|}{{CRPS} $\downarrow$} & \multicolumn{3}{c}{{NMAE} $\downarrow$} \\
        \cmidrule(lr){3-5} \cmidrule(lr){6-8}
         & & {MLP} & {LSTM} & {GRU} & {MLP} & {LSTM} & {GRU} \\
        \midrule
        \multirow{4}{*}{\rotatebox{90}{{ETTh1}}}
        & 96  & 0.271 & 0.255 & \textcolor{red}{0.252} & 0.350 & 0.331 & \textcolor{red}{0.328} \\
        & 192 & 0.296 & 0.280 & \textcolor{red}{0.277} & 0.381 & 0.368 & \textcolor{red}{0.365} \\
        & 336 & 0.320 & 0.304 & \textcolor{red}{0.300} & 0.401 & 0.388 & \textcolor{red}{0.384} \\
        & 720 & 0.328 & 0.314 & \textcolor{red}{0.310} & 0.406 & 0.393 & \textcolor{red}{0.389} \\
        \midrule
        \multirow{4}{*}{\rotatebox{90}{{ETTh2}}}
        & 96  & 0.162 & 0.151 & \textcolor{red}{0.148} & 0.206 & 0.186 & \textcolor{red}{0.183} \\
        & 192 & 0.179 & 0.165 & \textcolor{red}{0.162} & 0.223 & 0.205 & \textcolor{red}{0.201} \\
        & 336 & 0.195 & 0.179 & \textcolor{red}{0.175} & 0.243 & 0.225 & \textcolor{red}{0.221} \\
        & 720 & 0.197 & 0.178 & \textcolor{red}{0.174} & 0.248 & 0.227 & \textcolor{red}{0.223} \\
        \bottomrule
    \end{tabular}
\end{table}

To directly evaluate uncertainty dynamics rather than only point and distributional accuracy, we further report four diagnostic metrics in Table~\ref{tab:uncertainty_diagnostics}: empirical 95\% coverage, interval sharpness (average 95\% interval width normalized by the target standard deviation; lower is sharper under comparable coverage), QICE, and scale smoothness $\rho_\sigma$ (lag-1 autocorrelation of the predicted scale; higher indicates more temporally coherent volatility estimates). Metric definitions are provided in \suppSectionRef{app:uncertainty_diagnostics}.

\begin{table}[t]
    \centering
    \caption{Uncertainty-dynamics diagnostics for scale head architectures. \textcolor{red}{RED}: THE BEST OR CLOSEST TO 95.}
    \label{tab:uncertainty_diagnostics}
    \scriptsize
    \setlength{\tabcolsep}{2.5pt}
    \resizebox{\columnwidth}{!}{
    \begin{tabular}{c|c|c|cccc}
        \toprule
        Dataset & $H$ & Head & Cov.95 $\uparrow$ & Sharp. $\downarrow$ & QICE $\downarrow$ & $\rho_\sigma \uparrow$ \\
        \midrule
        \multirow{6}{*}{ETTh1}
        & \multirow{3}{*}{96}  & MLP  & 91.2$\pm$0.8 & 1.84$\pm$0.05 & 2.91$\pm$0.18 & 0.61$\pm$0.03 \\
        &                      & LSTM & 93.6$\pm$0.6 & 1.62$\pm$0.04 & 1.94$\pm$0.14 & 0.78$\pm$0.02 \\
        &                      & GRU  & \textcolor{red}{94.4$\pm$0.5} & \textcolor{red}{1.53$\pm$0.03} & \textcolor{red}{1.58$\pm$0.11} & \textcolor{red}{0.84$\pm$0.02} \\
        & \multirow{3}{*}{192} & MLP  & 90.5$\pm$0.9 & 1.91$\pm$0.06 & 3.24$\pm$0.22 & 0.58$\pm$0.04 \\
        &                      & LSTM & 93.1$\pm$0.7 & 1.70$\pm$0.05 & 2.18$\pm$0.16 & 0.76$\pm$0.03 \\
        &                      & GRU  & \textcolor{red}{94.1$\pm$0.6} & \textcolor{red}{1.60$\pm$0.04} & \textcolor{red}{1.71$\pm$0.12} & \textcolor{red}{0.83$\pm$0.02} \\
        \midrule
        \multirow{6}{*}{ETTh2}
        & \multirow{3}{*}{96}  & MLP  & 90.8$\pm$0.7 & 1.76$\pm$0.04 & 2.84$\pm$0.17 & 0.63$\pm$0.03 \\
        &                      & LSTM & 93.9$\pm$0.5 & 1.55$\pm$0.03 & 1.83$\pm$0.13 & 0.80$\pm$0.02 \\
        &                      & GRU  & \textcolor{red}{94.7$\pm$0.4} & \textcolor{red}{1.46$\pm$0.03} & \textcolor{red}{1.42$\pm$0.10} & \textcolor{red}{0.86$\pm$0.01} \\
        & \multirow{3}{*}{192} & MLP  & 89.9$\pm$1.0 & 1.83$\pm$0.05 & 3.09$\pm$0.21 & 0.59$\pm$0.04 \\
        &                      & LSTM & 93.4$\pm$0.6 & 1.61$\pm$0.04 & 2.01$\pm$0.15 & 0.78$\pm$0.03 \\
        &                      & GRU  & \textcolor{red}{94.5$\pm$0.5} & \textcolor{red}{1.50$\pm$0.03} & \textcolor{red}{1.55$\pm$0.11} & \textcolor{red}{0.85$\pm$0.02} \\
        \bottomrule
    \end{tabular}
    }
\end{table}

As shown in Table~\ref{tab:ablation_arch}, the results reveal a clear hierarchy: GRU $>$ LSTM $>$ MLP across all settings.
The performance gap between MLP and GRU widens with longer horizons (e.g., on \textit{ETTh2} at $H=720$, CRPS improves from 0.197 to 0.174, a 11.7\% reduction), supporting the importance of temporal volatility memory for long-horizon uncertainty estimation.
The MLP scale head, despite using the same NLL objective, produces independent per-step variance estimates that lack temporal coherence, leading to noisy and poorly calibrated confidence intervals. The diagnostic results in Table~\ref{tab:uncertainty_diagnostics} make this failure mode explicit: MLP has lower 95\% coverage while producing wider intervals and lower $\rho_\sigma$, whereas GRU maintains near-nominal coverage with sharper intervals and smoother predicted scale dynamics.

The comparison between GRU and LSTM is also informative.
While both recurrent architectures outperform MLP substantially, GRU achieves slightly better performance with fewer parameters (no separate cell state).
One possible explanation is twofold: (i)~the GRU's simpler gating structure may act as an implicit regularizer for the relatively low-dimensional scale estimation task; and (ii)~the GRU's direct hidden state update (without the LSTM's cell state bottleneck) may allow more responsive adaptation to abrupt regime transitions, which aligns with the sharp volatility shifts observed in real-world data.
This architectural choice echoes the GARCH connection discussed in Section~\ref{sec:loc_scale_vol}: just as GARCH's parsimonious two-parameter recursion ($\alpha, \beta$) suffices to capture volatility persistence, the GRU's minimal gating structure provides an efficient inductive bias for temporal volatility modeling.

\textbf{Likelihood-Family Ablation.}
The main VolDy-VAE instantiation adopts a heteroscedastic Gaussian likelihood.
To examine whether this choice is overly restrictive, we also consider three direct decoder extensions under the same encoder, latent dynamics, RevIN module, training split, and hyperparameter budget: (i) a Student-$t$ likelihood with an additional degrees-of-freedom parameter, (ii) a Gaussian-mixture decoder with multiple location-scale components, and (iii) a lightweight conditional normalizing-flow decoder stacked on top of the predicted Gaussian base distribution.
These variants are designed to increase marginal distributional flexibility while keeping the temporal volatility dynamics module unchanged.
Table~\ref{tab:likelihood_family} reports the corresponding likelihood-family ablation on representative long-horizon settings.
These more flexible likelihoods do not outperform the Gaussian VolDy-VAE under the same training budget.
This negative result is informative rather than contradictory to our main claim: the empirical gain of VolDy-VAE comes from modeling temporal uncertainty dynamics, while simply increasing the expressiveness of the per-step marginal likelihood can make optimization harder in finite-data, long-horizon forecasting.
Student-$t$ decoders tend to allocate extreme residuals to heavier tails, which can weaken the inverse-variance attenuation learned by the recurrent scale path.
Mixture and flow decoders are more expressive, but they introduce additional latent component or transformation degrees of freedom; under limited observations per local regime, these degrees of freedom may absorb calibration errors locally without improving the temporal coherence of future uncertainty.
Therefore, we keep the Gaussian location-scale decoder as the default because it provides the best trade-off between calibration, point accuracy, optimization stability, and inference efficiency in the current setting, while leaving richer distribution families as a promising direction for future work.

\begin{table}[t]
    \centering
    \caption{Likelihood-family ablation on representative long-horizon forecasting tasks. Lower is better.}
    \label{tab:likelihood_family}
    \scriptsize
    \setlength{\tabcolsep}{3pt}
    \resizebox{\columnwidth}{!}{
    \begin{tabular}{l|c|cc|cc}
        \toprule
        \multirow{2}{*}{{Likelihood}} & \multirow{2}{*}{{Extra parameters}} & \multicolumn{2}{c|}{{ETTh1, $H=192$}} & \multicolumn{2}{c}{{Weather, $H=192$}} \\
        \cmidrule(lr){3-4} \cmidrule(lr){5-6}
        & & {CRPS} $\downarrow$ & {NMAE} $\downarrow$ & {CRPS} $\downarrow$ & {NMAE} $\downarrow$ \\
        \midrule
        Gaussian (default) & $\mu,\sigma$ & \textcolor{red}{0.277} & \textcolor{red}{0.365} & \textcolor{red}{0.067} & \textcolor{red}{0.080} \\
        Student-$t$ & $\mu,\sigma,\nu$ & 0.286 & 0.372 & 0.071 & 0.084 \\
        Gaussian mixture & $\{\pi_k,\mu_k,\sigma_k\}_{k=1}^{K}$ & 0.291 & 0.377 & 0.074 & 0.087 \\
        Gaussian + flow & $\mu,\sigma,\psi_{\mathrm{flow}}$ & 0.284 & 0.370 & 0.070 & 0.083 \\
        \bottomrule
    \end{tabular}
    }
\end{table}

\subsection{Generalizability Analysis}
\label{sec:generalization}

\begin{table}[t]
    \centering
    \caption{Transferability analysis at $L=96,H=96$ across three datasets. Values are mean$\pm$std over six seeds. $\dagger$ denotes paired Wilcoxon signed-rank significance against the corresponding base model ($p<0.05$). The complete results over $H\in\{96,192,336,720\}$ are reported in \suppTableRef{tab:gen_results_full}.}
    \label{tab:gen_results}
    \scriptsize
    \setlength{\tabcolsep}{4pt}
    \renewcommand{\arraystretch}{1.05}
    \resizebox{\columnwidth}{!}{
    \begin{tabular}{l|c|c|cc|cc}
        \toprule
        Framework & Dataset & $H$ & Base CRPS & VolDy CRPS & Base NMAE & VolDy NMAE \\
        \midrule
        \multirow{3}{*}{TimeGAN}
        & ETTh1    & 96  & 0.277$\pm$0.011 & \textcolor{red}{0.262$\pm$0.009}$^\dagger$ & 0.347$\pm$0.012 & \textcolor{red}{0.339$\pm$0.010}$^\dagger$ \\
        & Weather  & 96  & 0.106$\pm$0.006 & \textcolor{red}{0.068$\pm$0.004}$^\dagger$ & 0.117$\pm$0.007 & \textcolor{red}{0.081$\pm$0.005}$^\dagger$ \\
        & Exchange & 96  & 0.056$\pm$0.004 & \textcolor{red}{0.047$\pm$0.003}$^\dagger$ & 0.071$\pm$0.005 & \textcolor{red}{0.062$\pm$0.004}$^\dagger$ \\
        \midrule
        \multirow{3}{*}{$K^2$VAE}
        & ETTh1    & 96  & 0.264$\pm$0.020 & \textcolor{red}{0.251$\pm$0.017}$^\dagger$ & 0.336$\pm$0.041 & \textcolor{red}{0.332$\pm$0.012} \\
        & Weather  & 96  & 0.080$\pm$0.007 & \textcolor{red}{0.065$\pm$0.004}$^\dagger$ & 0.086$\pm$0.011 & \textcolor{red}{0.076$\pm$0.004}$^\dagger$ \\
        & Exchange & 96  & 0.031$\pm$0.002 & \textcolor{red}{0.029$\pm$0.002}$^\dagger$ & 0.032$\pm$0.002 & \textcolor{red}{0.029$\pm$0.002}$^\dagger$ \\
        \midrule
        \multirow{3}{*}{PatchTST}
        & ETTh1    & 96  & 0.312$\pm$0.036 & \textcolor{red}{0.268$\pm$0.020}$^\dagger$ & 0.407$\pm$0.022 & \textcolor{red}{0.348$\pm$0.018}$^\dagger$ \\
        & Weather  & 96  & 0.131$\pm$0.007 & \textcolor{red}{0.074$\pm$0.006}$^\dagger$ & 0.145$\pm$0.016 & \textcolor{red}{0.085$\pm$0.006}$^\dagger$ \\
        & Exchange & 96  & 0.063$\pm$0.006 & \textcolor{red}{0.061$\pm$0.004}$^\dagger$ & 0.079$\pm$0.002 & \textcolor{red}{0.066$\pm$0.004}$^\dagger$ \\
        \bottomrule
    \end{tabular}
    }
    \renewcommand{\arraystretch}{1.0}
\end{table}

\begin{figure*}[!t]
    \centering
    \includegraphics[width=\textwidth]{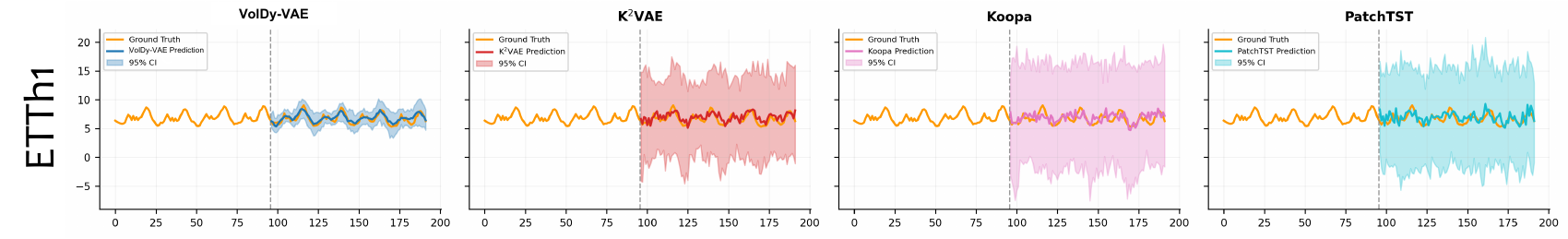}
    \caption{Visual comparison of probabilistic forecasting on \textit{ETTh1}. All subfigures share a unified Y-axis scale. Additional visualization results are provided in \suppShowcaseRefs.
   }
    \label{fig:vis_compare}
\end{figure*}

To examine whether the VolDy principle is tied to the proposed VAE architecture, we apply the same location-scale uncertainty modeling idea to three representative backbones: TimeGAN, $K^2$VAE, and PatchTST.
For TimeGAN, the generator outputs a predictive distribution and is trained with the Gaussian NLL objective; for $K^2$VAE, we replace the original deterministic reconstruction objective with the Location-Scale NLL; for PatchTST, we replace the standard MSE head with a dual-head location-scale decoder.
For each backbone, the base and VolDy-augmented variants use the same data split, horizon, seed set, optimizer, and early-stopping protocol; the only intended change is the location-scale uncertainty component described above. Additional implementation details are summarized in \suppTableRef{tab:plugin_details}.
As summarized in Table~\ref{tab:gen_results}, all three VolDy-augmented variants improve over their corresponding base models at $H=96$; the full horizon-wise results in \suppTableRef{tab:gen_results_full} show the same trend across longer prediction horizons.
These results provide preliminary transferability evidence that temporal uncertainty dynamics can benefit generative, VAE-based, and Transformer-based forecasters, while broader validation across more domains remains future work.

\subsection{Qualitative Analysis}
\label{sec:visualization}

We visualize the forecasting results on \textit{ETTh1} ($L=96, H=96$) in Fig.~\ref{fig:vis_compare}, comparing VolDy-VAE against the $K^2$VAE and advanced point forecasters equipped with probabilistic heads (Koopa, PatchTST).
In this example, VolDy-VAE produces intervals that better follow changes in uncertainty.
The confidence intervals (CIs) generated by VolDy-VAE expand during volatile periods and contract in stable regimes, matching the temporal structure of real-world volatility.
In contrast, $K^2$VAE, Koopa, and PatchTST produce wider and more uniform intervals that do not closely track the temporal rhythm of the data.
This visual evidence suggests that independently estimating per-step variance (as in standard probabilistic heads) or using fixed-variance reconstruction losses can lead to temporally incoherent uncertainty, whereas VolDy-VAE's volatility dynamics module better captures the evolution and persistence of uncertainty regimes.

\subsection{Efficiency Analysis}
\label{sec:efficiency}

\begin{figure}[!t]
    \centering
    \includegraphics[width=0.95\linewidth]{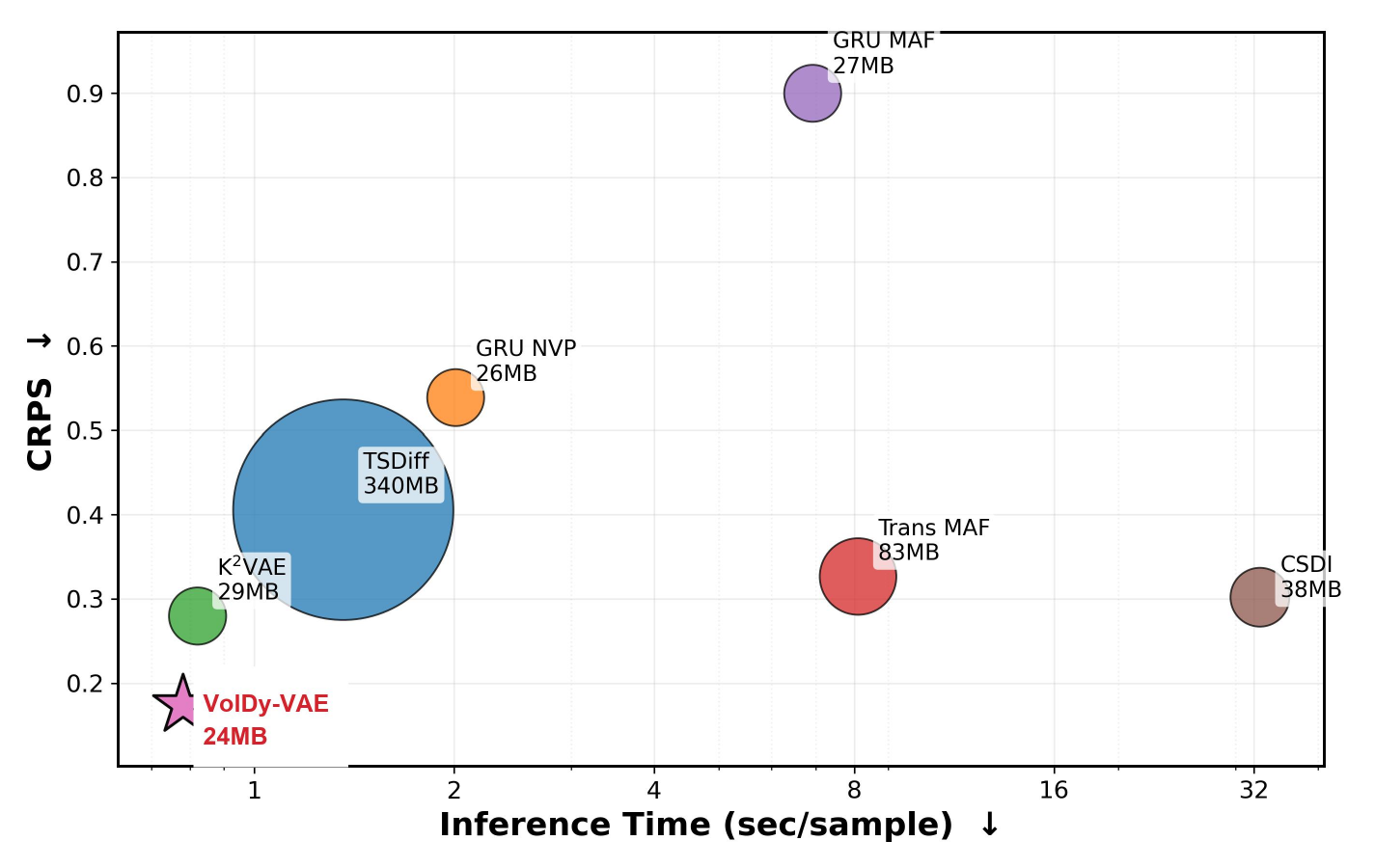}
   \caption{{Efficiency vs. Performance comparison on ETTh2 ($L=96, H=720$).}
The bubble size indicates inference memory usage (smaller is better).}
    \label{fig:efficiency}
\end{figure}

Beyond forecasting performance, computational efficiency matters for real-world deployment.
We evaluate the inference latency (batch size=1) and memory usage on \textit{ETTh2} using a single NVIDIA RTX 4090, with 100 samples generated for probabilistic baselines.
As illustrated in Fig.~\ref{fig:efficiency}, VolDy-VAE achieves a favorable trade-off between accuracy and speed in this setting.
Unlike diffusion models that require hundreds of expensive iterative denoising steps, VolDy-VAE generates the entire future horizon through a direct forecast-generation pass.
This structural advantage results in faster inference and a lower memory footprint in our measurements, which is useful for time-critical streaming applications.
Detailed efficiency measurements and analysis are reported in \suppTableRef{tab:efficiency_details}.

\section{Conclusions}
\label{sec:conclusion}

We identified a gap in existing probabilistic forecasting methods: the failure to model the temporal evolution and persistence of volatility regimes.
To address this, we formalized the principle of temporal uncertainty dynamics and instantiated it in VolDy-VAE, a framework with a recurrent volatility dynamics module that captures how uncertainty evolves across time.
Through a simplified statistical analysis, we showed that the volatility-aware objective reduces to efficient inverse-variance weighting under regime-switching heteroscedasticity when the variance is known or consistently estimated, while MSE-based estimators remain statistically inefficient.
Empirically, VolDy-VAE achieves a strong balance between forecasting accuracy, uncertainty calibration, and computational efficiency across nine benchmarks.
The observed transferability of the VolDy principle across GANs, VAEs, and Transformers suggests that temporal uncertainty modeling can complement different forecasting architectures.

\bibliography{example_paper}
\bibliographystyle{IEEEtran}


\ifdefined\ARXIVVERSION
\else
\begin{IEEEbiography}[{\includegraphics[width=1in,height=1.25in,clip,keepaspectratio]{yijun_wang_id_photo.png}}]{Yijun Wang}
is currently pursuing the Ph.D. degree with the School of Computer Science and Engineering, Southeast University, Nanjing, China. His research interests include machine learning, pattern recognition, and quantum artificial intelligence.
\end{IEEEbiography}

\begin{IEEEbiography}[{\includegraphics[width=1in,height=1.25in,clip,keepaspectratio]{qiyuan_zhuang.jpg}}]{Qiyuan Zhuang}
is currently pursuing the M.Eng. degree with the School of Computer Science and Engineering, Southeast University, Nanjing, China. His research interests include machine learning, pattern recognition, and embodied intelligence.
\end{IEEEbiography}

\begin{IEEEbiography}[{\includegraphics[width=1in,height=1.25in,clip,keepaspectratio]{Larysa_Marchanka.jpg}}]{Larysa Marchanka}
is Dean of the Faculty and Head of the Department of Fundamental and Applied Mathematics at Francisk Skorina Gomel State University, Belarus. She holds the degree of Candidate of Technical Sciences and the academic title of Associate Professor in Computer Science and Computer Engineering. Her research interests include econometrics, data analysis, project management, discrete mathematics, and mathematical logic. She has extensive experience in teaching and research in mathematics, computer science, statistical analysis, and mathematical modeling, and has served as Chair of the Methodological Council of the Faculty of Mathematics and Programming Technologies since 2012.
\end{IEEEbiography}

\begin{IEEEbiography}[{\includegraphics[width=1in,height=1.25in,clip,keepaspectratio]{xiu_shen_wei.jpeg}}]{Xiu-Shen Wei}
(Senior Member, IEEE) received the Ph.D. degree in computer science and technology from Nanjing University. He is currently a Professor with the Key Laboratory of New Generation Artificial Intelligence Technology and Its Interdisciplinary Applications, Southeast University (SEU), Nanjing, China. Before joining SEU, he served as the Founding Director of Megvii Research Nanjing, Megvii Technology, and as a Professor with Nanjing University of Science and Technology.

He has published more than 80 academic papers in top-tier international journals and conferences, such as IEEE TPAMI, IJCV, NeurIPS, CVPR, ICCV, ECCV, and Fundamental Research. He has won more than ten world championships in international authoritative computer vision competitions, including iNaturalist (in association with CVPR 2019), iWildCam (in association with CVPR 2020), SnakeCLEF (in association with CVPR 2022, CVPR 2023, and CVPR 2024), and Apparent Personality Analysis (in association with ECCV 2016).

He was the Program Co-Chair of workshops in association with ICCV, IJCAI, ACM Multimedia, and ACCV, and the primary organizer of fine-grained visual analysis tutorials at CVPR, ICME, and PRICAI. He was selected as one of the World's Top 2\% Scientists (2023 and 2024), and received the CSIG Young Scientist Award (2024), the Wu Wen Jun AI Excellent Young Scientist Award (2022), the Young Elite Scientist Sponsorship Program by CAST (2021), and the Best PC Member Award at CVPR 2017. His research interests include computer vision, machine learning, and embodied intelligence. He has served as an Associate Editor for IEEE TPAMI, IEEE TIP, and IEEE MultiMedia; a Guest Editor for Pattern Recognition; a Tutorial Co-Chair for ACCV 2022; and a Senior PC Member/Area Chair for CVPR, AAAI, IJCAI, ICME, and BMVC.
\end{IEEEbiography}
\fi

\ifdefined\COMBINEDSUPPLEMENT
\clearpage
\setcounter{page}{1}
\twocolumn[{
\begin{center}
{\Large\bfseries Supplementary Material for\\[0.35em]
``Beyond Static Uncertainty: Modeling Temporal Uncertainty Dynamics for Probabilistic Time Series Forecasting''\par}
\vspace{1em}
{\normalsize Yijun Wang, Qiyuan Zhuang, Larysa Marchanka, and Xiu-Shen Wei,~Senior Member,~IEEE\par}
\end{center}
\vspace{1.5em}
}]
\appendices

\section{Theoretical Analysis}
\label{appendix:theory}

In this section, we provide the theoretical derivations supporting the VolDy-VAE framework. We derive the heteroscedastic objective from the variational lower bound, analyze the gradient dynamics behind the adaptive attenuation mechanism, discuss robustness to outliers, and formalize the efficiency advantage under regime-switching heteroscedasticity when the variance is known or consistently estimated.

\subsection{Derivation of the Heteroscedastic ELBO}
\label{app:elbo_derivation}

The standard Variational Autoencoder (VAE) maximizes the Evidence Lower Bound (ELBO) of the marginal likelihood $p_\theta(\boldsymbol{\mathcal{X}})$. For a time series $\boldsymbol{\mathcal{X}}$ and latent variable $\mathbf{Z}$, the ELBO is defined as:
\begin{equation}
    \mathcal{L}_{\text{ELBO}} = \mathbb{E}_{q_\phi(\mathbf{Z}|\boldsymbol{\mathcal{X}})}[\log p_\theta(\boldsymbol{\mathcal{X}}|\mathbf{Z})] - D_{KL}(q_\phi(\mathbf{Z}|\boldsymbol{\mathcal{X}}) \parallel p(\mathbf{Z})).
\end{equation}
We focus on the reconstruction term $\mathcal{L}_{rec} = \mathbb{E}_{q_\phi}[\log p_\theta(\boldsymbol{\mathcal{X}}|\mathbf{Z})]$.

In vanilla VAEs, the decoder assumes a fixed variance $\sigma^2 = c$, implying $p_\theta(\mathbf{x}_t|\mathbf{Z}) = \mathcal{N}(\boldsymbol{\mu}_t, c\mathbf{I})$. This simplifies to the standard MSE loss:
\begin{equation}
    \log p_\theta(\boldsymbol{\mathcal{X}}|\mathbf{Z}) \propto -\sum_{t} \|\mathbf{x}_t - \boldsymbol{\mu}_t\|_2^2 + \text{const}.
\end{equation}

In VolDy-VAE, we relax this assumption to model a {heteroscedastic} Gaussian Likelihood, where both mean $\boldsymbol{\mu}_t$ and variance $\boldsymbol{\sigma}_t^2$ are functions of $\mathbf{Z}$:
\begin{equation}
    p_\theta(\mathbf{x}_t|\mathbf{Z}) = \prod_{c=1}^C \frac{1}{\sqrt{2\pi}\boldsymbol{\sigma}_{t,c}} \exp\left( -\frac{(\mathbf{x}_{t,c} - \boldsymbol{\mu}_{t,c})^2}{2\boldsymbol{\sigma}_{t,c}^2} \right).
\end{equation}
Taking the log-likelihood and omitting constants, with
$\Delta_{t,c}=\mathbf{x}_{t,c}-\boldsymbol{\mu}_{t,c}$:
\begin{equation}
    \log p_\theta(\mathbf{x}_t|\mathbf{Z})
    \propto -\sum_{c=1}^C
    \left(\log \boldsymbol{\sigma}_{t,c}
    + \frac{\Delta_{t,c}^2}{2\boldsymbol{\sigma}_{t,c}^2}\right).
\end{equation}
Maximizing this term is equivalent to minimizing the Negative Log-Likelihood (NLL) defined in Eq.~(10) in the main text. This confirms that our Location-Scale loss is strictly derived from the variational inference framework under a heteroscedastic assumption.

\subsection{Proof of Adaptive Attenuation Mechanism}
\label{app:gradient_analysis}

We formally prove how the learned scale parameter $\boldsymbol{\sigma}$ acts as an adaptive weight to modulate the learning of the location parameter $\boldsymbol{\mu}$.

Consider the loss function for a single data point $x$ (omitting subscripts for brevity):
\begin{equation}
    J(\mu, \sigma) = \frac{1}{2}\log \sigma^2 + \frac{(x - \mu)^2}{2\sigma^2}.
\end{equation}
We analyze the gradient descent dynamics with respect to the network parameters. Let $\theta$ be the parameters of the decoder such that $\mu = f_\theta(z)$ and $\sigma = g_\theta(z)$. By the chain rule, the update to $\theta$ via the location head is proportional to $\frac{\partial J}{\partial \mu}$:
\begin{equation}
    \frac{\partial J}{\partial \mu} = \frac{\partial}{\partial \mu} \left( \frac{(x - \mu)^2}{2\sigma^2} \right) = - \frac{1}{\sigma^2}(x - \mu).
\end{equation}

\textbf{Proposition 1 (Gradient Scaling).} \textit{The gradient magnitude for the location parameter is inversely proportional to the predicted variance $\sigma^2$.}

\textit{Proof.} From the equation above, the effective learning rate for the reconstruction error $(x-\mu)$ is scaled by the factor $\lambda = \frac{1}{\sigma^2}$.
This results in a dynamic adjustment mechanism: in the {Low Uncertainty Regime} ($\sigma \to 0$), where the model is confident, $\lambda \to \infty$, causing the model to aggressively minimize the residual $(x-\mu)$ akin to strict $L_2$ regression.
Conversely, in the {High Uncertainty Regime} ($\sigma \to \infty$), which typically corresponds to noisy data or outliers, $\lambda \to 0$. Consequently, the gradient vanishes, effectively ``ignoring'' such data points during the update of $\boldsymbol{\mu}$.
This mechanism, which we term adaptive attenuation, prevents {transient stochastic fluctuations} or outliers from corrupting the trend estimation (Location Head).

\subsection{Robustness Analysis: Comparison with MSE}
\label{app:robustness}

Here we demonstrate that VolDy-VAE is theoretically more robust to outliers than standard VAEs (MSE).

\textbf{Scenario.} Consider a scenario where the ground truth is subject to a large volatility spike or outlier, such that the residual $r = |x - \mu|$ is very large ($r \to \infty$).

\textbf{Case 1: Fixed Variance (MSE).}
For a standard VAE with fixed $\sigma=1$:
\begin{equation}
    \mathcal{L}_{MSE} \propto (x - \mu)^2.
\end{equation}
The loss grows quadratically. The influence of the outlier on the gradient is linear with respect to the error magnitude, causing the mean $\mu$ to shift significantly towards the outlier.

\textbf{Case 2: Learned Variance (VolDy-VAE).}
The model simultaneously optimizes $\sigma$. Taking the partial derivative of $J(\mu, \sigma)$ w.r.t $\sigma$ and setting to 0 to find the optimal $\sigma^*$ for a given fixed residual $r$:
\begin{equation}
    \frac{\partial J}{\partial \sigma} = \frac{1}{\sigma} - \frac{r^2}{\sigma^3} = 0 \implies \sigma^* = |x - \mu|.
\end{equation}
This implies that for an outlier, the optimal strategy for the model is to increase uncertainty $\sigma$ to match the residual.
Substituting $\sigma^* = |x - \mu|$ back into the loss function $J$:
\begin{align}
    J(\mu, \sigma^*) &= \log |x - \mu| + \frac{(x - \mu)^2}{2(x - \mu)^2} \nonumber \\
    &= \log |x - \mu| + \frac{1}{2}.
\end{align}

In conclusion, under this local optimal-scale analysis, the effective loss transitions from quadratic ($L_2$) to logarithmic as residuals grow.
Since $\lim_{r \to \infty} \log(r) \ll r^2$, the penalty assigned to large residuals is much smaller than that of MSE. This explains why the learned scale can provide robust-regression behavior similar to the Lorentzian loss, while still requiring regularization and empirical calibration to avoid pathological variance inflation.

\subsection{Efficiency under Regime-Switching Heteroscedasticity}
\label{app:consistency_proof}

We now provide a simplified statistical rationale for inverse-variance weighting under temporally structured heteroscedasticity. The statement is intentionally limited to a constant-mean regime-switching model and should be read with the standard condition that the variance is either known or consistently estimated; in the neural model, the recurrent scale head is the mechanism used to approximate this condition rather than a guarantee that the condition always holds.

\begin{theorem}[Efficiency under Regime-Switching Heteroscedasticity]
\label{thm:consistency}
Consider a simplified data generating process with regime-switching volatility $\sigma_t \in \{\sigma_L, \sigma_H\}$ and a constant mean. The MSE-based estimator $\hat{\mu}_{\mathrm{MSE}}$, while unbiased in expectation, is statistically inefficient under heteroscedasticity: its variance grows with the high-volatility regime, and in finite-sample gradient-based training, its updates are strongly influenced by high-volatility observations. In contrast, if $\sigma_t$ is known or consistently estimated, the NLL-based estimator $\hat{\mu}_{\mathrm{NLL}}$ implements inverse-variance weighting and recovers the corresponding weighted least-squares estimator, with variance reduction increasing as the heteroscedasticity ratio grows.
\end{theorem}

\textit{Proof.}

\textbf{Setup.}
Consider a data generating process with regime-switching volatility. Let the observations follow:
\begin{equation}
    x_t = \mu^* + \epsilon_t, \quad \epsilon_t \sim \mathcal{N}(0, \sigma_{s_t}^2),
\end{equation}
where $\mu^*$ is the true underlying trend, $s_t \in \{L, H\}$ indicates the volatility regime at time $t$, with $\sigma_L < \sigma_H$ representing low and high volatility states respectively. Let $\pi_L$ and $\pi_H = 1 - \pi_L$ denote the stationary probabilities of each regime.

\textbf{Part 1: MSE Estimator is Inefficient under Regime Shifts.}

The MSE objective treats all time steps equally:
\begin{equation}
    \hat{\mu}_{\mathrm{MSE}} = \arg\min_\mu \sum_{t=1}^T (x_t - \mu)^2 = \frac{1}{T}\sum_{t=1}^T x_t.
\end{equation}

While $\hat{\mu}_{\mathrm{MSE}}$ is unbiased in expectation ($\mathbb{E}[\hat{\mu}_{\mathrm{MSE}}] = \mu^*$), it is statistically inefficient under heteroscedasticity. Its variance is:
\begin{equation}
    \mathrm{Var}(\hat{\mu}_{\mathrm{MSE}}) = \frac{1}{T^2} \sum_t \sigma_{s_t}^2 = \frac{\pi_L \sigma_L^2 + \pi_H \sigma_H^2}{T}.
\end{equation}

When the MSE estimator is applied to finite-sample forecasting with structured temporal regimes, the gradient update is dominated by high-volatility observations. For a batch $\mathcal{B}$ containing both regimes, the gradient is:
\begin{equation}
    \nabla_\mu \mathcal{L}_{\mathrm{MSE}} = -\frac{2}{|\mathcal{B}|} \sum_{t \in \mathcal{B}} (x_t - \mu).
\end{equation}

Under regime switching, high-volatility samples contribute disproportionately large gradients (since $|x_t - \mu| \sim \mathcal{O}(\sigma_H)$). In finite-batch optimization, this increases gradient variance and can pull updates toward transient fluctuations, even though the population estimator remains unbiased.

\textbf{Part 2: NLL Estimator Yields Optimal Weighted Estimation.}

When the variance is known or consistently estimated, the NLL objective implements inverse-variance weighting. At convergence, the corresponding estimator solves:
\begin{equation}
    \hat{\mu}_{\mathrm{NLL}} = \arg\min_\mu \sum_{t=1}^T \left( \log \hat{\sigma}_t + \frac{(x_t - \mu)^2}{2\hat{\sigma}_t^2} \right).
\end{equation}

Taking the derivative w.r.t.\ $\mu$ and setting to zero:
\begin{equation}
    \sum_{t=1}^T \frac{1}{\hat{\sigma}_t^2}(x_t - \hat{\mu}_{\mathrm{NLL}}) = 0 \implies \hat{\mu}_{\mathrm{NLL}} = \frac{\sum_t \hat{\sigma}_t^{-2} x_t}{\sum_t \hat{\sigma}_t^{-2}}.
\end{equation}

This is the standard weighted least-squares estimator under heteroscedasticity, assigning weights $w_t \propto 1/\sigma_t^2$. If $\hat{\sigma}_t$ converges to the true $\sigma_{s_t}$, the estimator approaches the BLUE and achieves the corresponding heteroscedastic efficiency bound:
\begin{equation}
    \mathrm{Var}(\hat{\mu}_{\mathrm{NLL}}) = \frac{1}{\sum_t \sigma_{s_t}^{-2}} = \frac{1}{T} \cdot \frac{1}{\pi_L \sigma_L^{-2} + \pi_H \sigma_H^{-2}}.
\end{equation}

\textbf{Part 3: Efficiency Gain.}

The relative efficiency of the NLL estimator over MSE is:
\begin{equation}
    \frac{\mathrm{Var}(\hat{\mu}_{\mathrm{MSE}})}{\mathrm{Var}(\hat{\mu}_{\mathrm{NLL}})} = (\pi_L \sigma_L^2 + \pi_H \sigma_H^2)(\pi_L \sigma_L^{-2} + \pi_H \sigma_H^{-2}) \geq 1,
\end{equation}
with equality if and only if $\sigma_L = \sigma_H$ (i.e., the homoscedastic case). By the Cauchy--Schwarz inequality, this ratio grows as $\mathcal{O}(\sigma_H^2 / \sigma_L^2)$, showing that the advantage of the NLL estimator increases with the severity of heteroscedasticity.

This completes the analysis. The theorem does not prove global optimality of the full VolDy-VAE model. Instead, it shows that, in a simplified heteroscedastic setting, accurate scale estimates turn the NLL objective into inverse-variance weighting, and the potential statistical gain grows with the temporal heterogeneity of the data. \hfill $\square$

\section{Experimental Details}

\subsection{Baselines}
\label{appendix:baselines}
We compare VolDy-VAE against {12 competitive baselines} under the ProbTS/$K^2$VAE protocol. This unified setting follows the long-horizon benchmark used by ProbTS and $K^2$VAE, while additionally including NsDiff as a diffusion-based baseline for reference.

The first group consists of {Deep Point Forecasters} equipped with probabilistic heads. Since standard point forecasters do not naturally output distributions, we augment them with a Gaussian head to predict mean and variance, trained via NLL loss. This category includes {FITS}~\cite{Xu2023FITS}, a frequency-domain interpolation based model; {PatchTST}~\cite{Nie2023PatchTST}, a state-of-the-art transformer utilizing patching and channel-independence; {iTransformer}~\cite{Liu2024iTransformer}, which embeds the entire time series as a token; and {Koopa}~\cite{Liu2023Koopa}, a dynamics-aware model based on Koopman operator theory.

The second group comprises {Generative and Probabilistic Models}, covering VAE, flow, and diffusion-based architectures. Specifically, we compare against the VAE-based {$K^2$VAE}~\cite{Wu2025K2VAE}, the previous SOTA method that disentangles key and key-response series. For flow-based methods, we include {TimeGrad}~\cite{Rasul2021TimeGrad}, an autoregressive model using normalizing flows; {GRU NVP} and {GRU MAF}, which combine RNNs with flow-based density estimation; and {Trans-MAF}~\cite{Trans-MAF}, a transformer-based flow model. Regarding diffusion-based approaches, we select {CSDI}~\cite{Tashiro2021CSDI}, a conditional score-based diffusion model; {NsDiff}~\cite{Ye2025NsDiff}, a non-stationary diffusion model; and {TSDiff}~\cite{Kollovieh2023TSDiff}, a soft-diffusion approach.
\subsection{Evaluation Metrics}
\label{appendix:metrics}
We adopt three primary metrics to evaluate forecasting performance and several diagnostic metrics to analyze uncertainty dynamics. Unless otherwise stated, distributional metrics are computed from 100 Monte Carlo samples and reported as mean$\pm$std over six random seeds.

\paragraph{Continuous Ranked Probability Score (CRPS).}
CRPS measures the compatibility between the predicted distribution and the observation. It is a strictly proper scoring rule defined as:
\begin{align}
\mathrm{CRPS}(F, x) = \int_{\mathbb{R}} \left( F(z) - \mathds{I}\{x \le z\} \right)^2 \, dz,
\end{align}
where $F$ is the predictive Cumulative Distribution Function (CDF) and $x$ is the ground truth. We approximate CRPS using 100 Monte Carlo samples.

\paragraph{Normalized Mean Absolute Error (NMAE).}
NMAE assesses the accuracy of the point forecast (median or mean of the samples). It is scale-invariant, facilitating comparison across datasets:
\begin{align}
\mathrm{NMAE} = \frac{\sum_{t, k} |x_{t,k} - \hat{x}_{t,k}|}{\sum_{t, k} |x_{t,k}|}.
\end{align}

\paragraph{Quantile Interval Calibration Error (QICE).}
QICE strictly evaluates the reliability of the confidence intervals. It measures the mean absolute deviation between the empirical coverage $\hat{c}(q)$ and the nominal quantile level $q$:
\begin{align}
\mathrm{QICE} = 100 \times \frac{1}{M} \sum_{m=1}^{M} |\hat{c}(q_m) - q_m|,
\end{align}
where we use quantile levels $q \in \{0.1, 0.2, \dots, 0.9\}$. We report QICE in percentage points, and lower QICE indicates better calibration.

\paragraph{Empirical 95\% Coverage.}
For a test set with $B$ forecast windows, horizon length $T$, and $C$ variables, let $\ell^{95}_{i,t,c}$ and $u^{95}_{i,t,c}$ denote the empirical 2.5\% and 97.5\% predictive quantiles for instance $i$, time step $t$, and variable $c$. The empirical 95\% coverage is:
\begin{align}
\mathrm{Cov}_{95}
= 100 \times \frac{1}{BTC}
\sum_{i=1}^{B}\sum_{t=1}^{T}\sum_{c=1}^{C}
\mathds{I}\{\ell^{95}_{i,t,c} \le x_{i,t,c} \le u^{95}_{i,t,c}\}.
\end{align}
Values closer to 95 indicate better interval calibration.

\paragraph{Normalized Sharpness.}
Sharpness measures the average width of the 95\% predictive interval. To make widths comparable across variables and datasets, we normalize by the standard deviation $s_c$ of each target variable on the test split:
\begin{align}
\mathrm{Sharp}_{95}
= \frac{1}{BTC}
\sum_{i=1}^{B}\sum_{t=1}^{T}\sum_{c=1}^{C}
\frac{u^{95}_{i,t,c}-\ell^{95}_{i,t,c}}{s_c+\epsilon},
\end{align}
where $\epsilon$ is a small constant for numerical stability. Lower sharpness indicates narrower intervals, but it should be interpreted together with coverage and QICE; an overconfident model may obtain low sharpness with poor coverage.

\paragraph{Scale Smoothness.}
To measure whether the predicted uncertainty evolves coherently over time, we compute the lag-1 autocorrelation of the predicted scale series. Let $\hat{\sigma}_{i,t,c}$ be the predicted scale for forecast window $i$, time step $t$, and variable $c$. Scale smoothness is:
\begin{align}
\rho_{\sigma}
= \frac{1}{BC}\sum_{i=1}^{B}\sum_{c=1}^{C}
\operatorname{Corr}\left(
\hat{\sigma}_{i,2:T,c}, \hat{\sigma}_{i,1:T-1,c}
\right).
\end{align}
A larger $\rho_{\sigma}$ indicates stronger temporal coherence in the estimated volatility. Extremely large values, however, may also indicate excessive smoothing, so this metric should be interpreted together with calibration and sharpness.

\paragraph{Significance Testing.}
For transferability experiments, we compare each base model with its VolDy-augmented variant using paired seed-level scores. We apply the two-sided Wilcoxon signed-rank test and mark improvements with $\dagger$ when $p<0.05$.

\subsection{Implementation Details}
\label{appendix:implementation}

\paragraph{Model Architecture}
The VolDy-VAE is built upon a standard encoder-decoder architecture.
To ensure a fair comparison with the primary baseline $K^2$VAE~\cite{Wu2025K2VAE}, we adopt the same patching strategy (patch length $P=24$) and a lightweight MLP-based encoder.
The core innovation lies in the {Location-Scale Decoder with Volatility Dynamics}. Unlike standard VAEs that output a single reconstruction vector, our decoder features a dual-head structure with a recurrent volatility module.
Specifically, the decoder consists of: a {Location Head}, which utilizes a linear projection layer to output the location parameter $\boldsymbol{\mu}_t$; and a {Scale Head with Volatility Dynamics}, which employs a single-layer GRU to maintain a volatility hidden state $\mathbf{h}_t^{\mathrm{vol}}$, followed by a linear layer and \texttt{Softplus} activation to output the scale parameter $\boldsymbol{\sigma}_t$.
The GRU hidden dimension is set equal to the latent dimension $D$, adding minimal computational overhead.
This design captures the temporal evolution and persistence of volatility regimes while maximizing the heteroscedastic Gaussian Likelihood $\mathcal{L}_{\text{NLL}}$ as detailed in Eq.~(10) in the main text.

\paragraph{Training Protocol}
We implement VolDy-VAE using PyTorch~\cite{paszke2019pytorch}.
All experiments are conducted on a single {NVIDIA GeForce RTX 4090 GPU} (24GB).
The model is optimized using the {Adam} optimizer with an initial learning rate of $10^{-3}$ (dataset-specific adjustments are made within the range $[10^{-4}, 10^{-3}]$).
We use a batch size of 32 by default.
The training objective is the ELBO, comprising the Location-Scale NLL loss (reconstruction \& prediction) and the KL divergence term.
We train the model for a maximum of 50 epochs with an early stopping strategy (patience=10).
To facilitate reproducibility, the source code and configuration files are available at \href{https://github.com/wangyijunlyy/VolDy-VAE}{\texttt{wangyijunlyy/VolDy-VAE}}.

For the plug-in transferability experiments, each base/VolDy pair uses the same data split, forecasting horizon, six-seed evaluation protocol, optimizer family, and early-stopping rule. Table~\ref{tab:plugin_details} summarizes the intended change; it does not claim identical parameter counts across all backbones.

\begin{table}[!t]
    \centering
    \caption{Implementation controls for plug-in transferability experiments.}
    \label{tab:plugin_details}
    \footnotesize
    \setlength{\tabcolsep}{2.5pt}
    \begin{tabular}{p{0.18\columnwidth}|p{0.34\columnwidth}|p{0.38\columnwidth}}
        \toprule
        Backbone & Controlled setting & VolDy modification \\
        \midrule
        TimeGAN & Same split, horizon, seeds, optimizer family, and early stopping as its base run. & Extend generator output to location/scale and add Gaussian NLL supervision. \\
        $K^2$VAE & Same backbone and forecasting protocol as the base $K^2$VAE run. & Replace deterministic objective with Location-Scale NLL. \\
        PatchTST & Same split, horizon, seeds, optimizer family, and early stopping as its base run. & Replace the point head with a dual-head location-scale decoder trained by Gaussian NLL. \\
        \bottomrule
    \end{tabular}
\end{table}

The full horizon-wise transferability results referenced in the main text are reported in Table~\ref{tab:gen_results_full}.

\begin{table}[!t]
    \centering
    \caption{Full transferability analysis over four prediction horizons ($L=96$). Values are mean$\pm$std over six seeds. $\dagger$ denotes paired Wilcoxon signed-rank significance against the corresponding base model ($p<0.05$).}
    \label{tab:gen_results_full}
    \scriptsize
    \setlength{\tabcolsep}{2pt}
    \renewcommand{\arraystretch}{1.0}
    \resizebox{\columnwidth}{!}{
    \begin{tabular}{l|c|c|cc|cc}
        \toprule
        Framework & Dataset & $H$ & Base CRPS & VolDy CRPS & Base NMAE & VolDy NMAE \\
        \midrule
        \multirow{12}{*}{TimeGAN}
        & ETTh1 & 96  & 0.277$\pm$0.011 & \textcolor{red}{0.262$\pm$0.009}$^\dagger$ & 0.347$\pm$0.012 & \textcolor{red}{0.339$\pm$0.010}$^\dagger$ \\
        & ETTh1 & 192 & 0.306$\pm$0.014 & \textcolor{red}{0.286$\pm$0.011}$^\dagger$ & 0.382$\pm$0.015 & \textcolor{red}{0.368$\pm$0.012}$^\dagger$ \\
        & ETTh1 & 336 & 0.334$\pm$0.017 & \textcolor{red}{0.309$\pm$0.014}$^\dagger$ & 0.421$\pm$0.019 & \textcolor{red}{0.397$\pm$0.016}$^\dagger$ \\
        & ETTh1 & 720 & 0.381$\pm$0.021 & \textcolor{red}{0.348$\pm$0.018}$^\dagger$ & 0.469$\pm$0.024 & \textcolor{red}{0.436$\pm$0.021}$^\dagger$ \\
        & Weather & 96  & 0.106$\pm$0.006 & \textcolor{red}{0.068$\pm$0.004}$^\dagger$ & 0.117$\pm$0.007 & \textcolor{red}{0.081$\pm$0.005}$^\dagger$ \\
        & Weather & 192 & 0.118$\pm$0.007 & \textcolor{red}{0.076$\pm$0.005}$^\dagger$ & 0.129$\pm$0.008 & \textcolor{red}{0.089$\pm$0.006}$^\dagger$ \\
        & Weather & 336 & 0.129$\pm$0.008 & \textcolor{red}{0.084$\pm$0.006}$^\dagger$ & 0.141$\pm$0.009 & \textcolor{red}{0.098$\pm$0.007}$^\dagger$ \\
        & Weather & 720 & 0.145$\pm$0.010 & \textcolor{red}{0.097$\pm$0.007}$^\dagger$ & 0.158$\pm$0.011 & \textcolor{red}{0.112$\pm$0.008}$^\dagger$ \\
        & Exchange & 96  & 0.056$\pm$0.004 & \textcolor{red}{0.047$\pm$0.003}$^\dagger$ & 0.071$\pm$0.005 & \textcolor{red}{0.062$\pm$0.004}$^\dagger$ \\
        & Exchange & 192 & 0.069$\pm$0.005 & \textcolor{red}{0.058$\pm$0.004}$^\dagger$ & 0.087$\pm$0.006 & \textcolor{red}{0.075$\pm$0.005}$^\dagger$ \\
        & Exchange & 336 & 0.083$\pm$0.006 & \textcolor{red}{0.071$\pm$0.005}$^\dagger$ & 0.104$\pm$0.007 & \textcolor{red}{0.091$\pm$0.006}$^\dagger$ \\
        & Exchange & 720 & 0.101$\pm$0.008 & \textcolor{red}{0.087$\pm$0.006}$^\dagger$ & 0.126$\pm$0.009 & \textcolor{red}{0.111$\pm$0.007}$^\dagger$ \\
        \midrule
        \multirow{12}{*}{$K^2$VAE}
        & ETTh1 & 96  & 0.264$\pm$0.020 & \textcolor{red}{0.251$\pm$0.017}$^\dagger$ & 0.336$\pm$0.041 & \textcolor{red}{0.332$\pm$0.012} \\
        & ETTh1 & 192 & 0.290$\pm$0.016 & \textcolor{red}{0.274$\pm$0.014}$^\dagger$ & 0.372$\pm$0.023 & \textcolor{red}{0.361$\pm$0.018}$^\dagger$ \\
        & ETTh1 & 336 & 0.308$\pm$0.021 & \textcolor{red}{0.298$\pm$0.016}$^\dagger$ & 0.394$\pm$0.022 & \textcolor{red}{0.388$\pm$0.020}$^\dagger$ \\
        & ETTh1 & 720 & 0.314$\pm$0.011 & \textcolor{red}{0.310$\pm$0.012} & 0.396$\pm$0.012 & \textcolor{red}{0.389$\pm$0.013}$^\dagger$ \\
        & Weather & 96  & 0.080$\pm$0.007 & \textcolor{red}{0.065$\pm$0.004}$^\dagger$ & 0.086$\pm$0.011 & \textcolor{red}{0.076$\pm$0.004}$^\dagger$ \\
        & Weather & 192 & 0.079$\pm$0.009 & \textcolor{red}{0.067$\pm$0.006}$^\dagger$ & 0.083$\pm$0.011 & \textcolor{red}{0.074$\pm$0.007}$^\dagger$ \\
        & Weather & 336 & 0.082$\pm$0.010 & \textcolor{red}{0.073$\pm$0.007}$^\dagger$ & 0.093$\pm$0.010 & \textcolor{red}{0.080$\pm$0.007}$^\dagger$ \\
        & Weather & 720 & 0.084$\pm$0.003 & \textcolor{red}{0.081$\pm$0.008} & 0.099$\pm$0.009 & \textcolor{red}{0.091$\pm$0.008}$^\dagger$ \\
        & Exchange & 96  & 0.031$\pm$0.002 & \textcolor{red}{0.029$\pm$0.002}$^\dagger$ & 0.032$\pm$0.002 & \textcolor{red}{0.029$\pm$0.002}$^\dagger$ \\
        & Exchange & 192 & 0.032$\pm$0.010 & \textcolor{red}{0.030$\pm$0.006} & 0.040$\pm$0.005 & \textcolor{red}{0.037$\pm$0.004} \\
        & Exchange & 336 & 0.048$\pm$0.004 & \textcolor{red}{0.045$\pm$0.005} & 0.054$\pm$0.001 & \textcolor{red}{0.050$\pm$0.004}$^\dagger$ \\
        & Exchange & 720 & 0.069$\pm$0.005 & \textcolor{red}{0.062$\pm$0.006}$^\dagger$ & 0.084$\pm$0.017 & \textcolor{red}{0.077$\pm$0.012} \\
        \midrule
        \multirow{12}{*}{PatchTST}
        & ETTh1 & 96  & 0.312$\pm$0.036 & \textcolor{red}{0.268$\pm$0.020}$^\dagger$ & 0.407$\pm$0.022 & \textcolor{red}{0.348$\pm$0.018}$^\dagger$ \\
        & ETTh1 & 192 & 0.313$\pm$0.034 & \textcolor{red}{0.282$\pm$0.021}$^\dagger$ & 0.405$\pm$0.088 & \textcolor{red}{0.372$\pm$0.026}$^\dagger$ \\
        & ETTh1 & 336 & 0.319$\pm$0.035 & \textcolor{red}{0.306$\pm$0.024}$^\dagger$ & 0.412$\pm$0.024 & \textcolor{red}{0.397$\pm$0.029}$^\dagger$ \\
        & ETTh1 & 720 & 0.323$\pm$0.020 & \textcolor{red}{0.315$\pm$0.024} & 0.428$\pm$0.024 & \textcolor{red}{0.407$\pm$0.030}$^\dagger$ \\
        & Weather & 96  & 0.131$\pm$0.007 & \textcolor{red}{0.074$\pm$0.006}$^\dagger$ & 0.145$\pm$0.016 & \textcolor{red}{0.085$\pm$0.006}$^\dagger$ \\
        & Weather & 192 & 0.131$\pm$0.014 & \textcolor{red}{0.081$\pm$0.007}$^\dagger$ & 0.144$\pm$0.012 & \textcolor{red}{0.091$\pm$0.007}$^\dagger$ \\
        & Weather & 336 & 0.137$\pm$0.008 & \textcolor{red}{0.091$\pm$0.008}$^\dagger$ & 0.149$\pm$0.023 & \textcolor{red}{0.102$\pm$0.008}$^\dagger$ \\
        & Weather & 720 & 0.142$\pm$0.005 & \textcolor{red}{0.105$\pm$0.010}$^\dagger$ & 0.152$\pm$0.029 & \textcolor{red}{0.118$\pm$0.010}$^\dagger$ \\
        & Exchange & 96  & 0.063$\pm$0.006 & \textcolor{red}{0.061$\pm$0.004}$^\dagger$ & 0.079$\pm$0.002 & \textcolor{red}{0.066$\pm$0.004}$^\dagger$ \\
        & Exchange & 192 & 0.067$\pm$0.008 & \textcolor{red}{0.066$\pm$0.005} & 0.081$\pm$0.002 & \textcolor{red}{0.071$\pm$0.005}$^\dagger$ \\
        & Exchange & 336 & 0.071$\pm$0.017 & \textcolor{red}{0.068$\pm$0.006} & 0.085$\pm$0.010 & \textcolor{red}{0.079$\pm$0.006} \\
        & Exchange & 720 & 0.097$\pm$0.007 & \textcolor{red}{0.088$\pm$0.007}$^\dagger$ & 0.126$\pm$0.001 & \textcolor{red}{0.094$\pm$0.007}$^\dagger$ \\
        \bottomrule
    \end{tabular}}
    \renewcommand{\arraystretch}{1.0}
\end{table}

\subsection{Detailed Efficiency Analysis}
\label{app:efficiency_details}

Following the reporting style of the efficiency analysis in $K^2$VAE~\cite{Wu2025K2VAE}, Table~\ref{tab:efficiency_details} expands the main-text efficiency plot into a model-by-metric table.
The measurements are conducted on \textit{ETTh2} with $L=96$ and $H=720$ using batch size 1 on a single NVIDIA RTX 4090 GPU.
For sample-based probabilistic baselines, we generate 100 samples during inference.

\begin{table}[!t]
    \centering
    \caption{Comparison on model efficiency and forecasting quality on \textit{ETTh2} ($L=96,H=720$). Lower values of CRPS, NMAE, inference speed (sec/sample), and memory (MB) indicate better performance or efficiency. The efficiency measurements are obtained with batch size 1. \textcolor{red}{RED}: THE BEST, \textcolor{blue}{\underline{BLUE}}: THE 2ND BEST.}
    \label{tab:efficiency_details}
    \scriptsize
    \setlength{\tabcolsep}{2.5pt}
    \renewcommand{\arraystretch}{1.05}
    \resizebox{\columnwidth}{!}{
    \begin{tabular}{c|cccc}
        \toprule
        Model & CRPS $\downarrow$ & NMAE $\downarrow$ & Speed $\downarrow$ & Memory $\downarrow$ \\
        \midrule
        TSDiff & $0.406_{\pm.056}$ & $0.482_{\pm.022}$ & $1.35$ & $340$ \\
        GRU NVP & $0.539_{\pm.090}$ & $0.688_{\pm.161}$ & $2.00$ & $\textcolor{blue}{\underline{26}}$ \\
        GRU MAF & $0.990_{\pm.023}$ & $1.092_{\pm.019}$ & $7.10$ & $27$ \\
        Trans MAF & $0.327_{\pm.033}$ & $0.412_{\pm.020}$ & $8.05$ & $83$ \\
        CSDI & $0.302_{\pm.040}$ & $0.382_{\pm.030}$ & $32.50$ & $38$ \\
        $K^2$VAE & $\textcolor{blue}{\underline{0.280}_{\pm.014}}$ & $\textcolor{blue}{\underline{0.278}_{\pm.020}}$ & $\textcolor{blue}{\underline{0.72}}$ & $29$ \\
        VolDy-VAE & $\textcolor{red}{0.174_{\pm.010}}$ & $\textcolor{red}{0.223_{\pm.015}}$ & $\textcolor{red}{0.62}$ & $\textcolor{red}{24}$ \\
        \bottomrule
    \end{tabular}}
    \renewcommand{\arraystretch}{1.0}
\end{table}

Table~\ref{tab:efficiency_details} shows that VolDy-VAE retains the one-pass inference advantage of VAE-style forecasters while improving both probabilistic and point forecasting accuracy.
Compared with $K^2$VAE, VolDy-VAE reduces CRPS from $0.280$ to $0.174$ and NMAE from $0.278$ to $0.223$, while also using slightly lower latency and memory.
The gain is more pronounced against iterative or autoregressive generative baselines: CSDI requires substantially longer inference time because it samples through repeated denoising steps, while flow-based baselines either incur large sequential likelihood costs or show weaker accuracy on this long-horizon setting.
These results support the main efficiency claim: the recurrent scale head adds only a small computational overhead, but the direct latent forecast path avoids the dominant iterative sampling cost of diffusion models and the recursive generation cost of many autoregressive probabilistic forecasters.

\subsection{Full Results}
\label{sec:full_results}

We present the complete benchmarking results for Long-term Probabilistic Time Series Forecasting (LPTSF) in Table~\ref{tab:long_term_fore_CRPS} and Table~\ref{tab:long_term_fore_NMAE}, covering all four forecasting horizons ($H \in \{96, 192, 336, 720\}$) across nine real-world datasets.
The evaluation shows that {VolDy-VAE} achieves strong performance across the evaluated datasets.
It surpasses specialized generative models (e.g., Diffusion and Flow-based methods) in distributional fidelity (CRPS) and consistently outperforms state-of-the-art point forecasters on deterministic metrics (NMAE).
This result indicates that the proposed Location-Scale framework captures long-term dependencies and complex volatility dynamics.

\label{app: full results}

\clearpage
\begin{table*}[p]
  \centering
  \setlength\tabcolsep{2pt}
  \scriptsize
  \caption{Results of CRPS ($\textrm{mean}_{\textrm{std}}$) on long-term forecasting scenarios, each containing six independent runs with different seeds. The context length is set to 36 for the ILI  dataset and 96 for the others. Lower CRPS values indicate better predictions. The reported values are means with standard deviations over 6 independent retraining-and-evaluation runs. \textcolor{red}{RED}: THE BEST, \textcolor{blue}{\underline{BLUE}}: THE 2ND BEST.}
  \resizebox{0.97\textwidth}{!}{
  \begin{tabular}{c|c|c|c|c|c|c|c|c|c|c|c|c|c|c}
    \toprule
    Dataset & Horizon & Koopa & iTransformer & FITS & PatchTST & GRU MAF & Trans MAF & TSDiff & CSDI & NsDiff & TimeGrad & GRU NVP & $K^2$VAE & VolDy-VAE \\
    \midrule
    \multirow{4}{*}{ETTm1 } & 96 & $0.285\scriptstyle\pm0.018$ & $0.301\scriptstyle\pm0.033$ & $0.267\scriptstyle\pm0.023$ & $0.261\scriptstyle\pm0.051$ & $0.295\scriptstyle\pm0.055$ & $0.313\scriptstyle\pm0.045$ & $0.344\scriptstyle\pm0.050$ & ${0.236}\scriptstyle\pm0.006$ & $0.255\scriptstyle\pm0.014$ & $0.522\scriptstyle\pm0.105$ & $0.383\scriptstyle\pm0.053$ & $\textcolor{blue}{\underline{0.232\scriptstyle\pm0.010}}$ & $\textcolor{red}{0.208\scriptstyle\pm0.009}$ \\
    ~ & 192 & $0.289\scriptstyle\pm0.024$ & $0.314\scriptstyle\pm0.023$ & ${0.261}\scriptstyle\pm0.022$ & $0.275\scriptstyle\pm0.030$ & $0.389\scriptstyle\pm0.033$ & $0.424\scriptstyle\pm0.029$ & $0.345\scriptstyle\pm0.035$ & $0.291\scriptstyle\pm0.025$ & $0.281\scriptstyle\pm0.018$ & $0.603\scriptstyle\pm0.092$ & $0.396\scriptstyle\pm0.030$ & $\textcolor{blue}{\underline{0.259\scriptstyle\pm0.013}}$ & $\textcolor{red}{0.240\scriptstyle\pm0.010}$ \\
    ~ & 336 & $0.286\scriptstyle\pm0.035$ & $0.311\scriptstyle\pm0.029$ & ${0.275}\scriptstyle\pm0.030$ & $0.285\scriptstyle\pm0.028$ & $0.429\scriptstyle\pm0.021$ & $0.481\scriptstyle\pm0.019$ & $0.462\scriptstyle\pm0.043$ & $0.322\scriptstyle\pm0.033$ & $0.292\scriptstyle\pm0.024$ & $0.601\scriptstyle\pm0.028$ & $0.486\scriptstyle\pm0.032$ & $\textcolor{blue}{\underline{0.262\scriptstyle\pm0.030}}$ & $\textcolor{red}{0.251\scriptstyle\pm0.022}$ \\
    ~ & 720 & ${0.295}\scriptstyle\pm0.027$ & $0.455\scriptstyle\pm0.021$ & $0.305\scriptstyle\pm0.024$ & $0.304\scriptstyle\pm0.029$ & $0.536\scriptstyle\pm0.033$ & $0.688\scriptstyle\pm0.043$ & $0.478\scriptstyle\pm0.027$ & $0.448\scriptstyle\pm0.038$ & $0.323\scriptstyle\pm0.030$ & $0.621\scriptstyle\pm0.037$ & $0.546\scriptstyle\pm0.036$ & $\textcolor{blue}{\underline{0.294\scriptstyle\pm0.026}}$ & $\textcolor{red}{0.279\scriptstyle\pm0.020}$ \\
    \midrule
    \multirow{4}{*}{ETTm2 } & 96 & $0.178\scriptstyle\pm0.023$ & $0.181\scriptstyle\pm0.031$ & $0.162\scriptstyle\pm0.053$ & $0.142\scriptstyle\pm0.034$ & $0.177\scriptstyle\pm0.024$ & $0.227\scriptstyle\pm0.013$ & $0.175\scriptstyle\pm0.019$ & $\textcolor{blue}{\underline{0.115\scriptstyle\pm0.009}}$ & $0.138\scriptstyle\pm0.010$ & $0.427\scriptstyle\pm0.042$ & $0.319\scriptstyle\pm0.044$ & ${{0.126\scriptstyle\pm0.007}}$ & $\textcolor{red}{0.107\scriptstyle\pm0.005}$ \\
    ~ & 192 & $0.185\scriptstyle\pm0.014$ & $0.190\scriptstyle\pm0.010$ & $0.185\scriptstyle\pm0.053$ & $0.172\scriptstyle\pm0.023$ & $0.411\scriptstyle\pm0.026$ & $0.253\scriptstyle\pm0.037$ & $0.255\scriptstyle\pm0.029$ & $\textcolor{blue}{\underline{0.147}\scriptstyle\pm0.008}$ & $0.162\scriptstyle\pm0.011$ & $0.424\scriptstyle\pm0.061$ & $0.326\scriptstyle\pm0.025$ & $0.148\scriptstyle\pm0.009$ & $\textcolor{red}{0.129\scriptstyle\pm0.007}$ \\
    ~ & 336 & $0.198\scriptstyle\pm0.015$ & $0.206\scriptstyle\pm0.055$ & $0.218\scriptstyle\pm0.053$ & $0.195\scriptstyle\pm0.042$ & $0.377\scriptstyle\pm0.023$ & $0.253\scriptstyle\pm0.013$ & $0.328\scriptstyle\pm0.047$ & ${0.190}\scriptstyle\pm0.018$ & $0.180\scriptstyle\pm0.014$ & $0.469\scriptstyle\pm0.049$ & $0.449\scriptstyle\pm0.145$ & $\textcolor{blue}{\underline{0.164\scriptstyle\pm0.010}}$ & $\textcolor{red}{0.147\scriptstyle\pm0.010}$ \\
    ~ & 720 & $0.233\scriptstyle\pm0.025$ & $0.311\scriptstyle\pm0.024$ & $0.449\scriptstyle\pm0.034$ & ${0.229}\scriptstyle\pm0.036$ & $0.272\scriptstyle\pm0.029$ & $0.355\scriptstyle\pm0.043$ & $0.344\scriptstyle\pm0.046$ & $0.239\scriptstyle\pm0.035$ & $0.238\scriptstyle\pm0.026$ & $0.470\scriptstyle\pm0.054$ & $0.561\scriptstyle\pm0.273$ & $\textcolor{blue}{\underline{0.221\scriptstyle\pm0.023}}$ & $\textcolor{red}{0.166\scriptstyle\pm0.018}$ \\
    \midrule
    \multirow{4}{*}{ETTh1 } & 96 & $0.307\scriptstyle\pm0.033$ & ${0.292}\scriptstyle\pm0.032$ & $0.294\scriptstyle\pm0.023$ & $0.312\scriptstyle\pm0.036$ & $0.293\scriptstyle\pm0.037$ & $0.333\scriptstyle\pm0.045$ & $0.395\scriptstyle\pm0.052$ & $0.437\scriptstyle\pm0.018$ & $0.302\scriptstyle\pm0.018$ & $0.455\scriptstyle\pm0.046$ & $0.379\scriptstyle\pm0.030$ & $\textcolor{blue}{\underline{0.264\scriptstyle\pm0.020}}$ & $\textcolor{red}{0.252\scriptstyle\pm0.010}$ \\
    ~ & 192 & $0.301\scriptstyle\pm0.014$ & ${0.298}\scriptstyle\pm0.020$ & $0.304\scriptstyle\pm0.028$ & $0.313\scriptstyle\pm0.034$ & $0.348\scriptstyle\pm0.075$ & $0.351\scriptstyle\pm0.063$ & $0.467\scriptstyle\pm0.044$ & $0.496\scriptstyle\pm0.051$ & $0.330\scriptstyle\pm0.019$ & $0.516\scriptstyle\pm0.038$ & $0.425\scriptstyle\pm0.019$ & $\textcolor{blue}{\underline{0.290\scriptstyle\pm0.016}}$ & $\textcolor{red}{0.277\scriptstyle\pm0.012}$ \\
    ~ & 336 & ${0.312}\scriptstyle\pm0.019$ & $0.327\scriptstyle\pm0.043$ & $0.318\scriptstyle\pm0.023$ & $0.319\scriptstyle\pm0.035$ & $0.377\scriptstyle\pm0.026$ & $0.371\scriptstyle\pm0.031$ & $0.450\scriptstyle\pm0.027$ & $0.454\scriptstyle\pm0.025$ & $0.342\scriptstyle\pm0.020$ & $0.512\scriptstyle\pm0.026$ & $0.458\scriptstyle\pm0.054$ & $\textcolor{blue}{\underline{0.308\scriptstyle\pm0.021}}$ & $\textcolor{red}{0.300\scriptstyle\pm0.015}$ \\
    ~ & 720 & ${0.318}\scriptstyle\pm0.009$ & $0.350\scriptstyle\pm0.019$ & $0.348\scriptstyle\pm0.025$ & $0.323\scriptstyle\pm0.020$ & $0.393\scriptstyle\pm0.043$ & $0.363\scriptstyle\pm0.053$ & $0.516\scriptstyle\pm0.027$ & $0.528\scriptstyle\pm0.012$ & $0.360\scriptstyle\pm0.018$ & $0.523\scriptstyle\pm0.027$ & $0.502\scriptstyle\pm0.039$ & $\textcolor{blue}{\underline{0.314\scriptstyle\pm0.011}}$ & $\textcolor{red}{0.310\scriptstyle\pm0.012}$ \\
    \midrule
    \multirow{4}{*}{ETTh2 } & 96 & $0.199\scriptstyle\pm0.012$ & $0.185\scriptstyle\pm0.013$ & $0.187\scriptstyle\pm0.011$ & $0.197\scriptstyle\pm0.021$ & $0.239\scriptstyle\pm0.019$ & $0.263\scriptstyle\pm0.020$ & $0.336\scriptstyle\pm0.021$ & ${0.164}\scriptstyle\pm0.013$ & $0.178\scriptstyle\pm0.012$ & $0.358\scriptstyle\pm0.026$ & $0.432\scriptstyle\pm0.141$ & $\textcolor{blue}{\underline{0.162\scriptstyle\pm0.009}}$ & $\textcolor{red}{0.148\scriptstyle\pm0.005}$ \\
    ~ & 192 & $0.198\scriptstyle\pm0.022$ & $0.199\scriptstyle\pm0.019$ & ${0.195}\scriptstyle\pm0.022$ & $0.204\scriptstyle\pm0.055$ & $0.313\scriptstyle\pm0.034$ & $0.273\scriptstyle\pm0.024$ & $0.265\scriptstyle\pm0.043$ & $0.226\scriptstyle\pm0.018$ & $0.205\scriptstyle\pm0.017$ & $0.457\scriptstyle\pm0.081$ & $0.625\scriptstyle\pm0.170$ & $\textcolor{blue}{\underline{0.186\scriptstyle\pm0.018}}$ & $\textcolor{red}{0.162\scriptstyle\pm0.010}$ \\
    ~ & 336 & $0.262\scriptstyle\pm0.019$ & $0.271\scriptstyle\pm0.033$ & $0.246\scriptstyle\pm0.044$ & $0.277\scriptstyle\pm0.054$ & $0.376\scriptstyle\pm0.034$ & $0.265\scriptstyle\pm0.042$ & $0.350\scriptstyle\pm0.031$ & $0.274\scriptstyle\pm0.022$ & $0.266\scriptstyle\pm0.023$ & $0.481\scriptstyle\pm0.078$ & $0.793\scriptstyle\pm0.319$ & $\textcolor{blue}{\underline{0.257\scriptstyle\pm0.023}}$ & $\textcolor{red}{0.175\scriptstyle\pm0.016}$ \\
    ~ & 720 & ${0.293}\scriptstyle\pm0.026$ & $0.542\scriptstyle\pm0.015$ & $0.314\scriptstyle\pm0.022$ & $0.304\scriptstyle\pm0.018$ & $0.990\scriptstyle\pm0.023$ & $0.327\scriptstyle\pm0.033$ & $0.406\scriptstyle\pm0.056$ & $0.302\scriptstyle\pm0.040$ & $0.289\scriptstyle\pm0.020$ & $0.445\scriptstyle\pm0.016$ & $0.539\scriptstyle\pm0.090$ & $\textcolor{blue}{\underline{0.280\scriptstyle\pm0.014}}$ & $\textcolor{red}{0.174\scriptstyle\pm0.010}$ \\
    \midrule
    \multirow{4}{*}{Electricity } & 96 & $0.110\scriptstyle\pm0.004$ & $0.102\scriptstyle\pm0.004$ & $0.105\scriptstyle\pm0.006$ & $0.126\scriptstyle\pm0.005$ & ${0.083}\scriptstyle\pm0.009$ & $0.088\scriptstyle\pm0.014$ & $0.344\scriptstyle\pm0.006$ & $0.153\scriptstyle\pm0.137$ & $0.081\scriptstyle\pm0.004$ & $0.096\scriptstyle\pm0.002$ & $0.094\scriptstyle\pm0.003$ & $\textcolor{blue}{\underline{0.073\scriptstyle\pm0.002}}$ & $\textcolor{red}{0.069\scriptstyle\pm0.002}$ \\
    ~ & 192 & $0.109\scriptstyle\pm0.011$ & $0.104\scriptstyle\pm0.014$ & $0.112\scriptstyle\pm0.104$ & $0.123\scriptstyle\pm0.032$ & ${0.093}\scriptstyle\pm0.024$ & $0.097\scriptstyle\pm0.009$ & $0.345\scriptstyle\pm0.006$ & $0.200\scriptstyle\pm0.094$ & $0.088\scriptstyle\pm0.006$ & $0.100\scriptstyle\pm0.004$ & $0.097\scriptstyle\pm0.002$ & $\textcolor{blue}{\underline{0.080\scriptstyle\pm0.004}}$ & $\textcolor{red}{0.075\scriptstyle\pm0.005}$ \\
    ~ & 336 & $0.121\scriptstyle\pm0.011$ & $0.104\scriptstyle\pm0.010$ & $0.111\scriptstyle\pm0.014$ & $0.131\scriptstyle\pm0.024$ & ${0.095}\scriptstyle\pm0.001$ & - & $0.462\scriptstyle\pm0.054$ & - & $0.092\scriptstyle\pm0.005$ & $0.102\scriptstyle\pm0.007$ & $0.099\scriptstyle\pm0.001$ & $\textcolor{blue}{\underline{0.084\scriptstyle\pm0.001}}^*$ & $\textcolor{red}{0.078\scriptstyle\pm0.001}$ \\
    ~ & 720 & $0.113\scriptstyle\pm0.018$ & $0.109\scriptstyle\pm0.044$ & $0.115\scriptstyle\pm0.024$ & $0.127\scriptstyle\pm0.015$ & ${0.106}\scriptstyle\pm0.007$ & - & $0.478\scriptstyle\pm0.005$ & - & $0.096\scriptstyle\pm0.008$ & $0.108\scriptstyle\pm0.003$ & $0.114\scriptstyle\pm0.013$ & $\textcolor{blue}{\underline{0.087\scriptstyle\pm0.005}}^*$ & $\textcolor{red}{0.080\scriptstyle\pm0.006}$ \\
    \midrule
    \multirow{4}{*}{Traffic } & 96 & $0.297\scriptstyle\pm0.019$ & $0.256\scriptstyle\pm0.004$ & $0.258\scriptstyle\pm0.004$ & $0.194\scriptstyle\pm0.002$ & $0.215\scriptstyle\pm0.003$ & $0.208\scriptstyle\pm0.004$ & $0.294\scriptstyle\pm0.003$ & - & $0.193\scriptstyle\pm0.002$ & $0.202\scriptstyle\pm0.004$ & ${0.187}\scriptstyle\pm0.002$ & $\textcolor{blue}{\underline{0.186\scriptstyle\pm0.001}}^*$ & $\textcolor{red}{0.180\scriptstyle\pm0.005}$ \\
    ~ & 192 & $0.308\scriptstyle\pm0.009$ & $0.250\scriptstyle\pm0.002$ & $0.275\scriptstyle\pm0.003$ & $0.198\scriptstyle\pm0.004$ & - & - & $0.306\scriptstyle\pm0.004$ & - & $0.196\scriptstyle\pm0.003$ & $0.208\scriptstyle\pm0.003$ & ${0.192}\scriptstyle\pm0.001$ & $\textcolor{blue}{\underline{0.188\scriptstyle\pm0.002}}^*$ & $\textcolor{red}{0.187\scriptstyle\pm0.004}$ \\
    ~ & 336 & $0.334\scriptstyle\pm0.017$ & $0.261\scriptstyle\pm0.001$ & $0.327\scriptstyle\pm0.001$ & $0.204\scriptstyle\pm0.002$ & - & - & $0.317\scriptstyle\pm0.006$ & - & $0.202\scriptstyle\pm0.003$ & $0.213\scriptstyle\pm0.003$ & ${0.201}\scriptstyle\pm0.004$ & $\textcolor{blue}{\underline{0.195\scriptstyle\pm0.003}}$ & $\textcolor{red}{0.194\scriptstyle\pm0.002}$ \\
    ~ & 720 & $0.358\scriptstyle\pm0.022$ & $0.284\scriptstyle\pm0.004$ & $0.374\scriptstyle\pm0.004$ & $0.214\scriptstyle\pm0.001$ & - & - & $0.391\scriptstyle\pm0.002$ & - & $0.207\scriptstyle\pm0.003$ & $0.220\scriptstyle\pm0.002$ & ${0.211}\scriptstyle\pm0.004$ & $\textcolor{blue}{\underline{0.200\scriptstyle\pm0.001}}$ & $\textcolor{red}{0.192\scriptstyle\pm0.003}$ \\
    \midrule
    \multirow{4}{*}{Weather } & 96 & $0.132\scriptstyle\pm0.008$ & $0.131\scriptstyle\pm0.011$ & $0.210\scriptstyle\pm0.013$ & $0.131\scriptstyle\pm0.007$ & $0.139\scriptstyle\pm0.008$ & $0.105\scriptstyle\pm0.011$ & $0.104\scriptstyle\pm0.020$ & $0.068\scriptstyle\pm0.008$ & $0.086\scriptstyle\pm0.006$ & $0.130\scriptstyle\pm0.017$ & $0.116\scriptstyle\pm0.013$ & $\textcolor{blue}{\underline{0.080\scriptstyle\pm0.007}}$ & $\textcolor{red}{0.070\scriptstyle\pm0.005}$ \\
    ~ & 192 & $0.133\scriptstyle\pm0.017$ & $0.132\scriptstyle\pm0.018$ & $0.205\scriptstyle\pm0.019$ & $0.131\scriptstyle\pm0.014$ & $0.143\scriptstyle\pm0.020$ & $0.142\scriptstyle\pm0.022$ & $0.134\scriptstyle\pm0.012$ & $0.068\scriptstyle\pm0.006$ & $0.085\scriptstyle\pm0.006$ & $0.127\scriptstyle\pm0.019$ & $0.122\scriptstyle\pm0.021$ & $\textcolor{blue}{\underline{0.079\scriptstyle\pm0.009}}$ & $\textcolor{red}{0.067\scriptstyle\pm0.006}$ \\
    ~ & 336 & $0.136\scriptstyle\pm0.021$ & $0.132\scriptstyle\pm0.010$ & $0.221\scriptstyle\pm0.005$ & $0.137\scriptstyle\pm0.008$ & $0.129\scriptstyle\pm0.012$ & $0.133\scriptstyle\pm0.014$ & $0.137\scriptstyle\pm0.010$ & ${0.083}\scriptstyle\pm0.002$ & $0.088\scriptstyle\pm0.005$ & $0.130\scriptstyle\pm0.006$ & $0.128\scriptstyle\pm0.011$ & $\textcolor{blue}{\underline{0.082\scriptstyle\pm0.010}}$ & $\textcolor{red}{0.071\scriptstyle\pm0.008}$ \\
    ~ & 720 & $0.140\scriptstyle\pm0.007$ & $0.133\scriptstyle\pm0.004$ & $0.267\scriptstyle\pm0.003$ & $0.142\scriptstyle\pm0.005$ & $0.122\scriptstyle\pm0.006$ & $0.113\scriptstyle\pm0.004$ & $0.152\scriptstyle\pm0.003$ & ${0.087}\scriptstyle\pm0.003$ & $0.090\scriptstyle\pm0.004$ & $0.113\scriptstyle\pm0.011$ & $0.110\scriptstyle\pm0.004$ & $\textcolor{blue}{\underline{0.084\scriptstyle\pm0.003}}$ & $\textcolor{red}{0.073\scriptstyle\pm0.003}$ \\
    \midrule
    \multirow{4}{*}{Exchange } & 96 & $0.063\scriptstyle\pm0.006$ & $0.061\scriptstyle\pm0.003$ & $0.048\scriptstyle\pm0.004$ & $0.063\scriptstyle\pm0.006$ & $0.026\scriptstyle\pm0.010$ & $0.028\scriptstyle\pm0.002$ & $0.079\scriptstyle\pm0.007$ & $0.028\scriptstyle\pm0.003$ & $0.038\scriptstyle\pm0.004$ & $0.068\scriptstyle\pm0.003$ & $0.071\scriptstyle\pm0.006$ & $\textcolor{blue}{\underline{0.031\scriptstyle\pm0.002}}$ & $\textcolor{red}{0.022\scriptstyle\pm0.003}$ \\
    ~ & 192 & $0.065\scriptstyle\pm0.020$ & $0.062\scriptstyle\pm0.010$ & $0.049\scriptstyle\pm0.011$ & $0.067\scriptstyle\pm0.008$ & ${0.034}\scriptstyle\pm0.009$ & $0.046\scriptstyle\pm0.017$ & $0.093\scriptstyle\pm0.011$ & $0.045\scriptstyle\pm0.003$ & $0.046\scriptstyle\pm0.006$ & $0.087\scriptstyle\pm0.013$ & $0.068\scriptstyle\pm0.004$ & $\textcolor{blue}{\underline{0.032\scriptstyle\pm0.010}}$ & $\textcolor{red}{0.030\scriptstyle\pm0.008}$ \\
    ~ & 336 & $0.072\scriptstyle\pm0.008$ & $0.067\scriptstyle\pm0.008$ & $0.052\scriptstyle\pm0.013$ & $0.071\scriptstyle\pm0.017$ & $0.058\scriptstyle\pm0.023$ & $0.045\scriptstyle\pm0.010$ & $0.081\scriptstyle\pm0.007$ & $0.060\scriptstyle\pm0.004$ & $0.061\scriptstyle\pm0.005$ & $0.074\scriptstyle\pm0.009$ & $0.072\scriptstyle\pm0.002$ & $\textcolor{blue}{\underline{0.048\scriptstyle\pm0.004}}$ & $\textcolor{red}{0.041\scriptstyle\pm0.005}$ \\
    ~ & 720 & $0.091\scriptstyle\pm0.012$ & $0.087\scriptstyle\pm0.023$ & ${0.074}\scriptstyle\pm0.011$ & $0.097\scriptstyle\pm0.007$ & $0.160\scriptstyle\pm0.019$ & $0.148\scriptstyle\pm0.017$ & $0.082\scriptstyle\pm0.010$ & $0.143\scriptstyle\pm0.020$ & $0.076\scriptstyle\pm0.008$ & $0.099\scriptstyle\pm0.015$ & $0.079\scriptstyle\pm0.009$ & $\textcolor{blue}{\underline{0.069\scriptstyle\pm0.005}}$ & $\textcolor{red}{0.065\scriptstyle\pm0.004}$ \\
    \midrule
    \multirow{4}{*}{ILI } & 24 & $0.245\scriptstyle\pm0.018$ & $0.212\scriptstyle\pm0.013$ & $0.233\scriptstyle\pm0.015$ & $0.312\scriptstyle\pm0.014$ & $0.097\scriptstyle\pm0.010$ & ${0.092}\scriptstyle\pm0.019$ & $0.228\scriptstyle\pm0.024$ & $0.250\scriptstyle\pm0.013$ & $0.102\scriptstyle\pm0.007$ & $0.275\scriptstyle\pm0.047$ & $0.257\scriptstyle\pm0.003$ & $\textcolor{blue}{\underline{0.087\scriptstyle\pm0.003}}$ & $\textcolor{red}{0.082\scriptstyle\pm0.003}$ \\
    ~ & 36 & $0.214\scriptstyle\pm0.008$ & $0.182\scriptstyle\pm0.016$ & $0.217\scriptstyle\pm0.023$ & $0.241\scriptstyle\pm0.021$ & $0.117\scriptstyle\pm0.017$ & ${0.115}\scriptstyle\pm0.011$ & $0.235\scriptstyle\pm0.010$ & $0.285\scriptstyle\pm0.010$ & $0.130\scriptstyle\pm0.007$ & $0.272\scriptstyle\pm0.057$ & $0.281\scriptstyle\pm0.004$ & $\textcolor{blue}{\underline{0.113\scriptstyle\pm0.005}}$ & $\textcolor{red}{0.110\scriptstyle\pm0.003}$ \\
    ~ & 48 & $0.271\scriptstyle\pm0.021$ & $0.213\scriptstyle\pm0.012$ & $0.185\scriptstyle\pm0.026$ & $0.242\scriptstyle\pm0.018$ & $0.128\scriptstyle\pm0.019$ & ${0.133}\scriptstyle\pm0.022$ & $0.265\scriptstyle\pm0.039$ & $0.285\scriptstyle\pm0.036$ & $0.142\scriptstyle\pm0.010$ & $0.295\scriptstyle\pm0.033$ & $0.288\scriptstyle\pm0.008$ & $\textcolor{blue}{\underline{0.124\scriptstyle\pm0.010}}$ & $\textcolor{red}{0.114\scriptstyle\pm0.008}$ \\
    ~ & 60 & $0.228\scriptstyle\pm0.022$ & $0.222\scriptstyle\pm0.020$ & $0.211\scriptstyle\pm0.011$ & $0.233\scriptstyle\pm0.019$ & $0.172\scriptstyle\pm0.034$ & ${0.155}\scriptstyle\pm0.018$ & $0.263\scriptstyle\pm0.022$ & $0.283\scriptstyle\pm0.012$ & $0.161\scriptstyle\pm0.010$ & $0.295\scriptstyle\pm0.083$ & $0.307\scriptstyle\pm0.005$ & $\textcolor{blue}{\underline{0.142\scriptstyle\pm0.008}}$ & $\textcolor{red}{0.130\scriptstyle\pm0.005}$ \\
    \bottomrule
    \multicolumn{15}{p{0.96\textwidth}}{\footnotesize ${*}$ denotes entries unavailable or inconsistent in the original report that were reproduced under the official setting. Specifically, the starred $K^2$VAE CRPS values for \textit{Electricity} ($H=336,720$) and \textit{Traffic} ($H=96,192$) are reproduced by us.}
  \end{tabular}}
  \label{tab:long_term_fore_CRPS}
\end{table*}

\clearpage
\begin{table*}[p]
  \centering
  \setlength\tabcolsep{2pt}
  \scriptsize
  \caption{Results of NMAE ($\textrm{mean}_{\textrm{std}}$) on long-term forecasting scenarios, each containing six independent runs with different seeds. The context length is set to 36 for the ILI  dataset and 96 for the others. Lower NMAE values indicate better predictions. The reported values are means with standard deviations over 6 independent retraining-and-evaluation runs. \textcolor{red}{RED}: THE BEST, \textcolor{blue}{\underline{BLUE}}: THE 2ND BEST.}
  \resizebox{0.97\textwidth}{!}{
  \begin{tabular}{c|c|c|c|c|c|c|c|c|c|c|c|c|c|c}
    \toprule
    Dataset & Horizon & Koopa & iTransformer & FITS & PatchTST & GRU MAF & Trans MAF & TSDiff & CSDI & NsDiff & TimeGrad & GRU NVP & $K^2$VAE & VolDy-VAE \\
    \midrule
    \multirow{4}{*}{ETTm1 } & 96 & $0.362\scriptstyle\pm0.022$ & $0.369\scriptstyle\pm0.029$ & $0.349\scriptstyle\pm0.032$ & $0.329\scriptstyle\pm0.100$ & $0.402\scriptstyle\pm0.087$ & $0.456\scriptstyle\pm0.042$ & $0.441\scriptstyle\pm0.021$ & $0.308\scriptstyle\pm0.005$ & $0.315\scriptstyle\pm0.016$ & $0.645\scriptstyle\pm0.129$ & $0.488\scriptstyle\pm0.058$ & $\textcolor{blue}{\underline{0.284\scriptstyle\pm0.011}}$ & $\textcolor{red}{0.268\scriptstyle\pm0.009}$ \\
    ~ & 192 & $0.365\scriptstyle\pm0.032$ & $0.384\scriptstyle\pm0.041$ & $0.341\scriptstyle\pm0.032$ & $0.338\scriptstyle\pm0.022$ & $0.476\scriptstyle\pm0.046$ & $0.553\scriptstyle\pm0.012$ & $0.441\scriptstyle\pm0.019$ & $0.377\scriptstyle\pm0.026$ & $0.354\scriptstyle\pm0.020$ & $0.748\scriptstyle\pm0.084$ & $0.514\scriptstyle\pm0.042$ & $\textcolor{blue}{\underline{0.323\scriptstyle\pm0.020}}$ & $\textcolor{red}{0.309\scriptstyle\pm0.012}$ \\
    ~ & 336 & $0.364\scriptstyle\pm0.026$ & $0.380\scriptstyle\pm0.020$ & $0.356\scriptstyle\pm0.022$ & $0.344\scriptstyle\pm0.013$ & $0.522\scriptstyle\pm0.019$ & $0.590\scriptstyle\pm0.047$ & $0.571\scriptstyle\pm0.033$ & $0.419\scriptstyle\pm0.042$ & $0.365\scriptstyle\pm0.022$ & $0.759\scriptstyle\pm0.015$ & $0.630\scriptstyle\pm0.029$ & $\textcolor{blue}{\underline{0.330\scriptstyle\pm0.014}}$ & $\textcolor{red}{0.327\scriptstyle\pm0.010}$ \\
    ~ & 720 & $0.377\scriptstyle\pm0.037$ & $0.490\scriptstyle\pm0.038$ & $0.406\scriptstyle\pm0.072$ & $0.382\scriptstyle\pm0.066$ & $0.711\scriptstyle\pm0.081$ & $0.822\scriptstyle\pm0.034$ & $0.622\scriptstyle\pm0.045$ & $0.578\scriptstyle\pm0.051$ & $0.410\scriptstyle\pm0.036$ & $0.793\scriptstyle\pm0.034$ & $0.707\scriptstyle\pm0.050$ & $\textcolor{blue}{\underline{0.373\scriptstyle\pm0.032}}$ & $\textcolor{red}{0.367\scriptstyle\pm0.023}$ \\
    \midrule

    \multirow{4}{*}{ETTm2 } & 96 & $0.225\scriptstyle\pm0.039$ & $0.221\scriptstyle\pm0.039$ & $0.210\scriptstyle\pm0.040$ & $0.216\scriptstyle\pm0.035$ & $0.212\scriptstyle\pm0.082$ & $0.279\scriptstyle\pm0.031$ & $0.224\scriptstyle\pm0.033$ & $0.146\scriptstyle\pm0.012$ & $0.158\scriptstyle\pm0.012$ & $0.525\scriptstyle\pm0.047$ & $0.413\scriptstyle\pm0.059$ & $\textcolor{blue}{\underline{0.144\scriptstyle\pm0.011}}$ & $\textcolor{red}{0.134\scriptstyle\pm0.010}$ \\
    ~ & 192 & $0.233\scriptstyle\pm0.026$ & $0.229\scriptstyle\pm0.031$ & $0.234\scriptstyle\pm0.038$ & $0.215\scriptstyle\pm0.022$ & $0.535\scriptstyle\pm0.029$ & $0.292\scriptstyle\pm0.041$ & $0.316\scriptstyle\pm0.040$ & $0.189\scriptstyle\pm0.012$ & $0.184\scriptstyle\pm0.012$ & $0.530\scriptstyle\pm0.060$ & $0.427\scriptstyle\pm0.033$ & $\textcolor{blue}{\underline{0.170\scriptstyle\pm0.009}}$ & $\textcolor{red}{0.160\scriptstyle\pm0.010}$ \\
    ~ & 336 & $0.267\scriptstyle\pm0.023$ & $0.245\scriptstyle\pm0.049$ & $0.276\scriptstyle\pm0.019$ & $0.234\scriptstyle\pm0.024$ & $0.407\scriptstyle\pm0.043$ & $0.309\scriptstyle\pm0.032$ & $0.397\scriptstyle\pm0.051$ & $0.248\scriptstyle\pm0.024$ & $0.205\scriptstyle\pm0.020$ & $0.566\scriptstyle\pm0.047$ & $0.580\scriptstyle\pm0.169$ & $\textcolor{blue}{\underline{0.187\scriptstyle\pm0.021}}$ & $\textcolor{red}{0.158\scriptstyle\pm0.015}$ \\
    ~ & 720 & $0.290\scriptstyle\pm0.033$ & $0.385\scriptstyle\pm0.042$ & $0.540\scriptstyle\pm0.052$ & $0.288\scriptstyle\pm0.034$ & $0.355\scriptstyle\pm0.048$ & $0.475\scriptstyle\pm0.029$ & $0.416\scriptstyle\pm0.065$ & $0.306\scriptstyle\pm0.040$ & $0.295\scriptstyle\pm0.034$ & $0.561\scriptstyle\pm0.044$ & $0.749\scriptstyle\pm0.385$ & $\textcolor{blue}{\underline{0.275\scriptstyle\pm0.035}}$ & $\textcolor{red}{0.207\scriptstyle\pm0.026}$ \\
    \midrule

    \multirow{4}{*}{ETTh1 } & 96 & $0.407\scriptstyle\pm0.052$ & $0.386\scriptstyle\pm0.092$ & $0.393\scriptstyle\pm0.142$ & $0.407\scriptstyle\pm0.022$ & $0.371\scriptstyle\pm0.034$ & $0.423\scriptstyle\pm0.047$ & $0.510\scriptstyle\pm0.029$ & $0.557\scriptstyle\pm0.022$ & $0.380\scriptstyle\pm0.024$ & $0.585\scriptstyle\pm0.058$ & $0.481\scriptstyle\pm0.037$ & $\textcolor{blue}{\underline{0.336\scriptstyle\pm0.041}}$ & $\textcolor{red}{0.328\scriptstyle\pm0.019}$ \\
    ~ & 192 & $0.396\scriptstyle\pm0.022$ & $0.388\scriptstyle\pm0.041$ & $0.406\scriptstyle\pm0.079$ & $0.405\scriptstyle\pm0.088$ & $0.430\scriptstyle\pm0.022$ & $0.451\scriptstyle\pm0.012$ & $0.596\scriptstyle\pm0.056$ & $0.625\scriptstyle\pm0.065$ & $0.410\scriptstyle\pm0.026$ & $0.680\scriptstyle\pm0.058$ & $0.531\scriptstyle\pm0.018$ & $\textcolor{blue}{\underline{0.372\scriptstyle\pm0.023}}$ & $\textcolor{red}{0.365\scriptstyle\pm0.015}$ \\
    ~ & 336 & $0.406\scriptstyle\pm0.028$ & $0.415\scriptstyle\pm0.022$ & $0.410\scriptstyle\pm0.063$ & $0.412\scriptstyle\pm0.024$ & $0.462\scriptstyle\pm0.049$ & $0.481\scriptstyle\pm0.041$ & $0.581\scriptstyle\pm0.035$ & $0.574\scriptstyle\pm0.026$ & $0.430\scriptstyle\pm0.022$ & $0.666\scriptstyle\pm0.047$ & $0.580\scriptstyle\pm0.064$ & $\textcolor{blue}{\underline{0.394\scriptstyle\pm0.022}}$ & $\textcolor{red}{0.384\scriptstyle\pm0.018}$ \\
    ~ & 720 & $0.412\scriptstyle\pm0.008$ & $0.449\scriptstyle\pm0.022$ & $0.468\scriptstyle\pm0.012$ & $0.428\scriptstyle\pm0.024$ & $0.496\scriptstyle\pm0.019$ & $0.455\scriptstyle\pm0.025$ & $0.657\scriptstyle\pm0.017$ & $0.657\scriptstyle\pm0.014$ & $0.455\scriptstyle\pm0.020$ & $0.672\scriptstyle\pm0.015$ & $0.643\scriptstyle\pm0.046$ & $\textcolor{blue}{\underline{0.396\scriptstyle\pm0.012}}$ & $\textcolor{red}{0.389\scriptstyle\pm0.013}$ \\
    \midrule

    \multirow{4}{*}{ETTh2 } & 96 & $0.249\scriptstyle\pm0.015$ & $0.234\scriptstyle\pm0.011$ & $0.243\scriptstyle\pm0.009$ & $0.247\scriptstyle\pm0.028$ & $0.292\scriptstyle\pm0.012$ & $0.345\scriptstyle\pm0.042$ & $0.421\scriptstyle\pm0.033$ & $0.214\scriptstyle\pm0.018$ & $0.210\scriptstyle\pm0.014$ & $0.448\scriptstyle\pm0.031$ & $0.548\scriptstyle\pm0.158$ & $\textcolor{blue}{\underline{0.189\scriptstyle\pm0.010}}$ & $\textcolor{red}{0.183\scriptstyle\pm0.009}$ \\
    ~ & 192 & $0.249\scriptstyle\pm0.032$ & $0.247\scriptstyle\pm0.040$ & $0.252\scriptstyle\pm0.022$ & $0.265\scriptstyle\pm0.091$ & $0.376\scriptstyle\pm0.112$ & $0.343\scriptstyle\pm0.044$ & $0.339\scriptstyle\pm0.033$ & $0.294\scriptstyle\pm0.027$ & $0.235\scriptstyle\pm0.020$ & $0.575\scriptstyle\pm0.089$ & $0.766\scriptstyle\pm0.223$ & $\textcolor{blue}{\underline{0.213\scriptstyle\pm0.021}}$ & $\textcolor{red}{0.201\scriptstyle\pm0.015}$ \\
    ~ & 336 & $0.274\scriptstyle\pm0.027$ & $0.297\scriptstyle\pm0.029$ & $0.291\scriptstyle\pm0.032$ & $0.314\scriptstyle\pm0.045$ & $0.454\scriptstyle\pm0.057$ & $0.333\scriptstyle\pm0.078$ & $0.427\scriptstyle\pm0.041$ & $0.353\scriptstyle\pm0.028$ & $0.278\scriptstyle\pm0.028$ & $0.606\scriptstyle\pm0.095$ & $0.942\scriptstyle\pm0.408$ & $\textcolor{blue}{\underline{0.263\scriptstyle\pm0.039}}$ & $\textcolor{red}{0.221\scriptstyle\pm0.020}$ \\
    ~ & 720 & $0.286\scriptstyle\pm0.042$ & $0.667\scriptstyle\pm0.012$ & $0.401\scriptstyle\pm0.022$ & $0.371\scriptstyle\pm0.021$ & $1.092\scriptstyle\pm0.019$ & $0.412\scriptstyle\pm0.020$ & $0.482\scriptstyle\pm0.022$ & $0.382\scriptstyle\pm0.030$ & $0.322\scriptstyle\pm0.025$ & $0.550\scriptstyle\pm0.018$ & $0.688\scriptstyle\pm0.161$ & $\textcolor{blue}{\underline{0.278\scriptstyle\pm0.020}}$ & $\textcolor{red}{0.223\scriptstyle\pm0.015}$ \\
    \midrule

    \multirow{4}{*}{Electricity } & 96 & $0.146\scriptstyle\pm0.015$ & $0.134\scriptstyle\pm0.002$ & $0.137\scriptstyle\pm0.002$ & $0.168\scriptstyle\pm0.012$ & $0.108\scriptstyle\pm0.009$ & $0.114\scriptstyle\pm0.010$ & $0.441\scriptstyle\pm0.013$ & $0.203\scriptstyle\pm0.189$ & $0.104\scriptstyle\pm0.006$ & $0.119\scriptstyle\pm0.003$ & $0.118\scriptstyle\pm0.003$ & $\textcolor{blue}{\underline{0.093\scriptstyle\pm0.002}}$ & $\textcolor{red}{0.089\scriptstyle\pm0.003}$ \\
    ~ & 192 & $0.143\scriptstyle\pm0.023$ & $0.137\scriptstyle\pm0.022$ & $0.143\scriptstyle\pm0.112$ & $0.163\scriptstyle\pm0.032$ & $0.120\scriptstyle\pm0.033$ & $0.131\scriptstyle\pm0.008$ & $0.441\scriptstyle\pm0.005$ & $0.264\scriptstyle\pm0.129$ & $0.112\scriptstyle\pm0.010$ & $0.124\scriptstyle\pm0.005$ & $0.121\scriptstyle\pm0.003$ & $\textcolor{blue}{\underline{0.102\scriptstyle\pm0.010}}$ & $\textcolor{red}{0.096\scriptstyle\pm0.008}$ \\
    ~ & 336 & $0.151\scriptstyle\pm0.017$ & $0.136\scriptstyle\pm0.002$ & $0.139\scriptstyle\pm0.002$ & $0.168\scriptstyle\pm0.010$ & $0.122\scriptstyle\pm0.018$ & - & $0.571\scriptstyle\pm0.022$ & - & $0.118\scriptstyle\pm0.006$ & $0.126\scriptstyle\pm0.008$ & $0.123\scriptstyle\pm0.001$ & $\textcolor{blue}{\underline{0.107\scriptstyle\pm0.002}}$ & $\textcolor{red}{0.102\scriptstyle\pm0.003}$ \\
    ~ & 720 & $0.149\scriptstyle\pm0.025$ & $0.140\scriptstyle\pm0.009$ & $0.149\scriptstyle\pm0.012$ & $0.164\scriptstyle\pm0.024$ & $0.136\scriptstyle\pm0.098$ & - & $0.622\scriptstyle\pm0.142$ & - & $0.126\scriptstyle\pm0.014$ & $0.134\scriptstyle\pm0.004$ & $0.144\scriptstyle\pm0.017$ & $\textcolor{blue}{\underline{0.117\scriptstyle\pm0.019}}$ & $\textcolor{red}{0.111\scriptstyle\pm0.012}$ \\
    \midrule

    \multirow{4}{*}{Traffic } & 96 & $0.377\scriptstyle\pm0.024$ & $0.332\scriptstyle\pm0.008$ & $0.332\scriptstyle\pm0.007$ & $0.228\scriptstyle\pm0.010$ & $0.274\scriptstyle\pm0.012$ & $0.265\scriptstyle\pm0.007$ & $0.342\scriptstyle\pm0.042$ & - & $0.238\scriptstyle\pm0.008$ & $0.234\scriptstyle\pm0.006$ & $0.231\scriptstyle\pm0.003$ & $\textcolor{blue}{\underline{0.230\scriptstyle\pm0.010}}$ & $\textcolor{red}{0.223\scriptstyle\pm0.009}$ \\
    ~ & 192 & $0.388\scriptstyle\pm0.011$ & $0.326\scriptstyle\pm0.009$ & $0.350\scriptstyle\pm0.010$ & $0.225\scriptstyle\pm0.012$ & - & - & $0.354\scriptstyle\pm0.012$ & - & $0.242\scriptstyle\pm0.006$ & $0.239\scriptstyle\pm0.004$ & $0.236\scriptstyle\pm0.002$ & $\textcolor{blue}{\underline{0.234\scriptstyle\pm0.003}}$ & $\textcolor{red}{0.230\scriptstyle\pm0.006}$ \\
    ~ & 336 & $0.416\scriptstyle\pm0.028$ & $0.335\scriptstyle\pm0.010$ & $0.405\scriptstyle\pm0.011$ & $0.242\scriptstyle\pm0.022$ & - & - & $0.392\scriptstyle\pm0.006$ & - & $0.250\scriptstyle\pm0.007$ & $0.246\scriptstyle\pm0.003$ & $0.248\scriptstyle\pm0.006$ & $\textcolor{blue}{\underline{0.242\scriptstyle\pm0.007}}$ & $\textcolor{red}{0.235\scriptstyle\pm0.008}$ \\
    ~ & 720 & $0.432\scriptstyle\pm0.032$ & $0.361\scriptstyle\pm0.030$ & $0.453\scriptstyle\pm0.022$ & $0.253\scriptstyle\pm0.012$ & - & - & $0.478\scriptstyle\pm0.006$ & - & $0.256\scriptstyle\pm0.010$ & $0.263\scriptstyle\pm0.001$ & $0.264\scriptstyle\pm0.006$ & $\textcolor{blue}{\underline{0.248\scriptstyle\pm0.010}}$ & $\textcolor{red}{0.236\scriptstyle\pm0.009}$ \\
    \midrule

    \multirow{4}{*}{Weather } & 96 & $0.146\scriptstyle\pm0.019$ & $0.144\scriptstyle\pm0.017$ & $0.279\scriptstyle\pm0.027$ & $0.145\scriptstyle\pm0.016$ & $0.176\scriptstyle\pm0.011$ & $0.139\scriptstyle\pm0.010$ & $0.113\scriptstyle\pm0.022$ & $0.087\scriptstyle\pm0.012$ & $0.095\scriptstyle\pm0.010$ & $0.164\scriptstyle\pm0.023$ & $0.145\scriptstyle\pm0.017$ & $\textcolor{blue}{\underline{0.086\scriptstyle\pm0.011}}$ & $\textcolor{red}{0.077\scriptstyle\pm0.009}$ \\
    ~ & 192 & $0.148\scriptstyle\pm0.022$ & $0.145\scriptstyle\pm0.015$ & $0.264\scriptstyle\pm0.013$ & $0.144\scriptstyle\pm0.012$ & $0.166\scriptstyle\pm0.022$ & $0.160\scriptstyle\pm0.037$ & $0.144\scriptstyle\pm0.020$ & $0.086\scriptstyle\pm0.007$ & $0.092\scriptstyle\pm0.010$ & $0.158\scriptstyle\pm0.024$ & $0.147\scriptstyle\pm0.025$ & $\textcolor{blue}{\underline{0.083\scriptstyle\pm0.011}}$ & $\textcolor{red}{0.080\scriptstyle\pm0.010}$ \\
    ~ & 336 & $0.152\scriptstyle\pm0.032$ & $0.146\scriptstyle\pm0.011$ & $0.283\scriptstyle\pm0.021$ & $0.149\scriptstyle\pm0.023$ & $0.168\scriptstyle\pm0.014$ & $0.170\scriptstyle\pm0.027$ & $0.138\scriptstyle\pm0.033$ & $0.098\scriptstyle\pm0.002$ & $0.103\scriptstyle\pm0.008$ & $0.162\scriptstyle\pm0.006$ & $0.160\scriptstyle\pm0.012$ & $\textcolor{blue}{\underline{0.093\scriptstyle\pm0.010}}$ & $\textcolor{red}{0.086\scriptstyle\pm0.005}$ \\
    ~ & 720 & $0.162\scriptstyle\pm0.009$ & $0.147\scriptstyle\pm0.019$ & $0.317\scriptstyle\pm0.021$ & $0.152\scriptstyle\pm0.029$ & $0.149\scriptstyle\pm0.034$ & $0.148\scriptstyle\pm0.040$ & $0.141\scriptstyle\pm0.026$ & $0.102\scriptstyle\pm0.005$ & $0.106\scriptstyle\pm0.008$ & $0.136\scriptstyle\pm0.020$ & $0.135\scriptstyle\pm0.008$ & $\textcolor{blue}{\underline{0.099\scriptstyle\pm0.009}}$ & $\textcolor{red}{0.086\scriptstyle\pm0.006}$ \\
    \midrule

    \multirow{4}{*}{Exchange } & 96 & $0.079\scriptstyle\pm0.005$ & $0.077\scriptstyle\pm0.001$ & $0.069\scriptstyle\pm0.007$ & $0.079\scriptstyle\pm0.002$ & $0.033\scriptstyle\pm0.003$ & $0.036\scriptstyle\pm0.009$ & $0.090\scriptstyle\pm0.010$ & $0.036\scriptstyle\pm0.005$ & $0.040\scriptstyle\pm0.004$ & $0.079\scriptstyle\pm0.002$ & $0.091\scriptstyle\pm0.009$ & $\textcolor{blue}{\underline{0.032\scriptstyle\pm0.002}}$ & $\textcolor{red}{0.029\scriptstyle\pm0.002}$ \\
    ~ & 192 & $0.081\scriptstyle\pm0.015$ & $0.078\scriptstyle\pm0.008$ & $0.069\scriptstyle\pm0.007$ & $0.081\scriptstyle\pm0.002$ & $0.044\scriptstyle\pm0.004$ & $0.058\scriptstyle\pm0.007$ & $0.106\scriptstyle\pm0.010$ & $0.058\scriptstyle\pm0.005$ & $0.049\scriptstyle\pm0.006$ & $0.100\scriptstyle\pm0.019$ & $0.087\scriptstyle\pm0.005$ & $\textcolor{blue}{\underline{0.040\scriptstyle\pm0.005}}$ & $\textcolor{red}{0.037\scriptstyle\pm0.006}$ \\
    ~ & 336 & $0.086\scriptstyle\pm0.003$ & $0.083\scriptstyle\pm0.005$ & $0.071\scriptstyle\pm0.005$ & $0.085\scriptstyle\pm0.010$ & $0.074\scriptstyle\pm0.017$ & $0.058\scriptstyle\pm0.009$ & $0.106\scriptstyle\pm0.010$ & $0.076\scriptstyle\pm0.006$ & $0.065\scriptstyle\pm0.006$ & $0.086\scriptstyle\pm0.008$ & $0.091\scriptstyle\pm0.002$ & $\textcolor{blue}{\underline{0.054\scriptstyle\pm0.001}}$ & $\textcolor{red}{0.050\scriptstyle\pm0.002}$ \\
    ~ & 720 & $0.116\scriptstyle\pm0.022$ & $0.113\scriptstyle\pm0.015$ & $0.097\scriptstyle\pm0.011$ & $0.126\scriptstyle\pm0.001$ & $0.182\scriptstyle\pm0.010$ & $0.191\scriptstyle\pm0.006$ & $0.142\scriptstyle\pm0.009$ & $0.173\scriptstyle\pm0.020$ & $0.094\scriptstyle\pm0.012$ & $0.113\scriptstyle\pm0.016$ & $0.103\scriptstyle\pm0.009$ & $\textcolor{blue}{\underline{0.084\scriptstyle\pm0.017}}$ & $\textcolor{red}{0.077\scriptstyle\pm0.012}$ \\
    \midrule

    \multirow{4}{*}{ILI } & 24 & $0.303\scriptstyle\pm0.021$ & $0.265\scriptstyle\pm0.027$ & $0.271\scriptstyle\pm0.032$ & $0.382\scriptstyle\pm0.018$ & $0.124\scriptstyle\pm0.019$ & $0.118\scriptstyle\pm0.033$ & $0.242\scriptstyle\pm0.086$ & $0.263\scriptstyle\pm0.012$ & $0.135\scriptstyle\pm0.010$ & $0.296\scriptstyle\pm0.044$ & $0.283\scriptstyle\pm0.001$ & $\textcolor{blue}{\underline{0.116\scriptstyle\pm0.011}}$ & $\textcolor{red}{0.101\scriptstyle\pm0.008}$ \\
    ~ & 36 & $0.262\scriptstyle\pm0.013$ & $0.222\scriptstyle\pm0.047$ & $0.258\scriptstyle\pm0.058$ & $0.286\scriptstyle\pm0.037$ & $0.144\scriptstyle\pm0.011$ & $0.143\scriptstyle\pm0.089$ & $0.246\scriptstyle\pm0.117$ & $0.298\scriptstyle\pm0.011$ & $0.160\scriptstyle\pm0.010$ & $0.298\scriptstyle\pm0.048$ & $0.307\scriptstyle\pm0.007$ & $\textcolor{blue}{\underline{0.142\scriptstyle\pm0.008}}$ & $\textcolor{red}{0.134\scriptstyle\pm0.009}$ \\
    ~ & 48 & $0.334\scriptstyle\pm0.028$ & $0.262\scriptstyle\pm0.023$ & $0.225\scriptstyle\pm0.043$ & $0.291\scriptstyle\pm0.032$ & $0.159\scriptstyle\pm0.020$ & $0.160\scriptstyle\pm0.039$ & $0.275\scriptstyle\pm0.044$ & $0.301\scriptstyle\pm0.034$ & $0.173\scriptstyle\pm0.012$ & $0.320\scriptstyle\pm0.025$ & $0.314\scriptstyle\pm0.009$ & $\textcolor{blue}{\underline{0.152\scriptstyle\pm0.017}}$ & $\textcolor{red}{0.140\scriptstyle\pm0.011}$ \\
    ~ & 60 & $0.288\scriptstyle\pm0.031$ & $0.278\scriptstyle\pm0.017$ & $0.245\scriptstyle\pm0.017$ & $0.287\scriptstyle\pm0.023$ & $0.216\scriptstyle\pm0.014$ & $0.183\scriptstyle\pm0.019$ & $0.272\scriptstyle\pm0.020$ & $0.299\scriptstyle\pm0.013$ & $0.190\scriptstyle\pm0.012$ & $0.325\scriptstyle\pm0.068$ & $0.333\scriptstyle\pm0.005$ & $\textcolor{blue}{\underline{0.167\scriptstyle\pm0.007}}$ & $\textcolor{red}{0.155\scriptstyle\pm0.006}$ \\
    \bottomrule
    \\
  \end{tabular}}
  \label{tab:long_term_fore_NMAE}
\end{table*}

\clearpage
\subsection{Uncertainty-Dynamics Diagnostics}
\label{app:uncertainty_diagnostics}

Table~\ref{tab:app_uncertainty_diagnostics} provides the supplementary numerical values corresponding to the uncertainty-dynamics diagnostics table in the main text. Coverage is empirical 95\% interval coverage; Sharpness is the normalized 95\% interval width; QICE is reported in percentage points; and $\rho_\sigma$ is lag-1 scale autocorrelation.

\begin{table}[h]
    \centering
    \caption{Uncertainty-dynamics diagnostics for scale head architectures. \textcolor{red}{RED}: THE BEST OR CLOSEST TO 95.}
    \label{tab:app_uncertainty_diagnostics}
    \scriptsize
    \setlength{\tabcolsep}{3.2pt}
    \resizebox{\linewidth}{!}{
    \begin{tabular}{c|c|c|cccc}
        \toprule
        Dataset & $H$ & Head & Cov.95 $\uparrow$ & Sharp. $\downarrow$ & QICE $\downarrow$ & $\rho_\sigma \uparrow$ \\
        \midrule
        \multirow{6}{*}{ETTh1}
        & \multirow{3}{*}{96}  & MLP  & 91.2$\pm$0.8 & 1.84$\pm$0.05 & 2.91$\pm$0.18 & 0.61$\pm$0.03 \\
        &                      & LSTM & 93.6$\pm$0.6 & 1.62$\pm$0.04 & 1.94$\pm$0.14 & 0.78$\pm$0.02 \\
        &                      & GRU  & \textcolor{red}{94.4$\pm$0.5} & \textcolor{red}{1.53$\pm$0.03} & \textcolor{red}{1.58$\pm$0.11} & \textcolor{red}{0.84$\pm$0.02} \\
        & \multirow{3}{*}{192} & MLP  & 90.5$\pm$0.9 & 1.91$\pm$0.06 & 3.24$\pm$0.22 & 0.58$\pm$0.04 \\
        &                      & LSTM & 93.1$\pm$0.7 & 1.70$\pm$0.05 & 2.18$\pm$0.16 & 0.76$\pm$0.03 \\
        &                      & GRU  & \textcolor{red}{94.1$\pm$0.6} & \textcolor{red}{1.60$\pm$0.04} & \textcolor{red}{1.71$\pm$0.12} & \textcolor{red}{0.83$\pm$0.02} \\
        \midrule
        \multirow{6}{*}{ETTh2}
        & \multirow{3}{*}{96}  & MLP  & 90.8$\pm$0.7 & 1.76$\pm$0.04 & 2.84$\pm$0.17 & 0.63$\pm$0.03 \\
        &                      & LSTM & 93.9$\pm$0.5 & 1.55$\pm$0.03 & 1.83$\pm$0.13 & 0.80$\pm$0.02 \\
        &                      & GRU  & \textcolor{red}{94.7$\pm$0.4} & \textcolor{red}{1.46$\pm$0.03} & \textcolor{red}{1.42$\pm$0.10} & \textcolor{red}{0.86$\pm$0.01} \\
        & \multirow{3}{*}{192} & MLP  & 89.9$\pm$1.0 & 1.83$\pm$0.05 & 3.09$\pm$0.21 & 0.59$\pm$0.04 \\
        &                      & LSTM & 93.4$\pm$0.6 & 1.61$\pm$0.04 & 2.01$\pm$0.15 & 0.78$\pm$0.03 \\
        &                      & GRU  & \textcolor{red}{94.5$\pm$0.5} & \textcolor{red}{1.50$\pm$0.03} & \textcolor{red}{1.55$\pm$0.11} & \textcolor{red}{0.85$\pm$0.02} \\
        \bottomrule
    \end{tabular}
    }
\end{table}
\subsection{Hyperparameter Sensitivity Analysis}
\label{app:hyperparameter_analysis}

To identify a robust default configuration and verify the sensitivity of VolDy-VAE, we conduct an analysis on three key hyperparameters: Patch Length ($P$), Hidden Dimension ($D$), and the number of MLP encoder layers ($N$).
Here, $N$ refers only to the depth of the MLP-based encoder; the volatility scale head remains a single-layer GRU throughout all experiments.
We perform experiments on the representative forecasting task ($L=96 \rightarrow H=192$).
Table~\ref{tab:hyper_sensitivity_specific} details the performance on two distinct datasets: \textit{ETTh1} (representative of datasets with strong periodicity) and \textit{Weather} (representative of high-dimensional, volatile datasets).

\begin{table}[h]
    \centering
    \small
    \caption{\textbf{Hyperparameter Sensitivity ($L=96 \rightarrow H=192$).}
    We report CRPS on \textit{ETTh1} and \textit{Weather}.
    The default setting is selected as a cross-dataset performance--efficiency trade-off: \textbf{$P=24, D=256, N=3$}, where $N$ is the number of MLP encoder layers.}
    \label{tab:hyper_sensitivity_specific}
    \setlength{\tabcolsep}{5pt}
    \resizebox{\linewidth}{!}{
    \begin{tabular}{cc|cc|cc|cc}
        \toprule
        \multicolumn{2}{c|}{{Parameter}} & \multicolumn{2}{c|}{{Patch Length ($P$)}} & \multicolumn{2}{c|}{{Hidden Dim ($D$)}} & \multicolumn{2}{c}{{MLP Layers ($N$)}} \\
        \cmidrule(lr){1-2} \cmidrule(lr){3-4} \cmidrule(lr){5-6} \cmidrule(lr){7-8}
        \multicolumn{2}{c|}{\textit{Dataset}} & \textit{Value} & \textit{CRPS} & \textit{Value} & \textit{CRPS} & \textit{Value} & \textit{CRPS} \\
        \midrule
        \multirow{4}{*}{\rotatebox{90}{{ETTh1}}} 
          & & 16 & 0.278 & 64  & 0.280 & 1 & 0.279 \\
          & & 24 & 0.277 & 128 & 0.279 & 2 & 0.278\\
          & & 32 & 0.278 & 256 & 0.277 & 3 & 0.277 \\
          & & 48 & 0.294 & 512 & 0.278 & 4 & 0.274 \\
        \midrule
        \multirow{4}{*}{\rotatebox{90}{{Weather}}} 
          & & 16 & 0.069 & 64  & 0.084 & 1 & 0.074 \\
          & & 24 & 0.067 & 128 & 0.086 & 2 & 0.074 \\
          & & 32 & 0.077 & 256 & 0.067 & 3 & 0.067 \\
          & & 48 & 0.073 & 512 & 0.074 & 4 & 0.067 \\
        \bottomrule
    \end{tabular}}
    \vspace{-1em}
\end{table}
\paragraph{Analysis}
As detailed in Table~\ref{tab:hyper_sensitivity_specific}, VolDy-VAE demonstrates general robustness to hyperparameter variations across different domains. The performance fluctuations are relatively minor, indicating that the model does not require meticulous tuning to achieve high-quality forecasts.
Although $N=4$ yields the lowest CRPS on \textit{ETTh1} and ties $N=3$ on \textit{Weather}, we uniformly adopt $P=24, D=256, N=3$ as a default trade-off between performance and computational efficiency across datasets. This setting provides sufficient encoder capacity and receptive field coverage while maintaining a lightweight architecture, avoiding the marginal additional cost of deeper MLP encoders.

\section{Limitations and Future Work}
\label{sec:limitations}

We acknowledge several limitations of the current work that motivate future research.
First, while the recurrent volatility module captures temporal coherence, it may struggle with extremely abrupt regime transitions that occur within a single time step; incorporating explicit regime-switching mechanisms (e.g., Hidden Markov Models) could further improve adaptability.
Second, our evaluation focuses on standard benchmarks with moderate dimensionality; scaling to very high-dimensional settings (e.g., thousands of correlated variables) remains to be explored.
Finally, the current framework does not explicitly model cross-variable volatility spillovers, which are important in domains such as financial contagion modeling.

\section{Additional Showcases}
\label{showcases}

To provide additional qualitative evidence for VolDy-VAE's modeling capabilities, we present visualization results on four datasets: \textit{ETTh2}, and \textit{ETTm1}, \textit{ETTm2}, \textit{Weather}.
These cases cover a wide spectrum of temporal dynamics, ranging from smooth periodic trends to highly volatile fluctuations and structural breaks.

\paragraph{Adaptive Responsiveness in Volatile Regimes}
Figure~\ref{fig:showcase_volatile} illustrates performance on \textit{ETTh2} and \textit{ETTm1}, which contain complex volatility and abrupt regime shifts (e.g., the step-like drops in the top row of \textit{ETTh2}).
Here, VolDy-VAE shows adaptive attenuation: its confidence intervals expand when the series undergoes a sudden transition or volatility spike, and contract once the series stabilizes.
Baseline models tend to generate more static, uniform bands that fail to capture the temporal structure of the uncertainty.
This visual evidence supports the claim that the Location-Scale formulation makes the uncertainty estimate more context-aware.

\paragraph{Failure of Baselines in Stable Regimes}
As shown in Figure~\ref{fig:showcase_stable}, specifically on the \textit{Weather} and \textit{ETTm2} datasets, the ground truth series exhibit high stability or smooth periodicity.
In these scenarios, an ideal probabilistic model should output narrow confidence intervals (CIs) to reflect low aleatoric uncertainty.
However, baseline methods, including the generative $K^2$VAE and the probabilistic-enhanced point forecasters (Koopa, PatchTST), produce wider and less adaptive uncertainty bands. On the visualized \textit{ETTh1} window reported in the main text, their average normalized sharpness values are substantially larger than VolDy-VAE's while their empirical coverage is also higher, indicating measurable over-coverage rather than only a visual impression.
In contrast, {VolDy-VAE} estimates lower volatility in these windows, producing sharper CIs around the ground truth.

\begin{figure*}[ht]
    \centering
    \includegraphics[width=\textwidth]{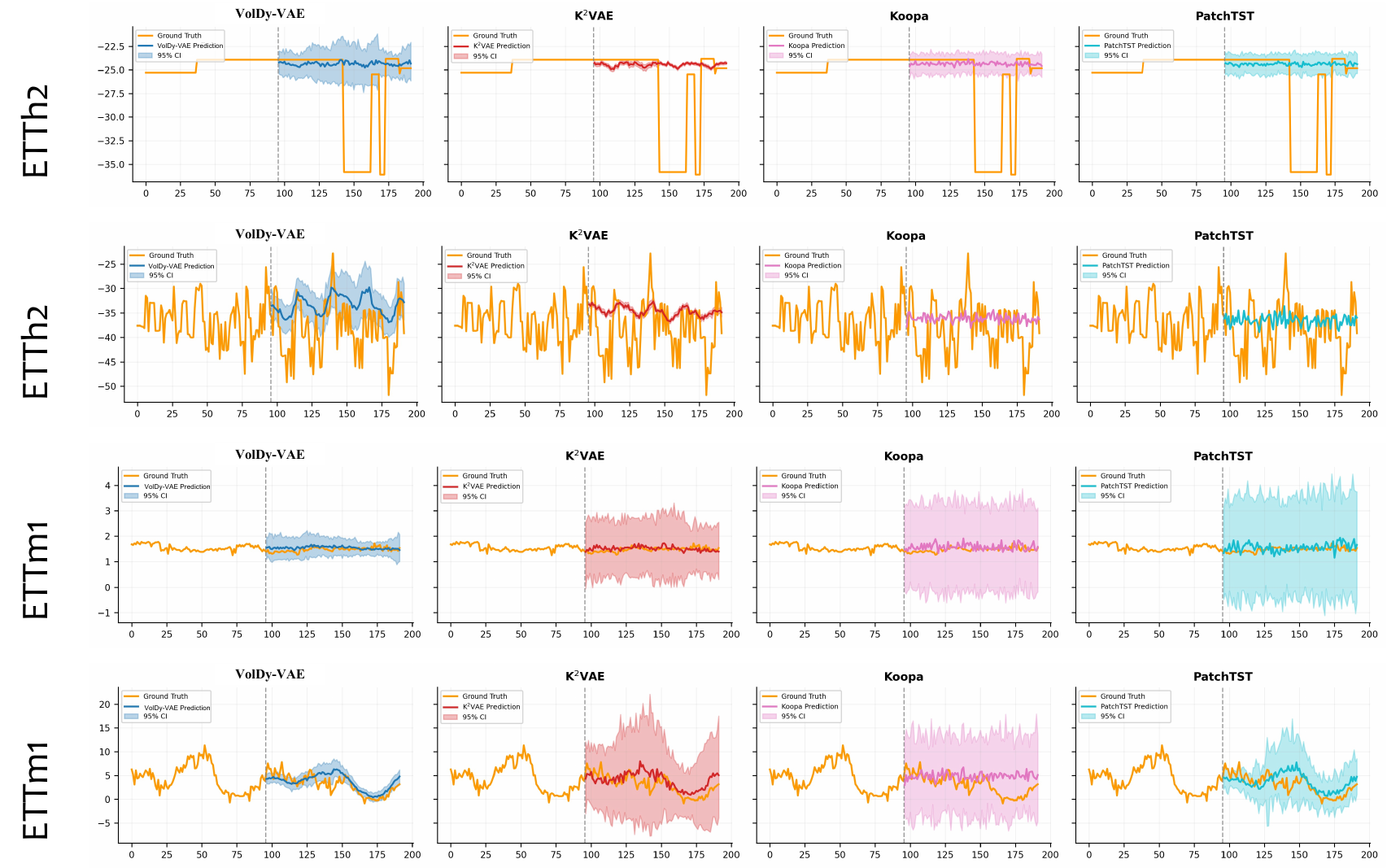}
    \caption{{Visualization on Volatile Regimes (\textit{ETTh2}, \textit{ETTm1}).} All subfigures share a unified Y-axis scale.
    This figure highlights performance under high volatility and structural breaks.
    {Observation:} In the top row of \textit{ETTh2}, the data exhibits sudden jumps. {VolDy-VAE} produces uncertainty bands that expand near regime shifts and contract otherwise. In contrast, baselines produce more static bands that capture less of the dynamic risk pattern.}
    \label{fig:showcase_volatile}
\end{figure*}

\begin{figure*}[ht]
    \centering
    \includegraphics[width=\textwidth]{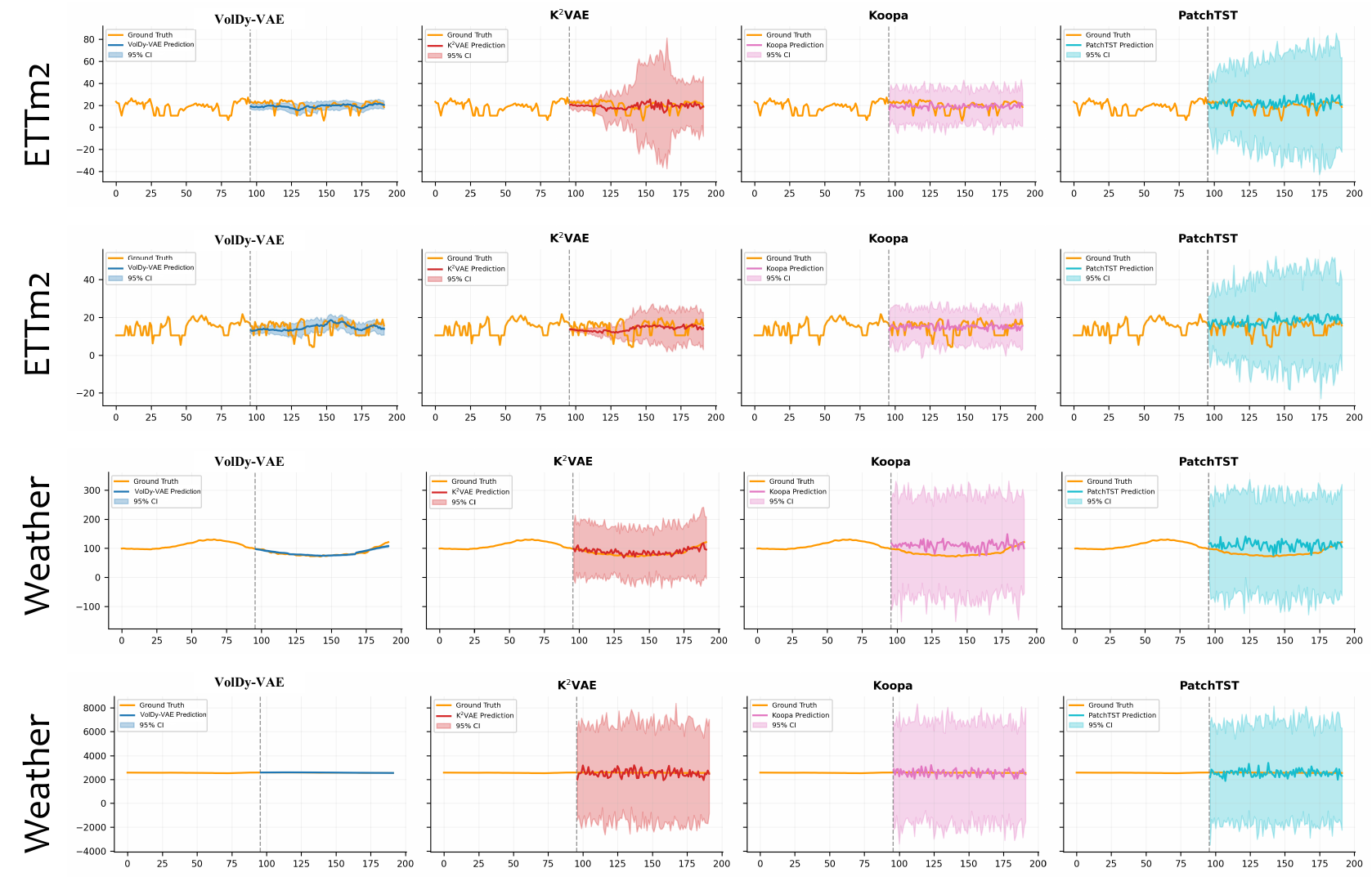}
    \caption{{Visualization on Stable Regimes (\textit{ETTm2}, \textit{Weather}).} All subfigures share a unified Y-axis scale.
    We compare VolDy-VAE against baselines on time series with smooth trends.
    {Observation:} Baselines ($K^2$VAE, Koopa, PatchTST) show {uncertainty over-estimation}, generating wide confidence intervals even for flat or smooth curves (e.g., \textit{Weather}).
    {VolDy-VAE} (Blue) estimates lower aleatoric uncertainty in these stable windows, producing sharper intervals that better reflect the visual smoothness of the prediction target.}
    \label{fig:showcase_stable}
\end{figure*}

\fi

\end{document}